\documentclass[11pt]{article} 

\pdfoutput=1

\usepackage{etoolbox}
\newtoggle{arxiv}
\toggletrue{arxiv}

\usepackage[utf8]{inputenc}
\usepackage{mathrsfs}
\usepackage{microtype}
\usepackage{subfigure}
\usepackage{booktabs} 
\usepackage{longtable,makecell}
\usepackage{amsmath, amsfonts, amsthm, amssymb, graphicx}
\usepackage{mathtools,verbatim}
\usepackage{enumitem}
\usepackage{bm}
\usepackage{thm-restate}
\usepackage{natbib}
\usepackage{xspace,upgreek}
\usepackage[dvipsnames]{xcolor}

\usepackage{color-edits}

\addauthor{aw}{red}
\addauthor{ms}{orange}
\addauthor{kevin}{ForestGreen}

\usepackage[noend]{algpseudocode}
\usepackage{algorithm}

\iftoggle{arxiv}
{
	\usepackage[scaled=.90]{helvet}
	
	\usepackage[vmargin=1.00in,hmargin=1.00in,centering,letterpaper]{geometry}
	\usepackage{fancyhdr, lastpage}

	\setlength{\headsep}{.10in}
	\setlength{\headheight}{15pt}
	\cfoot{}
	\lfoot{}
	\rfoot{\sc Page\ \thepage\ of\ \protect\pageref{LastPage}}
	\renewcommand\headrulewidth{0.4pt}
	\renewcommand\footrulewidth{0.4pt}
}

\usepackage{hyperref}
\hypersetup{colorlinks,citecolor=blue}

\newtoggle{upperbound}
\togglefalse{upperbound}

\usepackage{cleveref}
\Crefname{lem}{Lemma}{Lemmas}
\Crefname{prop}{Proposition}{Propositions}

\newcommand{\algcomment}[1]{\textcolor{blue!70!black}{\footnotesize{\texttt{\textbf{//
          #1}}}}}

\newcommand{\poly}{\mathrm{poly}}
\newcommand{\calX}{\mathcal{X}}
\newcommand{\calS}{\mathcal{S}}

\numberwithin{equation}{section}

\newcommand{\op}{\mathrm{op}}

{
 \theoremstyle{plain}
      \newtheorem{asm}{Assumption}
}
\creflabelformat{asm}{#2#1#3}
\creflabelformat{asm}{#2#1#3}

\theoremstyle{plain}
\newtheorem{thm}{Theorem}
\newtheorem{lem}{Lemma}[section]
\newtheorem{cor}{Corollary}
\newtheorem{claim}[lem]{Claim}

\newtheorem{prop}[thm]{Proposition}
\theoremstyle{definition}
\newtheorem{defn}{Definition}[section]

\newtheorem{exmp}{Example}[section]

\renewcommand{\Pr}{\mathbb{P}}
\newcommand{\Exp}{\mathbb{E}}
\newcommand{\Var}{\mathrm{Var}}

\newcommand{\rmd}{\mathrm{d}}

\newcommand{\dist}{\mathrm{dist}}

\newcommand{\N}{\mathbb{N}}

\newcommand{\R}{\mathbb{R}}

\DeclareMathOperator*{\argmax}{arg\,max}
\DeclareMathOperator*{\argmin}{arg\,min}

\def\ddefloop#1{\ifx\ddefloop#1\else\ddef{#1}\expandafter\ddefloop\fi}
\def\ddef#1{\expandafter\def\csname bb#1\endcsname{\ensuremath{\mathbb{#1}}}}
\ddefloop ABCDEFGHIJKLMNOPQRSTUVWXYZ\ddefloop

\def\ddefloop#1{\ifx\ddefloop#1\else\ddef{#1}\expandafter\ddefloop\fi}
\def\ddef#1{\expandafter\def\csname fr#1\endcsname{\ensuremath{\mathfrak{#1}}}}
\ddefloop ABCDEFGHIJKLMNOPQRSTUVWXYZ\ddefloop
\def\ddefloop#1{\ifx\ddefloop#1\else\ddef{#1}\expandafter\ddefloop\fi}
\def\ddef#1{\expandafter\def\csname scr#1\endcsname{\ensuremath{\mathscr{#1}}}}
\ddefloop ABCDEFGHIJKLMNOPQRSTUVWXYZ\ddefloop
\def\ddefloop#1{\ifx\ddefloop#1\else\ddef{#1}\expandafter\ddefloop\fi}
\def\ddef#1{\expandafter\def\csname b#1\endcsname{\ensuremath{\mathbf{#1}}}}
\ddefloop ABCDEFGHIJKLMNOPQRSTUVWXYZ\ddefloop
\def\ddef#1{\expandafter\def\csname c#1\endcsname{\ensuremath{\mathcal{#1}}}}
\ddefloop ABCDEFGHIJKLMNOPQRSTUVWXYZ\ddefloop
\def\ddef#1{\expandafter\def\csname h#1\endcsname{\ensuremath{\widehat{#1}}}}
\ddefloop ABCDEFGHIJKLMNOPQRSTUVWXYZ\ddefloop
\def\ddef#1{\expandafter\def\csname t#1\endcsname{\ensuremath{\widetilde{#1}}}}
\ddefloop ABCDEFGHIJKLMNOPQRSTUVWXYZ\ddefloop

\def\ddefloop#1{\ifx\ddefloop#1\else\ddef{#1}\expandafter\ddefloop\fi}
\def\ddef#1{\expandafter\def\csname mat#1\endcsname{\ensuremath{\mathbf{#1}}}}
\ddefloop abcdefgijklmnopqrstuvwxyzABCDEFGHIJKLMNOPQRSTUVWXYZ\ddefloop

\newcommand{\Vst}{V^\star}
\newcommand{\Vpi}{V^\pi}

\newcommand{\Qst}{Q^\star}
\newcommand{\Exphat}{\widehat{\Exp}}

\newcommand{\pist}{\pi^\star}

\newcommand{\what}{\widehat{w}}

\newcommand{\Qpi}{Q^\pi}

\newcommand{\simplex}{\bigtriangleup}

\newcommand{\cOtil}{\widetilde{\cO}}

\newcommand{\wtil}{\widetilde{\bm{w}}}

\newcommand{\euler}{\textsc{Euler}\xspace}

\newcommand{\false}{\texttt{false}}
\newcommand{\true}{\texttt{true}}

\newcommand{\sighat}{\bar{\mathsf{v}}}
\newcommand{\sighatb}{\widehat{\sigma}}
\newcommand{\sigmin}{\mathsf{v}_{\mathrm{min}}}
\newcommand{\sigminb}{\sigma_{\mathrm{min}}}
\newcommand{\Xtil}{\widetilde{X}}
\newcommand{\cXtil}{\widetilde{\cX}}
\newcommand{\zst}{z^*}
\newcommand{\ztilst}{\widetilde{z}^*}
\newcommand{\ftil}{\widetilde{f}}
\newcommand{\alphatil}{\widetilde{\alpha}}

\newcommand{\ytil}{\widetilde{y}}

\newcommand{\betatil}{\widetilde{\beta}}
\newcommand{\cRtil}{\widetilde{\cR}}

\newcommand{\alphamax}{\alpha_{\max}}

\newcommand{\Cmdp}{C_{\mathrm{mdp}}}
\newcommand{\Kinit}{K_{\mathrm{init}}}

\newcommand{\cvar}{c_{\sighat}}

\newcommand{\efficient}{\texttt{efficient}}

\newcommand{\algname}{\textsc{Force}\xspace}

\newcommand{\inner}[2]{\left \langle #1, #2 \right \rangle}
\newcommand{\innerb}[2]{\langle #1, #2 \rangle}

\newcommand{\catonii}{\mathsf{cat}}

\newcommand{\bx}{\bm{x}}
\newcommand{\vLam}{(\bv,\bLambda)}
\newcommand{\logterm}{\mathsf{logs}}

\newcommand{\Ball}{\mathcal{B}}
\newcommand{\bSigma}{\bm{\Sigma}}

\newcommand{\bthettil}{\tilde{\bm{\theta}}}

\newcommand{\bvtil}{\widetilde{\bv}}
\newcommand{\Cpoly}{\mathcal{C}_{\mathrm{poly}}}
\newcommand{\calN}{\mathcal{N}}

\newcommand{\covnum}{\mathsf{N}}
\newcommand{\fcatoni}{f_{\catonii}}
\newcommand{\psicat}{\psi_{\catonii}}

\newcommand{\bLambda}{\bm{\Lambda}}
\newcommand{\bv}{\bm{v}}
\newcommand{\bphi}{\bm{\phi}}
\newcommand{\calF}{\mathcal{F}}
\newcommand{\btheta}{\bm{\theta}}
\newcommand{\bust}{\bm{u}_{\star}}

\newcommand{\bw}{\bm{w}}

\newcommand{\calV}{\mathcal{V}}

\newcommand{\fst}{f_{\star}}
\newcommand{\Fclass}{\mathscr{F}}

\renewcommand{\what}{\widehat{\bm{w}}}
\newcommand{\bLamtil}{\widetilde{\bm{\Lambda}}}
\newcommand{\bthetast}{\btheta_{\star}}
\newcommand{\bthetahat}{\widehat{\btheta}}
\newcommand{\bu}{\bm{u}}
\newcommand{\bthetatil}{\widetilde{\btheta}}
\newcommand{\betaw}{\beta_{\bw}}
\newcommand{\Fclassmdp}{\Fclass_{\mathrm{mdp}}}
\newcommand{\Clog}{C_{\mathrm{log}}}
\newcommand{\bwpi}{\bw^{\pi}}
\newcommand{\Vstval}{\Vst_1}
\newcommand{\bmu}{\bm{\mu}}
\newcommand{\Rclass}{\mathscr{R}}
\newcommand{\dR}{d_{\Rclass}}
\newcommand{\RR}{R_{\Rclass}}

\usepackage{wrapfig}

\newlength\tindent
\allowdisplaybreaks

\makeatletter
\providecommand\theHALG@line{\thealgorithm.\arabic{ALG@line}}
\makeatother

\title{First-Order Regret in Reinforcement Learning with Linear Function Approximation: A Robust Estimation Approach}

\author{Andrew Wagenmaker\footnote{University of Washington, Seattle. Email: \texttt{ajwagen@cs.washington.edu}} \and Yifang Chen\footnote{University of Washington, Seattle. Email: \texttt{yifangc@cs.washington.edu}} \and Max Simchowitz\footnote{CSAIL, MIT. Email: \texttt{msimchow@mit.edu}} \and Simon S. Du\footnote{University of Washington, Seattle. Email: \texttt{ssdu@cs.washington.edu}} \and Kevin Jamieson\footnote{University of Washington, Seattle. Email: \texttt{jamieson@cs.washington.edu}}}

\date{October 20, 2022}

\begin{document}

\maketitle

\begin{abstract}

Obtaining first-order regret bounds---regret bounds scaling not as the worst-case but with some measure of the performance of the optimal policy on a given instance---is a core question in sequential decision-making. While such bounds exist in many settings, they have proven elusive in reinforcement learning with large state spaces. In this work we address this gap, and show that it is possible to obtain regret scaling as $\widetilde{\mathcal{O}}(\sqrt{d^3 H^3 \cdot V_1^\star \cdot K} + d^{3.5}H^3\log K )$ in reinforcement learning with large state spaces, namely the linear MDP setting. Here  $V_1^\star$ is the value of the optimal policy and $K$ is the number of episodes. We demonstrate that existing techniques based on least squares estimation are insufficient to obtain this result, and instead develop a novel robust self-normalized concentration bound based on the robust Catoni mean estimator, which may be of independent interest.
\end{abstract}

\newcommand{\bmy}{\bm{y}}
\section{Introduction}
A central question in reinforcement learning (RL) is understanding precisely how long an agent must interact with its environment before learning to behave near-optimally. One popular way to measure this duration of interaction is by studying the \emph{regret} $\cR_{K}$, or cumulative suboptimality, of online reinforcement algorithms that explore an unknown environment across $K$ episodes of interaction. Typical regret guarantees scale as $\cR_K \le \cO(\sqrt{\poly(d,H) \cdot K})$, where $d$ measures the ``size'' of the environment and $H$ the horizon length of each episode. 

In many cases, however, regret bounds scaling at least as large as $\Omega(\sqrt{K})$ may be deeply unsatisfactory. Consider, for example, an environment where the agent receives rewards only at very hard-to-reach states; that is, states which can only be visited with some small probability $p \ll 1$. In this case, the maximal cumulative reward, optimal cumulative expected-reward, or \emph{value} $\Vstval$ will also be quite small. In other words, the cost of making a ``mistake'' at any given episode results in a loss of at most $\Vstval$ reward, and the cumulative loss associated with, say $\sqrt{K}$, mistakes, should also scale with this maximal penalty.

Motivated by this observation, there has been much recent interest in achieving so-called small-value, small-loss, or ``first-order'' regret bounds, which scale in proportion to $\Vstval$: $\cR_K \le \cO(\sqrt{\Vstval \cdot \poly(d,H) \cdot K})$ (it is well know that the the scaling $\sqrt{\Vstval K}$ is unimprovable in general, even in simple settings). Bounds of this form have received considerable attention in the online learning, bandits, and contextual bandits communities, and were responsible for initiating the study of a broad array of \emph{instance-dependent} regret bounds in tabular (i.e. finite-state, finite-action) RL settings as well.

\paragraph{First-Order Regret Beyond Tabular RL.} Though first-order regret has been achieved in both non-dynamic environments (e.g. contextual bandits) and in dynamic environments with finite state spaces (tabular RL) \citep{zanette2019tighter,foster2021efficient}, extension to reinforcement learning in large state and action spaces has proven elusive. The main difficulty is that, even though the cumulative expected value of any policy is bounded as $\Vstval$, the value-to-go associated with starting at some state $s_h$ at step $h$, denoted $V_h^\star(s_h)$, may be considerably larger. Again, the paradigmatic example is when the reward is equal to $1$ on a handful of very hard-to-reach states. This means that the variance of any learned predictor of the value function $V_h^\star(s_h)$ may also be highly nonuniform in the state $s_h$. In the RL setting, this becomes more challenging because the distribution across states evolves as the agent refines its policies. And while in tabular settings, one can address the non-uniformity by reasoning about each of the finitely-many states separately, there is no straightforward way to generalize the argument to larger state spaces. 

\paragraph{Contributions and Techniques.} In this paper, we provide first-order regret bounds for reinforcement learning in large state spaces, the first of their kind in this setting. Our results focus on the setting of MDPs with linear function approximation \citep{jin2020provably}, where the transition operators are described by linear functions in a known, $d$-dimensional featurization of a potentially infinite-cardinality state space. In this setting, we achieve the following regret bound.
\begin{thm}[Informal]
Our proposed algorithm, \algname, achieves the following first-order regret bound with high probability: $\cR_{K} \le \cOtil(\sqrt{d^3 H^3 \cdot \Vstval \cdot K} + d^{3.5}H^3 \log K )$. 
\end{thm}
To our knowledge, \algname is the first algorithm to achieve first-order regret for RL in large state spaces. Our algorithm builds on the \textsc{LSVI-UCB} algorithm of \citep{jin2020provably} for worst-case (non-first-order) regret in linear MDPs. \textsc{LSVI-UCB} relies on solving successive linear regression problems to estimate the Bellman-backups of optimistic overestimates of the optimal value function. In that work, the analysis of the regression estimates relies on a so-called ``self-normalized martingale'' inequality for online least squares---a powerful tool which quantifies the refinement of a ridge-regularized least-squares estimator under an arbitrary sequence of regression covariates $\bm{\phi}_t$ to targets $y_t$ satisfying $\Exp[y_t \mid \bm{\phi}_t] = \langle \bm{\phi}_t, \bm{\theta}_{\star} \rangle$, and under the assumption of sub-Gaussian noise. 
This tool has seen widespread application not only in linear RL, but in bandit and control domains as well \citep{abbasi2011improved,sarkar2019near}.

In the tabular RL setting, first-order regret bounds can be obtained by applying Bernstein-style concentration bounds, which allows the exploration level to adapt to the underlying problem difficulty. Towards achieving first-order regret in linear RL, we might hope that a similar approach could be used, and that developing variance-aware or Bernstein-style self-normalized bounds may provide the necessary refinements. A second challenge arises in the RL setting, however, since, as mentioned, the ``noise'' is inherently heteroscedastic (i.e., the noise variance changes with time)---the variance of $y_t$ depends on $\bphi_t$. Thus, not only do we require a variance-aware self-normalized bound, but such a bound must be able to handle heteroscedastic noise as well.

The recent work of \cite{zhou2020nearly} addresses both of these issues---proposing a Bernstein-style self-normalized bound, and overcoming the heteroscedasticity by relying on a weighted least-squares estimator which normalizes each sample by its variance. A naive application of these techniques, however, results in a scaling of $1/\sigminb$ in the regret bound, where $\sigminb$ is the \emph{minimum noise variance across time}. While this dependence can be reduced somewhat, ultimately, it could be prohibitively large, and prevents us from achieving a first-order regret bound in the case when $\Vst_1$ is small.


The $1/\sigminb$ dependence arises because, if we normalize by the variance in our weighted least-squares estimate, the normalized ``noise'' has magnitude, in the worst case, of $\cO(1/\sigminb)$. In other words, we are paying for the ``heavy tail'' of the noise, rather than simply its variance. Obtaining concentration independent of such heavy tails is a problem well-studied in the robust statistics literature. Towards addressing this difficulty in the RL setting, we take inspiration from this literature, and propose applying the \emph{robust Catoni estimator} \citep{catoni2012challenging}. In particular, we develop a novel self-normalized version of the Catoni estimator, as follows.

\begin{prop}[Self-Normalized Heteroscedastic Catoni Estimation, Informal]\label{prop:informal:self_norm:no_function_approx} 
Given observations $y_t = \langle \bthetast, \bphi_t \rangle + \eta_t$ with $\Exp[y_t \mid \bphi_t] = \innerb{\bphi_t}{\bthetast}$, $\Exp[|\eta_t|^2 \mid \bphi_t] < \infty$, and $|\eta_t| < \infty$ with probability 1, let $\catonii[ \bv ]$ denote a Catoni estimate  of $\bthetast$ in direction $\bv$ from the observed data. Then, with high probability, for all $\bv$ simultaneously:
\begin{align*}
\left| \catonii\left[ \bv\right] - \bv^\top \btheta_{\star} \right | \lesssim   \|\bv\|_{\bLambda_T^{-1}} \cdot \left ( \sqrt{\log 1/\delta + d \cdot \Clog} + \sqrt{\lambda} \| \bthetast \|_2 \right ) + \text{\normalfont (lower order term)}. 
\end{align*}
where $\bLambda_T = \lambda I + \sum_{t=1}^T \sigma_{t}^{-2}\bphi_t \bphi_t^\top$, $\sigma_t^2$ is an upper bound on $\Exp[y_t^2 \mid \bphi_t]$, $\Clog$ is logarithmic in problem parameters, and the {\normalfont `lower order term'} can be made as small as $T^{-q}$ for any constant $q > 0$.
\end{prop}

To apply \Cref{prop:informal:self_norm:no_function_approx}, we take $\bphi_{h,k} = \bphi(s_{h,k},a_{h,k})$ as the features, and $y_{h,k} = V^k_{h+1}(s_{h+1,k})$ as the targets, where $V^k_{h+1}(\cdot)$ is an optimistic overestimate of the value function. In particular, $V^k_{h+1}(\cdot)$ depends on, and may be correlated with, past data. Following \cite{jin2020provably}, we address this issue by establishing an error bound which holds uniformly over possible value functions $V^k_{h+1}(\cdot)$. We call this guarantee the `Heteroscedastic Self-Normalized Inequality with Function Approximation', and state it formally in \Cref{sec:catoni_summary}. The proof combines  \Cref{prop:informal:self_norm:no_function_approx} with  a careful covering argument, which (unlike past approaches based on standard ridge-regularized least squares) requires a novel sensitivity analysis of the Catoni estimator.


\section{Related Work}

\paragraph{Worst-Case Regret Bounds in Tabular RL.}
A significant amount of work has been devoted to obtaining worst-case optimal bounds in the setting of tabular RL
\citep{kearns2002near,kakade2003sample,azar2017minimax,dann2017unifying,jin2018q,dann2019policy,wang2020long,zhang2020almost,zhang2020reinforcement}. These approaches fall into both the model-based \citep{azar2017minimax,dann2017unifying} as well as the model-free category \citep{jin2018q}. While the exact bounds differ, they all take the form $\cOtil(\sqrt{\poly(H) \cdot SAK} + \poly(S,A,H))$. Recently, several works have focused on obtaining bounds that only scale logarithmically with the horizon, $H$, in the setting of time-invariant MDPs with rewards absolutely bounded by 1. \cite{zhang2020reinforcement} answers the question of whether horizon-free learning is possible by proposing an algorithm with regret scaling as $\cOtil(\sqrt{SAK} + S^2A)$---independent of polynomial dependence on $H$. This is known to be worst-case minimax optimal.

\paragraph{RL with Function Approximation.}
In the last several years, there has been an explosion of interest in the RL community in obtaining provably efficient RL algorithms relying on function approximation. An early work in this direction, \cite{jiang2017contextual}, considers general function classes and shows that MDPs having small ``Bellman rank'' are efficiently learnable. Several recent works have extended their results significantly \citep{du2021bilinear,jin2021bellman}. In the special case of linear function approximation, a vast body of recent work exists \citep{yang2019sample,jin2020provably,wang2019optimism,du2019good,zanette2020frequentist,zanette2020learning,ayoub2020model,jia2020model,weisz2021exponential,zhou2020nearly,zhou2021provably,zhang2021variance,wang2021exponential}. A variety of assumptions are made in these works, and we highlight two of them in particular. First, the linear MDP model of \cite{jin2020provably}, which is the setting we consider in this work, assumes the transition probabilities and reward functions can both be parameterized as a linear function of a feature map. Second, the linear mixture MDP setting of \citep{jia2020model,ayoub2020model,zhou2020nearly} makes no linearity assumption on the reward function, but assumes that the transition probabilities are the linear parameterization of $d$ known transition kernels. Notably, the linear MDP assumption has infinite degrees of freedom, and as such model-free approaches are more appropriate, while the linear mixture MDP setting has only $dH$ degrees of freedom, making model-based learning effective. 

As mentioned above, of note in the linear function approximation literature is the work of \cite{zhou2020nearly}, which proposes an algorithm with regret scaling as $\cOtil(\sqrt{\left(d^2 H^3+dH^3\right) K})$, which they show is minimax optimal when $d \ge H$. Their result relies on a Bernstein-style self-normalized confidence bound. While they show that the variance dependence of the Bernstein bound allows them to achieve minimax optimality, as noted, it is insufficient to achieve a first-order bound, motivating our use of the Catoni estimator.

\paragraph{First-Order and Problem-Dependent Regret Bounds in RL.}
The RL community has tended to pursue two primary directions towards obtaining problem-dependent regret bounds. The first is the aforementioned first-order bounds, the focus of this work. To our knowledge, the only work in the RL literature to obtain first-order regret is that of \cite{zanette2019tighter}, which only holds in the tabular setting. \cite{zanette2019tighter} obtain several different forms of such a bound, showing that their algorithm, \euler, has regret which can be bounded as either
\begin{align*}
\cR_K \le \cOtil(\sqrt{\bbQ^* SAHK}) \quad \text{or} \quad \cR_K \le \cOtil(\sqrt{\cG^2 SAK})
\end{align*}
where $\bbQ^* = \max_{s,a,h} \left ( \Var [R_h(s,a)] + \Var_{s' \sim P_h(\cdot | s,a)} [\Vst_{h+1}(s')] \right )$ and $\cG$ is a deterministic upper bound on the maximum attainable reward on a single trajectory for any policy $\pi$: $\sum_{h=1}^H R(s_h, \pi_h(s_h)) \le \cG$. A subsequent work, \cite{jin2020reward}, showed that a slight modification to the analysis of \euler allows one to obtain regret of\footnote{Note that this result was shown for an MDP where the reward function was non-zero only at a single $(s,h)$. Their analysis can be extended to arbitrary reward functions, however, though extra $H$ factors will be incurred.}
\begin{align*}
\cR_K \le \cOtil(\sqrt{SAH \cdot \Vstval K}) . 
\end{align*}

A second approach to instance-dependence, taken by \citep{simchowitz2019non,xu2021fine,dann2021beyond}, seeks to obtain regret scaling with the suboptimality gaps. This yields regret bounds of the form $\cO \left ( \sum_{s,a,h : \Delta_h(s,a) > 0} \frac{\poly(H) \cdot \log K}{\Delta_h(s,a)} \right )$
where $\Delta_h(s,a) := \Vst_h(s) - \Qst_h(s,a)$ is the suboptimality of playing action $a$ in state $s$ at step $h$. While these works consider only the tabular setting, recently \cite{he2021logarithmic} obtained regret in the linear MDP setting of $\cO(\frac{d^3 H^5 \log(K)}{\Delta_{\min}})$ and in the linear mixture MDP setting of $\cO(\frac{d^2 H^5 \log^3(K)}{\Delta_{\min}})$, where $\Delta_{\min}$ is the minimum non-zero gap in the MDP. Gap-dependent regret bounds allow for a characterization of the regret in terms of fine-grained problem-dependent quantities. However, they typically capture the \emph{total regret incurred to solve the problem}, and are therefore overly pessimistic over shorter time horizons.

\paragraph{First-Order Regret Beyond RL.}
A significant body of literature exists towards obtaining first-order regret bounds in settings other than RL. This work spans areas as diverse as statistical learning \citep{vapnik1971uniform,srebro2010smoothness}, online learning \citep{freund1997decision,auer2002adaptive,cesa2007improved,luo2015achieving,koolen2015second,foster2015adaptive}, and multi-armed bandits, adversarial bandits, and semibandits  \citep{allenberg2006hannan,hazan2011better,neu2015first,lykouris2018small,wei2018more,bubeck2020first,ito2020tight}.

We highlight in particular the work in the contextual bandit setting. A COLT 2017 open problem \citep{agarwal2017open} posed the question of obtaining first-order bounds for contextual bandits to the community, which \cite{allen2018make} subsequently addressed by obtaining a computationally inefficient algorithm achieving this. \cite{foster2021efficient} built on this, showing that it is possible to achieve such a bound with a computationally efficient algorithm. While \cite{foster2021efficient}  considers function approximation, their regret bound scales with the number of actions, and is therefore not applicable to large action spaces.

\paragraph{Robust Mean Estimation.}
Our algorithm critically relies on robust mean estimation to obtain concentration bounds that avoid large lower-order terms. We rely in particular on the Catoni estimator, first proposed in \cite{catoni2012challenging}. While the original Catoni estimator assumes i.i.d. data, \cite{wei2020taking} show that a martingale version of Catoni is possible, which is what we apply in this work. We remark that several applications of the Catoni estimator to linear bandits have been proposed recently \citep{camilleri2021high,lee2021achieving}. We refer the reader to the survey \cite{lugosi2019mean} for a discussion of other robust mean estimators.


\section{Preliminaries} 
\newcommand{\roundtwo}{\mathsf{rnd}_2}
\paragraph{Notation.}
All logarithms are base-$e$ unless otherwise noted. We let $\logterm(x_1,x_2,\dots,x_n) := \sum_{i=1}^n \log  (e+x_i)$ denote a term which is at most logarithmic in arguments $x_1,x_2,\dots,x_n \ge 0$. We let $\Ball^d(R) := \{\bx \in \R^d:\|\bx\|\le R\}$ denote the ball of radius $R$ in $\R^d$, and specialize $\Ball^d := \Ball^d(1)$ to denote the unit ball. $\cS^{d-1}$ denotes the unit sphere in $\R^d$. We use $\lesssim$ to denote inequality up to absolute constants, $\cO(\cdot)$ to hide absolute constants and lower-order terms, and $\cOtil(\cdot)$ to hide absolute constants, logarithmic terms, and  lower-order terms. Throughout, we let bold characters refer to vectors and matrices and standard characters refer to scalars. 

We also highlight MDP-specific notation; see below for further exposition. We let $s_{h,k}$ and $a_{h,k}$ denote the state and action at step $h$ and episode $k$, and denote features and rewards $\bphi_{h,k} := \bphi(s_{h,k},a_{h,k})$, $r_{h,k} := r_h(s_{h,k},a_{h,k})$. $\pi^k$ denotes the policy played at episode $k$. We use $\cF_{h,k}$ to denote the $\sigma$-field $\sigma( \cup_{h'=1}^H \cup_{k'=1}^{k-1} \{ (s_{h',k'},a_{h',k'}) \} \cup_{h'=1}^{h} \{ (s_{h',k},a_{h',k}) \})$, so that $\bphi_{h,k}$ is $\cF_{h,k}$-measurable. We will let $\Exp_h[V](s,a) = \Exp_{s' \sim P_h(\cdot | s,a)}[V(s')]$, so $\Exp_h[V](s,a)$ denotes the expected next-state value of $V$ given that we are in state $s$ and play action $a$ at time $h$.

\subsection{Markov Decision Processes}
We consider finite-horizon, episodic Markov Decision Processes (MDPs) with time inhomogeneous transition kernel. An MDP is described by a tuple $(\cS,\cA,H,\{ P_h \}_{h=1}^H, \{ r_h \}_{h=1}^H)$, with $\cS$ the set of states, $\cA$ the set of actions, $H$ the horizon, $P_h : \cS \times \cA \rightarrow \simplex(\cS)$ the probability transition kernel at time $h$, and $r_h : \cS \times \cA \rightarrow [0,1]$ the reward function. We assume that $\{ P_h \}_{h=1}^H$ is initially unknown to the learner, but that $r_h$ is deterministic and known. Without loss of generality, we further assume the intial state $s_1$ is deterministic.
\iftoggle{arxiv}
{
	
}
{}
At each episode, the agents begins in state $s_1$; then for each time step $h \ge 1$, an agent in state $s_h$ takes action $a_h$, receives reward $r_h(s_h,a_h)$ and transitions to state $s'$ with probability $P_h(s'|s_h,a_h)$. This process continues for $H$ steps, at which point the MDP resets and the process repeats.

A policy  $\pi :  \cS \times [H] \rightarrow \simplex(\cA)$ is a mapping from states to distributions over actions. For deterministic policies  ($\forall h,s, ~\pi_h(s)$ is supported on only 1 action) we let $\pi_h(s)$ denote the unique action in the support of the distribution $\pi_h(s)$. To an agent playing a policy $\pi$, at step $h$ they choose an action $a_h \sim \pi_h(s_h)$. We let $\Exp_{\pi}[\cdot]$ denote the expectation over the joint distribution trajectories $(s_1,a_1,\dots,s_H,a_H)$ induced by policy $\pi$. 


\iftoggle{arxiv}
{
\paragraph{Value Functions.} 
Given a policy $\pi$, the $Q$-\emph{value function} for policy $\pi$  is defined as follows:
\begin{align*}
\Qpi_h(s,a) := \Exp_\pi \left [ \sum_{h'=h}^H r_{h'}(s_{h'},a_{h'}) | s_h = s, a_h = a \right ]  .
\end{align*}
In words, $\Qpi_h(s,a)$ denotes the expected reward we will acquire by taking action $a$ in state $s$ at time $h$ and then playing $\pi$ for all subsequent steps. We also denote the \emph{value function} by $\Vpi_h(s) = \Exp_{a \sim \pi_h(s)}[\Qpi_h(s,a)]$, which corresponds to the expected reward we will acquire by playing policy $\pi$ from state $s$ at time $h$. 
 The $Q$-function satisfies the Bellman equation:
\begin{align*}
\Qpi_h(s,a) = r_h(s,a) + \Exp_h[\Vpi_{h+1}](s,a) .
\end{align*} 
}
{
	The the $Q$-\emph{value function} for policy $\pi$  is defined as $\Qpi_h(s,a) := \Exp_\pi \left [ \sum_{h'=h}^H r_{h'}(s_{h'},a_{h'}) | s_h = s, a_h = a \right ]$, and the \emph{value function} by $\Vpi_h(s) = \Exp_{a \sim \pi_h(s)}[\Qpi_h(s,a)]$. 
}
We denote the optimal $Q$-value function by $\Qst_h(s,a) = \sup_\pi \Qpi_h(s,a)$, the optimal value function by $\Vst_h(s) = \sup_\pi \Vpi_h(s)$, and the optimal policy by $\pist$. We define $\Vpi_{H+1}(s) = \Qpi_{H+1}(s,a) = 0$ for all $s$ and $a$. Finally, note that we  always have that $\Qpi_h(s,a) \le H$, for all $\pi,h,s,a$, since we collect a reward of at most 1 at every step. 

\paragraph{Episodic MDPs and Regret.}
In this paper, we study minimizing the \emph{regret} over $K$ episodes of interaction. At each episode $k$, the learning agent selects a policy $\pi^k$, and receives a trajectory $(s_{1,k},a_{1,k},\dots,s_{H,k},a_{H,k})$. Again, the transition kernels $(P_h)_{h=1}^H$ are \emph{unknown} to the learner, whereas (as discussed above), the reward function is known. The regret is defined as the cumulative suboptimality of the learner's policies:
\iftoggle{arxiv}
{
	\begin{align*}
\cR_K = \sum_{k=1}^K \left [ \Vst_1(s_1) - V_1^{\pi_k}(s_1) \right ] .
\end{align*}
}
{
\begin{align*}
\cR_K = \textstyle\sum_{k=1}^K \left [ \Vst_1(s_1) - V_1^{\pi_k}(s_1) \right ] .
\end{align*}
}
As $s_1$ is fixed, we will denote the \emph{value} of policy $\pi$ as $V_1^\pi := V_1^{\pi}(s_1)$. Using this notation we can express the regret as $\cR_K = \sum_{k=1}^K [ \Vstval - V_1^{\pi_k}]$.

\subsection{Reinforcement Learning with Linear Function Approximation}
In the tabular RL setting, it is assumed that $|\cS|$ and $|\cA|$ are both finite. This assumption is quite limited in practice, however, and is not able to model real-world settings where the state and action spaces may be infinite. Towards relaxing this assumption, we consider the linear MDP setting of \cite{jin2020provably}, which allows for infinite state and action spaces. In particular, this setting is defined as follows.

\begin{defn}[Linear MDPs]\label{defn:linear_mdp}
We say that an MDP is a $d$-\emph{dimensional linear MDP}, if there exists some (known) feature map $\bphi(s,a) : \cS \times \cA \rightarrow \R^d$ and $H$ (unknown) signed measures $\bmu_h \in \R^d$ over $\cS$ such that:
\begin{align*}
P_h(\cdot | s,a) = \innerb{\bphi(s,a)}{\bmu_h(\cdot)}. 
\end{align*}
We will assume that $\| \bphi(s,a) \|_2  \le 1$ for all $s,a$,  and $\| |\bmu_h|(\cS) \|_2  = \| \int_{s \in \cS} |\mathrm{d}\bmu_h(s)| \|_2 \le \sqrt{d}$.
\end{defn}


Note that, unlike the standard definition of linear MDPs which assumes that the reward is also linear, $r_h(s,a) = \innerb{\bphi(s,a)}{\btheta_h}$, we consider more general possibly non-linear (though bounded) reward functions. To accommodate this change we must assume that the reward is deterministic and known to the learner. We also consider time-varying reward in the appendix, and in the subsequent section remark on how unknown rewards can be accommodated, if we assume they are linear. 

\iftoggle{arxiv}
{
	We further note that there cannot exist an $(s,a)$ for which $\bphi(s,a) = \bm{0}$, for otherwise \Cref{defn:linear_mdp} would imply that $P_h(\cdot | s,a)$ is not a valid distribution. As shown in \cite{jin2020provably}, the linear MDP setting includes tabular MDPs, while also encompassing more general, non-tabular settings, for example where the feature space corresponds to the $d$-dimensional simplex. A key property of linear MDPs is the following.

\begin{lem}[Lemma 2.3 of \cite{jin2020provably}]\label{lem:q_fun_linear}
For a linear MDP and any policy $\pi$, there exists some set of weights $\{ \bwpi_h \}_{h=1}^H$ such that $\Qpi_h(s,a) = \innerb{\bphi(s,a)}{\bwpi_h}$ for all $s,a,h$. 
\end{lem}

\Cref{lem:q_fun_linear} motivates us to consider linear policy classes in developing our algorithm. More generally, the linear structure of the MDP implies that $\Exp_h[V](s,a) = \innerb{\bphi(s,a)}{\bw_V}$ for some $\bw_V$ and \emph{any arbitrary} function $V: \calS \to \R$.

}
{
	
}

\subsection{Catoni Estimation}
A key tool in our algorithm is the robust Catoni estimator \citep{catoni2012challenging}. The Catoni estimator is defined as follows.

\begin{defn}[The Catoni Estimator] Let $X_1,\dots,X_T$ be a sequence of real-values. The \emph{Catoni robust mean estimator} with parameter $\alpha > 0$, denoted $\catonii_{T,\alpha}$,  is the unique root $z$ of the function 
\iftoggle{arxiv}
{
	\begin{align}\label{eq:fun_catoni}
\fcatoni(z;X_{1:T},\alpha) := \sum_{t=1}^T \psicat(\alpha(X_t - z)),
\end{align}
}
{
	\begin{align}\label{eq:fun_catoni}
\fcatoni(z;X_{1:T},\alpha) := \textstyle\sum_{t=1}^T \psicat(\alpha(X_t - z)),
\end{align}
}

where $\psicat(\cdot)$ is defined by
\iftoggle{arxiv}
{
\begin{align*}
\psicat(y) = \begin{cases} \log (1 + y + y^2) & y \ge 0 \\
-\log(1-y+y^2) & y < 0 \end{cases} . 
\end{align*}
}
{
\begin{align*}
\psicat(y) = \small\begin{cases} \log (1 + y + y^2) & y \ge 0 \\
-\log(1-y+y^2) & y < 0 \end{cases}. 
\end{align*}
}
\end{defn}

The following result illustrates the key property of the Catoni estimator.

\begin{prop}[Theorem 5 of \cite{lugosi2019mean}]\label{lem:vanilla_catoni}
Let $X_1,\ldots,X_T$ be independent, identically distributed random variables with mean $\mu$ and finite variance $\sigma^2 < \infty$. Let $\delta \in (0,1)$ be such that $T \ge 2\log(1/\delta)$. Then the  Catoni mean estimator $\catonii_{T,\alpha}$ with parameter
\iftoggle{arxiv}
{
	\begin{align*}
\alpha = \sqrt{\frac{2 \log 1/\delta}{T \sigma^2 ( 1 + \frac{2 \log 1/\delta}{T - 2 \log 1/\delta} )}} 
\end{align*}
}
{
	\begin{align*}
\alpha = \sqrt{\tfrac{2 \log 1/\delta}{T \sigma^2 ( 1 + \frac{2 \log 1/\delta}{T - 2 \log 1/\delta} )}} 
\end{align*}
}

satisfies the following guarantee with probability $1-2\delta$,
\iftoggle{arxiv}
{
	\begin{align*}
| \catonii_{T,\alpha} - \mu | < \sqrt{\frac{2 \sigma^2 \log 1/\delta}{T - 2 \log 1/\delta}} . 
\end{align*}
}
{
	\begin{align*}
| \catonii_{T,\alpha} - \mu | < \sqrt{\tfrac{2 \sigma^2 \log 1/\delta}{T - 2 \log 1/\delta}} . 
\end{align*}
}

\end{prop}

As \Cref{lem:vanilla_catoni} shows, the Catoni estimator requires only that the second moment of the distribution is bounded to obtain concentration, and has estimation error which scales only with the second moment and independent of other properties of the distribution. We make key use of this result in the following analysis, and state our novel extension of the Catoni estimator to general regression settings in \Cref{sec:catoni_summary}.

\section{First-Order Regret in Linear MDPs}\label{sec:guarantees}
We are now ready to present our algorithm, \algname.

\paragraph{Summary of Key Parameters.} Our algorithm applies the robust Catoni estimator to measure the next-state expectation of the value function. The Catoni estimator requires an estimated upper bound on the value function, which we denote as $\sighat_{h,k}$ and describe in detail below. Throughout, we let $\sigmin = 1/K$ denote a lower floor on these estimates. Using these estimates, we introduce the value-normalized feature covariance, with its regularized analogue 
\begin{align}
\bSigma_{h,k} = \sum_{\tau=1}^{k} \frac{1}{\sighat_{h,\tau}^2} \bphi_{h,\tau} \bphi_{h,\tau}^\top, \quad \bLambda_{h,k} = \lambda I + \bSigma_{h,k} \label{eq:Sigma_Lambda}
\end{align}
For a given $(k,h)$ and direction $\bv$, we use the  above covaraince to define a (directional) Catoni parameter
\begin{align*}
\alpha_{k,h}(\bv) = \min \left \{ \beta \cdot \| \bv \|_{\bSigma_{h,k-1}}^{-1} , \alphamax \right \}, \quad
\end{align*}
where $\beta$ is defined in the \algname pseudocode, and we take $\alphamax = K/\sigmin = K^2$. For a given $k,h$ and direction $\bv \in \R^{d}$, we adopt the  $\alpha_{k,h}(\bv)$ as the Catoni parameter, and set $\catonii_{h,k}[\bv]$ to refer to the associated Catoni estimate on the data
\begin{align*}
X_\tau =  \bv^\top \bphi_{h,\tau}  V_{h+1}^k(s_{h+1,\tau}) / \sighat_{h,\tau}^2, \quad \tau = 1,\ldots,k-1 .
\end{align*}

\begin{algorithm}[h]
\begin{algorithmic}[1]
\State \textbf{input:} confidence $\delta$, number of episodes $K$
\State $\lambda \leftarrow 1/H^2$, $\sigmin \leftarrow 1/K$, $c \leftarrow$ universal constant
\State \iftoggle{arxiv}{$\Kinit \leftarrow c \left ( d^2 \log( \max \{ d, \sigmin^{-1}, K, H \} ) + \log (2HK/\delta) \right )$}{$\Kinit \leftarrow c \left ( d^2 \log( \max \{ d, \sigmin^{-1}, K, H \} ) + \log \tfrac{2HK}{\delta} \right )$}
\State \iftoggle{arxiv}{$\beta \leftarrow 6 \sqrt{c d^2 \log \left ( \max \{ d, \sigmin^{-1}, H,  K \} \right ) + \log(2HK/\delta)}$}{$\beta \leftarrow 6 \sqrt{c d^2 \log \left ( \max \{ d, \sigmin^{-1}, H,  K \} \right ) + \log \tfrac{2HK}{\delta}}$}
\For{$k = 1,2,3,\ldots, K$}
	\For{$h = H,H-1,\ldots,1$}
		\iftoggle{arxiv}
		{
			\If{$k \le \Kinit$}
			\State $\sighat_{h,k-1}^2 \leftarrow 2H^2$
		\Else
			\State $\sighat_{h,k-1}^2 \leftarrow \max \Big \{   20 H \catonii_{h,k-1}[ (k-2) \bLambda_{h,k-2}^{-1} \bphi_{h,k-1}]    + 20 H \beta \| \bphi_{h,k-1} \|_{\bLambda_{h,k-2}^{-1}}$  \label{line:sighat}  \\ \hspace{13em} $+ 20 H \sigmin \beta^2/(k-1)^2  , \sigmin^2 \Big \} $
		\EndIf
		}
		{
		\State{}\textbf{if} $k \le \Kinit$, set $\sighat_{h,k-1}^2 \leftarrow 2H^2$
		\State{}\textbf{else} set $\sighat_{h,k-1}^2$ as in \Cref{eq:sighat} \label{line:sighat}
		}
		
		\State Form \iftoggle{arxiv}{$\bLambda_{h,k-1} \leftarrow \lambda I + \sum_{\tau=1}^{k-1} \frac{1}{\sighat_{h,\tau}^2} \bphi_{h,\tau} \bphi_{h,\tau}^\top$}{$\bLambda_{h,k-1}$} as in \Cref{eq:Sigma_Lambda}
		\Statex \hspace{2.7em} \algcomment{$\Exphat_h[V_{h+1}^k](\bv) := \catonii_{h,k}[(k-1) \bLambda_{h,k-1}^{-1} \bv]$} \label{line:set_qhat}
		\State Compute 
		\iftoggle{arxiv}{\vspace{-0.75em} linear summary  $\what_h^{k} \leftarrow \argmin_{\bw} \sup_{\bv \in \Ball^d \backslash \{ \bm{0} \}} | \innerb{\bv}{\bw} - \Exphat_h[V_{h+1}^k](\bv)|/\| \bv \|_{\bLambda_{h,k-1}^{-1}}$.}{summary $\what_h^{k}$ as in \Cref{eq:what}} \label{line:qhat_lin1}
		\State \iftoggle{arxiv}{$Q_h^{k}(\cdot , \cdot) \leftarrow \min \{ r_h(\cdot,\cdot) + \innerb{\bphi(\cdot ,\cdot)}{\what_h^{k}} + 6\beta \| \bphi(\cdot , \cdot) \|_{\bLambda_{h,k-1}^{-1}} + 12\sigmin \beta^2/k^2, H \}$}{Compute $Q_h^{k}(\cdot , \cdot) $ as in \Cref{eq:Qk_defn}} \label{line:update_q}
		\State $V_h^k(\cdot) \leftarrow \max_a Q_h^k(\cdot,a)$
	\EndFor
	\For{$h=1,2,\ldots,H$}
		\State Play $a_{h,k} = \argmax_a Q_h^{k}(s_{h,k},a)$\iftoggle{arxiv}{, observe $r_{h,k},s_{h+1,k}$}{} \label{line:play_action}
	\EndFor
\EndFor
\end{algorithmic}
\caption{\textbf{F}irst-\textbf{O}rder \textbf{R}egret via \textbf{C}atoni \textbf{E}stimation (\algname)}
\label{alg:catoni_regret}
\end{algorithm}

\paragraph{Algorithm Description.}
\algname proceeds similarly to the \textsc{LSVI-UCB} algorithm of \cite{jin2020provably} by approximating the classical value-iteration update:
\begin{align}\label{eq:value_iteration}
\Qst_h(s,a) \leftarrow r_h(s,a) + \Exp_h[\max_{a'} \Qst_{h+1}(\cdot,a')](s,a)\iftoggle{arxiv}{, \quad \forall s,a.}{.}
\end{align}
It is known that this update converges to the optimal value function. While in practice we cannot evaluate the expectation directly, it stands to reason that an update approximating \eqref{eq:value_iteration} may converge to an approximation of the optimal value function. As in \cite{jin2020provably}, we therefore apply an \emph{optimistic}, empirical variant of the value iteration update, which replaces $\Qst$ with $Q^k$, the optimistic estimate of $\Qst$ at round $k$, and the exact expectation with an empirical expectation. The key difference in our approach as compared to \cite{jin2020provably} is the setting of the optimistic estimate. While \cite{jin2020provably} rely on a simple least-squares estimator to approximate the expectation, we rely on the Catoni estimator. We show that, with high probability:
\iftoggle{arxiv}{
\begin{align*}
 | \catonii_{h,k}[(k-1) \bLambda_{h,k-1}^{-1} \bphi(s,a)]  - \Exp_h[V_{h+1}^k](s,a) | \lesssim \beta \| \bphi(s,a) \|_{\bLambda_{h,k-1}^{-1}} .
\end{align*}
}
{
	\begin{multline*}
 | \catonii_{h,k}[(k-1) \bLambda_{h,k-1}^{-1} \bphi(s,a)]  - \Exp_h[V_{h+1}^k](s,a) | \\
 \lesssim \beta \| \bphi(s,a) \|_{\bLambda_{h,k-1}^{-1}} .
\end{multline*}
}
Thus, setting $\Exphat_h[V_{h+1}^k](\bv)$, our estimate of the next-state expectation, to $\catonii_{h,k}[(k-1)\bLambda_{h,k-1}^{-1} \bv]$ ensures that $\Exphat_h[V_{h+1}^k](\bv)$ 
approximates the expectation in \eqref{eq:value_iteration} for $\bv = \bphi(s,a)$. As discussed in more detail in \Cref{sec:lin_approx_cat_sketch}, instead of using $\Exphat_h[V_{h+1}^k](\bv)$ directly, \Cref{line:qhat_lin1} summarizes $\Exphat_h[V_{h+1}^k](\bv)$ with a linear approximation to it,
\begin{align}
\what_h^{k} \leftarrow \argmin_{\bw} \sup_{\bv \in \Ball^d \backslash \{ \bm{0} \}} \frac{| \innerb{\bv}{\bw} - \Exphat_h[V_{h+1}^k](\bv)|}{\| \bv \|_{\bLambda_{h,k-1}^{-1}}} . \label{eq:what}
\end{align}
This approximation, $\innerb{\bphi(s,a)}{\what_h^k}$, is shown to be an accurate approximation of $\Exphat_h[V_{h+1}^k](\bv)$ in \Cref{lem:cat_lin_approx}, intuitively because the ``ground truth'' is itself linear. Solving the optimization on \Cref{line:qhat_lin1} may be computationally inefficient, so we provide a computationally-efficient modification in \Cref{sec:comp_eff_alg} which has only slightly larger regret.

With the aforementioned linear approximation, we define our optimistic overestimate of the $Q$-function on \Cref{line:update_q} as
\iftoggle{arxiv}
{
\begin{align*}
Q_h^{k}(\cdot,\cdot) \leftarrow \min \{  r_h(\cdot,\cdot) + \innerb{\bphi(\cdot,\cdot)}{\what_h^{k}} + 6\beta \| \bphi(\cdot,\cdot) \|_{\bLambda_{h,k-1}^{-1}} + 12\sigmin \beta^2/k^2, H \},
\end{align*}
}
{
	\begin{multline}\label{eq:Qk_defn}
Q_h^{k}(\cdot,\cdot) \leftarrow \min \Big\{  r_h(\cdot,\cdot) + \innerb{\bphi(\cdot,\cdot)}{\what_h^{k}}  \\
+ 6\beta \| \bphi(\cdot,\cdot) \|_{\bLambda_{h,k-1}^{-1}} + 12\sigmin \beta^2/k^2, H \Big\},
\end{multline}
}
approximating the value-iteration update of \eqref{eq:value_iteration} with additional bonuses to account for the approximation error. 

\paragraph{A Note On Scaling.} To achieve first-order regret, we need both the \emph{errors in our estimates} and the \emph{magnitude of the bonuses} to scale with the magnitude of the value function. To accomplish this, we ensure the bonuses scale with $\| \bphi(s,a) \|_{\bLambda_{h,k-1}^{-1}}$, where $\bLambda_{h,k-1}$ is  the regularized variance-normalized covariance in \eqref{eq:Sigma_Lambda}, and  $\sighat_{h,k-1}^2$ is defined as
\iftoggle{arxiv}{
\begin{equation}\label{eq:sighat}
\sighat_{h,k-1}^2 := \max \Big \{   20 H \catonii_{h,k-1}[ (k-2) \bLambda_{h,k-2}^{-1} \bphi_{h,k-1}] + 20 H \beta (\| \bphi_{h,k-1} \|_{\bLambda_{h,k-2}^{-1}} +\tfrac{\sigmin \beta}{(k-1)^2})  , \,\sigmin^2 \Big \}
\end{equation}
}{
\begin{multline}\label{eq:sighat}\sighat_{h,k-1}^2 := \max \Big \{   20 H \catonii_{h,k-1}[ (k-2) \bLambda_{h,k-2}^{-1} \bphi_{h,k-1}]\\    + 20 H \beta (\| \bphi_{h,k-1} \|_{\bLambda_{h,k-2}^{-1}} +\tfrac{\sigmin \beta}{(k-1)^2})  , \,\sigmin^2 \Big \}
\end{multline}
}
so that, up to effectively lower-order terms accounting for the estimation error,
\begin{align*}
 \Exp_h[(V_{h+1}^{\tau})^2](s_{h,\tau},a_{h,\tau}) \approx \sighat_{h,\tau}^2.
\end{align*}
As we will show, this choice of $\sighat_{h,\tau}^2$ is sufficiently large to ensure our Catoni estimate, $\Exphat_h[V_{h+1}^k](\bv)$, concentrates. At the same time, when $ \Exp_h[(V_{h+1}^{\tau})^2](s_{h,\tau},a_{h,\tau})$ is small for some $\tau$, the variance-normalized regularization $\bLambda_{h,k-1}$ ensures the bonus $\| \bphi(s,a) \|_{\bLambda_{h,k-1}^{-1}}$ is small as well.

\subsection{Formal Regret Guarantee}
We analyze two regret bounds for \algname. In the first bound,  we analyze the description given in \Cref{alg:catoni_regret}, which gives a sharper regret guarantee at the expense of computational inefficiency:
\begin{thm}[Main Regret Bound]\label{lem:regret_bound} Fix a failure probability $\delta \in (0,1)$ and $K \in \N$. Then, the regret of \algname as specified in \Cref{alg:catoni_regret} satisfies the following bound with probability at least $1-3\delta$:
\iftoggle{arxiv}
{
	\begin{align*}
\cR_K \le c_1 \sqrt{d^3 H^3 \Vstval K \cdot \log^3(HK/\delta) } + c_2 d^{7/2} H^3 \log^{7/2}(HK/\delta) 
\end{align*}
for universal constants $c_1,c_2$. 
}
{
	\begin{align*}
\cR_K \lesssim \sqrt{d^3 H^3 \Vstval K \cdot \log^3\tfrac{HK}{\delta} } +  d^{7/2} H^3 \log^{7/2}\tfrac{HK}{\delta}.
\end{align*}
}
\end{thm}

As \Cref{lem:regret_bound} shows, up to lower order terms scaling only polynomially in $d,H,\log K,$ and $\log 1/\delta$, \algname achieves a first-order scaling in its leading order term of $\cO(\sqrt{\Vstval K})$. We sketch the proof of \Cref{lem:regret_bound} in \Cref{sec:regret_sketch} and defer the full proof to \Cref{sec:regret_proofs}. 

\paragraph{Comparison to \cite{jin2020provably}.}
Note that $\Vstval \le H$, since we assume that the reward at each step is bounded by $1$. Thus, \Cref{lem:regret_bound} shows that in the worst case \algname has regret scaling as $\cOtil(\sqrt{d^3 H^4 K})$. This exactly matches the regret of \textsc{LSVI-UCB} given in \cite{jin2020provably}. However, we could have that $\Vstval \ll H$, in which case \algname significantly improves on \textsc{LSVI-UCB}. Note also that the minimax lower bound scales at least as $\Omega(\sqrt{d^2 K})$ \citep{zanette2020learning}---while we do not match this in general, our $d$ dependence does match the $d$-dependence of the best-known computationally efficient algorithm \citep{jin2020provably}.

\paragraph{Extension to Linear Mixture MDPs.}
While we have focused on the linear MDP setting in this work, we believe our techniques and use of the Catoni estimator could be easily extended to obtain first-order regret bounds in the linear mixture MDP setting. As noted, while \cite{zhou2020nearly} achieves nearly minimax optimal regret, their techniques do not easily generalize to obtain a first-order regret bound. We leave extending our method to linear mixture MDPs to future work.

\paragraph{Handling Unknown and Linear Rewards.}
We have assumed that the reward function is known, but that it may be nonlinear. If we are willing to make the additional assumption that the reward is linear, $r_h(s,a) = \innerb{\bphi(s,a)}{\btheta_h}$ for some $\btheta_h$, we can handle unknown reward by modifying the Catoni estimator on \Cref{line:set_qhat} to use the data
\iftoggle{arxiv}
{
	\begin{align*}
X_\tau =  \bv^\top \bphi_{h,t} (r_{h,\tau} + V_{h+1}^k(s_{h+1,\tau})) / \sighat_{h,\tau}^2, \quad \tau = 1,\ldots,k-1,
\end{align*}
}
{
	\begin{align*}
X_\tau =  \tfrac{1}{\sighat_{h,\tau}^2}\bv^\top \bphi_{h,t} (r_{h,\tau} + V_{h+1}^k(s_{h+1,\tau})),\,\tau \in [k-1],
\end{align*}
}
and adding $4 r_{h,k-1}^2$ to $\sighat_{h,k-1}^2$. With this small modification, \algname is able to handle unknown rewards and achieves the same regret as given in \Cref{lem:regret_bound}.

\subsection{Computationally Efficient Implementation}\label{sec:comp_eff_alg}
As noted, \algname is not computationally efficient because it is not clear how to efficiently solve the optimization on \Cref{line:qhat_lin1}\footnote{Note that one could also solve \Cref{line:qhat_lin1} by approximating the $\sup$ over $\Ball^d \backslash \{ \bm{0} \}$ to a $\max$ over a sufficiently-fine $\epsilon$-net of $\Ball^d \backslash \{ \bm{0} \}$. Using standard covering estimates, this would require an exponentially-large-in-$d$ cover of the ball, and thus require computing $2^{\Omega(d)}$ Catoni estimates.}. In this section, we provide a computationally efficient alternative, which only suffers slightly worse regret.

To obtain a computationally efficient variant of \algname, we propose replacing \Cref{line:set_qhat,line:qhat_lin1} with the following update:
\iftoggle{arxiv}
{
\begin{align}\label{eq:efficient_update}
\begin{split}
& \Exphat_h[V_{h+1}^k](\bu_i) \leftarrow \catonii_{h,k}[(k-1) \bLambda_{h,k-1}^{-1} \bu_i] \text{ for } i = 1,\ldots,d \\
& \what_h^k \leftarrow [\bu_1, \ldots, \bu_d] \cdot [ \Exphat_h[V_{h+1}^k](\bu_1), \ldots \Exphat_h[V_{h+1}^k](\bu_d)]^\top
\end{split}
\end{align}
}
{
\begin{align}
&\Exphat_h[V_{h+1}^k](\bu_i) \leftarrow \catonii_{h,k}[(k-1) \bLambda_{h,k-1}^{-1} \bu_i] \text{ for } i = 1,\ldots,d \nonumber \\
& \what_h^k \leftarrow \bm{U} \cdot [ \Exphat_h[V_{h+1}^k](\bu_1), \ldots , \Exphat_h[V_{h+1}^k](\bu_d)]^\top \label{eq:efficient_update}
\end{align}
}
where $\bm{U} = [\bu_1, \ldots, \bu_d] $ denotes the eigenvectors of $\bLambda_{h,k-1}$. This update \emph{is} computationally efficient, as it involves only an eigendecomposition, the computation of $d$ Catoni estimates (which can be computed efficiently), and a matrix-vector multiplication. 
The above approach satisfies the following guarantee:

\begin{thm}[Computationally Efficient Regret Bound]\label{lem:regret_bound_eff}
Consider the variant of \algname with \Cref{line:set_qhat,line:qhat_lin1} replaced by \eqref{eq:efficient_update}, and \iftoggle{arxiv}{the update to}{} $Q_h^k(s,a)$ on \Cref{line:update_q} replaced with
\iftoggle{arxiv}
{
\begin{align*}
Q_h^{k}(\cdot,\cdot) = \min \{ r_h(\cdot,\cdot) +   \innerb{\bphi(\cdot,\cdot)}{\what_h^{k}} + 3(\sqrt{d}+2) \beta \| \bphi(\cdot,\cdot) \|_{\bLambda_{h,k-1}^{-1}} + 3 (\sqrt{d}+2)^2 \sigmin \beta^2/k^2, H \}.
\end{align*}
}
{
	\begin{multline*}
Q_h^{k}(\cdot,\cdot) = \min \{ r_h(\cdot,\cdot) +   \innerb{\bphi(\cdot,\cdot)}{\what_h^{k}} \\
+ 3(\sqrt{d}+2) \beta \| \bphi(\cdot,\cdot) \|_{\bLambda_{h,k-1}^{-1}}  +  \tfrac{3 (\sqrt{d}+2)^2 \sigmin \beta^2}{k^2}, H \}.
\end{multline*}
}
Then with probability at least $1-3\delta$, the regret is at most
\iftoggle{arxiv}
{
\begin{align*}
\cR_K \le c_1 \sqrt{d^4 H^3 \Vstval K \cdot \log^3(HK/\delta) } + c_2 d^4 H^3 \log^{7/2}(HK/\delta) 
\end{align*}
}
{
\begin{align*}
\cR_K \lesssim \sqrt{d^4 H^3 \Vstval K \cdot \log^3 \tfrac{HK}{\delta} } +  d^4 H^3 \log^{7/2}\tfrac{HK}{\delta} 
\end{align*}
}
and computation scales polynomially in  $d,H,K,$ and $\min \{ |\cA|, \cO(2^d) \}$. 
\end{thm}

If we are willing to pay an additional factor of $\sqrt{d}$, it follows that we can run \algname in a computationally efficient manner, assuming $|\cA|$ is small. The dependence on $|\cA|$ seems unavoidable and will be suffered by \cite{jin2020provably} as well, since computing the best action to play, $a_{h,k}$, on \Cref{line:play_action} will require enumerating all possible choices of $a$. If $|\cA|$ is infinite, we can reduce this to only $2^{\cO(d)}$ by covering all possible directions of $\bphi(s_{h,k},a)$, but it is not clear if this can be reduced further in general.


\section{Catoni Estimation in General Regression Settings}\label{sec:catoni_summary}

In this section we develop a set of results that extend the standard Catoni estimator to general martingale and heteroscedastic regression settings. The results presented here are critical to obtaining the first-order regret scaling of \algname. We remark that the results in this section are based on a martingale version of the Catoni estimator first proposed in \cite{wei2020taking}. 

\subsection{Martingale Catoni Estimation}

We begin by formalizing a martingale-linear regression setting in which our bounds (without function approximation) apply. The setting is reminiscent of that considered in \cite{abbasi2011improved}, but with two key generalizations: (a) the targets $y_t$ can be \emph{heavy-tailed}, we only require they have finite-variance, and (b) for each target $y_t$ we have an associated upper bound $\sigma_t^2$ on its \emph{conditional} square expectation. This latter point is crucial for modeling heteroscedastic noise. 

\begin{restatable}[Heteroscedastic Heavy-Tailed Martingale Linear Regression]{defn}{defnmartlinreg}\label{defn:heavy_mart_reg}
Let $(\calF_t)_{t \ge 0}$ denote a filtration, let $\bphi_t \in \R^d$ be a sequence of random $\calF_{t-1}$-measurable vectors satisfying $\| \bphi_t \|_2 \le 1$, and $y_t \in \R$ be $\calF_{t}$-measurable random scalars satisfying
\begin{align*}
y_t = \langle \bphi_t, \btheta_{\star} \rangle + \eta_t, \quad \Exp[y_t | \cF_{t-1}] &= \langle \bphi_t, \btheta_{\star} \rangle
\end{align*}
for some $\btheta_{\star} \in \R^d$ and $\eta_t$ satisfying $\Exp[\eta_t| \cF_{t-1}] = 0$, and $\Exp[ \eta_t^2 | \cF_{t-1}] < \infty$, but otherwise arbitrary (as such, the distribution of $\eta_t$ may depend on $\bphi_t$). Furthermore, let $\sigma_t^2$ be a $\cF_{t-1}$-measurable sequence of scalars satisfying $\sigma_t^2 \ge \Exp[y_t^2 | \cF_{t-1}]$ and $\sigma_t \ge \sigma_{\min}$, and let
\begin{align*}
\bSigma_T := \sum_{t=1}^T \sigma_t^{-2} \bphi_t \bphi_t^\top .
\end{align*}
\end{restatable}

In the regression setting of \Cref{defn:heavy_mart_reg}, our goal is to estimate $\bthetast$ in a particular direction, $\bv \in \R^d$, given observations $\{ (\bphi_t,y_t,\sigma_t) \}_{t=1}^T$. As a warmup, the following lemma bounds certain \emph{directional} Catoni estimates.
\begin{restatable}[Heteroscedastic Catoni Estimator]{lem}{heterodirectcatoni}\label{lem:catoni2}
Assume we are in the regression setting of \Cref{defn:heavy_mart_reg}. For a fixed vector $\bv \in \R^d$ let $\catonii[\bv]$ denote the Catoni estimate applied to $( X_{t})_{t=1}^T$ where $X_t := \bv^\top\bphi_t y_t / \sigma_t^2$, with a fixed (deterministic) parameter $\alpha > 0$.  Then, for any failure probability $\delta > 0$ and fixed $\alphamax > 0$, if our deterministic $\alpha$ can be written as
\begin{align}\label{eq:cat_alpha}
 \alpha = \min \left \{ \gamma \cdot \frac{ \sqrt{ \log (c_T/\delta)}}{\|\bv\|_{\bSigma_T}}, \alphamax \right \}
\end{align}
for some (possibly random) $\gamma \ge 1$ and $c_T := 16 T( 1 +   \| \bv \|_2 \| \bthetast \|_2 \alpha_{\max} / \sigma_{\min}^2 )^2$, then with probability at least $1-\delta$,
\begin{align*}
\left| \catonii[\bv] - \frac{1}{T}\bv^\top \bSigma_T \cdot\btheta_{\star} \right| \le   (2+2\gamma) \|\bv\|_{\bSigma_T} \sqrt{ \frac{ \log \frac{c_T}{\delta}}{T^2}} + \frac{2 \log \frac{c_T}{\delta}}{\alpha_{\max} T}
\end{align*}
provided that $T \ge (2+2\gamma^2)\log \frac{c_T}{\delta}$.
\end{restatable}
Note that the introduction of the slack parameter $\gamma \ge 1$ accounts for the fact that $\alpha$ is assumed to be chosen deterministically, while $\bSigma_T$ is random.  \Cref{lem:catoni2} shows that we can apply the Catoni estimator to estimate $\bthetast$ in a particular direction, with estimation error scaling only with an upper bound on $\Exp[ \eta_t^2 | \cF_{t-1}]$ and independent of other properties of $\eta_t$, such as its magnitude. This is in contrast to Bernstein-style bounds which exhibit lower-order terms scaling with the absolute magnitude of $\eta_t$. \Cref{lem:catoni2} serves as a building block for our subsequent estimation bounds, where the choice of finite $\alphamax$ serves a useful technical purpose.

\subsection{Self-Normalized Catoni Inequality }
Next, we bootstrap \Cref{lem:catoni2} into a full-fledged self-normalized inequality for heteroscedastic noise. To do so, we need to address two technical points:
\begin{itemize}
	\item The ideal choice of $\alpha$ (for which $\gamma$ is close to $1$) is not deterministic, but data-dependent. 
	\item To estimate $\bthetast$ in direction $\bv$, we would like to consider $\catonii[\bvtil]$, where $\bvtil = T \bSigma_T^{-1} \bv$, since then $\frac{1}{T}\bvtil^\top \bSigma_T \cdot\btheta_{\star} = \bv^\top \bthetast$. However, this choice of $\bv$ introduces correlations between $\bv$ and our observations $\{ (\bphi_t,y_t,\sigma_t) \}_{t=1}^T$, which prevents us from applying \Cref{lem:catoni2} directly.
\end{itemize}
We adress both via a uniform-convergence-style argument and argue that a bound of the form given in \Cref{lem:catoni2} holds for \emph{all} $\bv$ simultaneously. This requires a subtle argument to bound the sensitivity of the Catoni estimator, given in \Cref{app:Catoni_perturb}. With this bound in hand, we establish the following \emph{truly heteroscedastic self-normalized concentration inequality}, the formal statement of \Cref{prop:informal:self_norm:no_function_approx} in the introduction.
\begin{cor}[Self-Normalized Heteroscedastic Catoni Estimation]\label{cor:catoni_self_norm_no_function} 
Consider the setting of \Cref{defn:heavy_mart_reg}, and suppose that with probability $1$,  $|\eta_t| \le \beta_{\eta} < \infty$ and $\sigma_t^2 \ge \sigma_{\min}^2 > 0$ for all $t$. For a fixed regularization parameter $\lambda > 0$, define the effective dimension
\begin{align*}
d_T := c \cdot d \cdot \logterm\left(T, \alpha_{\max}^2,  \lambda^{-1},\sigma_{\min}^{-2},\beta_{\eta}, \| \bthetast \|_2 \right).
\end{align*}
Let $\catonii\left[T \bLambda_T^{-1} \bv\right]$ denote the Catoni estimate applied to $(X_t)_{t=1}^T$ and parameter $\alpha$ given by
\begin{align*}
X_t = T \bv^\top \bLambda_T^{-1} \bphi_t y_t / \sigma_t^2, \quad \alpha =   \min \left \{ \sqrt{ \|T \bLambda_T^{-1} \bv \|_{\bSigma_T}^{-2} \cdot (d_T  + \log 1/\delta)},\, \alphamax \right \} 
\end{align*}
and for $\bLambda_T = \lambda I + \sum_{t=1}^T \sigma_{t}^{-2}\bphi_t \bphi_t^\top$. Then, as long as $T \ge 5 ( \log 1/\delta + d_T)$, with probability at least $1-\delta$, for all $\bv \in \Ball^d$ simultaneously, 
\begin{align}\label{eq:cor_cat_bound_one}
\left| \catonii\left[T \bLambda_T^{-1} \bv\right] - \bv^\top \btheta_{\star} \right | \le   5 \|\bv\|_{\bLambda_T^{-1}} \cdot \left ( \sqrt{\log 1/\delta + d_T} + \sqrt{\lambda} \| \bthetast \|_2 \right ) + \frac{3 (\log\frac{1}{\delta}+d_T)}{\alphamax T} .
\end{align}
\end{cor}

In contrast to \Cref{lem:martingale_catoni}, \Cref{cor:catoni_self_norm_no_function} only adds the requirement that $\eta_t$ and $\sigma_{\min}^2$ satisfy probability-one upper and lower bounds, respectively, which enter only logarithmically into our final bound\footnote{In the case when the noise is unbounded, note that, by Chebyshev's inequality, one can just take $\beta_u \le \sqrt{\max_{t} \sigma_t^2/\delta}$, at the expense of at most $\delta > 0$ failure probability, whilst maintaining a logarithmic dependence on $1/\delta$ in the final bound.}. Similarly, the parameter $\alpha_{\max}$ also enters at most logarithmically into the final bound, and hence can also be chosen suitably large to make the second term in \Cref{eq:cor_cat_bound_one} suitably small. Intuitively, $\alpha_{\max}$ ensures that the Catoni estimator is sufficiently robust to perturbation, which is necessary for our uniform convergence arguments.

\Cref{cor:catoni_self_norm_no_function} is a special case of a more general result, \Cref{cor:catoni}, whose statement and proof we detail in the following subsection.  
Up to logarithmic factors our guarantee matches that of \cite{abbasi2011improved}. The key difference is that, whereas  \cite{abbasi2011improved} considers the a norm in a covariance \emph{not weighted by the variance} $ \|\bv\|_{\widetilde{\bLambda}_T^{-1}}$ with $\widetilde{\bLambda}_T := \lambda I +  \sum_{t=1}^T \bphi_t\bphi_t^\top$, our guarantee uses the weighted-covariance norm $\bLambda_T := \lambda I + \sum_{t=1}^T \sigma_{t}^{-2}\bphi_t\bphi_t^\top$. It is clear that the latter is much \emph{larger} when $\sigma_t^2$ are small, leading to a smaller error bound. 

Our bound is similar in spirit to another self-normalized heteroscedastic inequality recently provided by \cite{zhou2020nearly}. The key distinction is that the Catoni estimator lets us obtain estimates that scale with the standard deviation of the noise, $\sigma$, and only logarithmically with the absolute magnitude, $\beta_\eta$. This is in contrast to the bound obtained in \cite{zhou2020nearly}, which scale only with $\sigma$ in the leading order term, but scales with $\beta_\eta$ in the lower order term. In situations where $\beta_\eta$ is large, which will be the case when deriving first-order bounds for linear RL, this scaling could be significantly worse. To make this concrete, the following example illustrates \Cref{cor:catoni_self_norm_no_function} on a simple problem.

\begin{exmp}[Regression with Bounded Noise]\label{exmp:catoni}
Consider the linear regression setting where we receive observations
\begin{align*}
y_t = \innerb{\bphi_t}{\bthetast} + \eta_t
\end{align*}
for some $\cF_{t-1}$-measurable $\bphi_t$, $\| \bthetast \|_2 \le 1$, and noise $\eta_t$ satisfying $\Exp[\eta_t | \cF_{t-1}] = 0$, $\Var[\eta_t \mid \cF_{t-1}] = \sigma^2$, and $| \eta_t | \le \beta_\eta$ almost surely for some $\beta_\eta$. Assume $\sigma^2$ is known and that $|\innerb{\bthetast}{\bphi_t}| \le \epsilon$ for all $t$. Define $\sigma_t^2 = 2(\epsilon^2+\sigma^2)$ for all $t$ and note that
\begin{align*}
\sigma_t^2 = 2(\epsilon^2 + \sigma^2) \ge 2 ( \innerb{\bthetast}{\bphi_t}^2 + \Exp[\eta_t^2 \mid \cF_{t-1}]) = 2 \Exp[y_t^2 \mid \cF_{t-1}] .
\end{align*}
Now take some $\bv \in \R^d$ and consider applying the Catoni estimator to the data 
\begin{align*}
X_t = T \bv^\top \bLambda_T^{-1} \bphi_t y_t / (2\epsilon^2 + 2 \sigma^2), \quad \bLambda_T = \frac{1}{2\epsilon^2 + 2 \sigma^2}  \left (  I + \sum_{t=1}^T  \bphi_t \bphi_t^\top \right )
\end{align*}
and with $\alpha$ set as in \Cref{cor:catoni_self_norm_no_function}. We can then apply \Cref{cor:catoni_self_norm_no_function} to get that, with probability $1-\delta$,
\begin{align*}
\left | \catonii \left [T \bLambda_T^{-1} \bv \right ]  - \bv^\top \bthetast \right | \lesssim \| \bv \|_{\bLambda_T^{-1}} \cdot  \sqrt{ \log 1/\delta + d_T}  + \frac{\log 1/\delta + d_T}{\alphamax T} .
\end{align*}
Note that, given our setting of $\bLambda_T$, we have
\begin{align*}
\| \bv \|_{\bLambda_T^{-1}} =  \sqrt{2\epsilon^2 + 2\sigma^2} \| \bv \|_{(I + \bSigma_T)^{-1}}
\end{align*}
and we can set $\alphamax = T$, $\sigminb = 1/T$, so $d_T = \cO(d \cdot \logterm(T,\sigma^2,\beta_\eta) )$. We conclude that
\begin{align*}
\left | \catonii \left [T \bLambda_T^{-1} \bv \right ]  - \bv^\top \bthetast \right | & \lesssim \| \bv \|_{(I + \bSigma_T)^{-1}} \cdot (\epsilon + \sigma) \sqrt{ \log 1/\delta + d \cdot \logterm(T,\sigma^2,\beta_\eta)}  + \frac{\log 1/\delta + d \cdot \logterm(T,\sigma^2,\beta_\eta)}{T^2} \\
& = \cOtil \left ( \| \bv \|_{(I + \bSigma_T)^{-1}} \cdot (\epsilon+\sigma) \sqrt{d} \right ) .
\end{align*}
By \Cref{cor:catoni_self_norm_no_function} this holds for all $\bv$ simultaneously. 

In contrast to this, using the same regularization as above, the Bernstein self-normalized bound of \cite{zhou2020nearly} will scale as (hiding logarithmic terms),
\begin{align*}
\left | \bv^\top ( \bthetahat - \bthetast) \right | \le\cOtil \left ( \| \bv \|_{(I + \bSigma_T)^{-1}} \cdot \left ( \sigma \sqrt{d } + \beta_\eta  \right ) \right ) 
\end{align*} 
where $\bthetahat$ denotes the least-squares estimate.
\end{exmp}

\Cref{exmp:catoni} could model, for example, a linear bandit problem where the value of the optimal arm is 0 (which is always achievable by shifting the problem), and we are in the regime where we are playing near-optimally, so that $\innerb{\bthetast}{\bphi_t} \approx 0$. In this regime, $\epsilon \approx 0$, so the dominant scaling will simply be $\cOtil(\| \bv \|_{(I + \bSigma_T)^{-1}} \cdot \sigma \sqrt{d})$.

\subsection{Self-Normalized Catoni Estimation with Function Approximation}

To apply our self-normalized bound in the linear RL setting, we need to allow for regression targets which are potentially \emph{correlated} with the features $\bphi_t$ in a verify specific way. More precisely, the targets $y_t$ take the form $y_t = \langle \bust, \bphi_t \rangle + \fst(\bphi'_t)$ where $\bphi'_t$ is a $\cF_t$-measurable feature vector, and $\fst$ is a function which may \emph{depend on all the data} $\{ (\bphi_t,y_t,\sigma_t) \}_{t=1}^T$. The function $\fst$ is therefore \emph{not} $\cF_{t}$-measurable, and so $y_t$ does not satisfy the condition of \Cref{defn:heavy_mart_reg}. To handle these challenges, we introduce the following regression setting, which specifies the precise conditions needed for our most general result.

\begin{restatable}[Heteroscedastic Regression with Function Approximation]{defn}{regfunapprox}\label{defn:regression_with_function_approx} Given dimension parameters $d,d',p \in \N$, scaling parameters $H,\beta_u,\beta_{\mu} > 0$, and  minimal varaince $\sigma_{\min}^2$, the heteroscedastic regression with function approximation setting is defined as follows.  Let $(\cF_{t})_{t\ge 0}$ be a filtration, and consider a sequence of random vectors $ (\bphi_t, \bphi_t')_{t=1}^T$ and random scalar oututs $(y_t)_{t=1}^T$ and noises $(\eta_t)_{t=1}^T$ and variance bounds $(\sigma_t^2)_{t=1}^T$ such that
\begin{itemize}
\item $\bphi_t \in \R^d$ is $\calF_{t-1}$-measurable,  $\bphi_t' \in \R^{d'}$ is $\calF_{t}$ measurable, and $\|\bphi_t\|_2,\|\bphi_t'\|_2 \le 1$. 
\item There exists a signed measure $\bmu$ over $\Ball^{d'}$ with total mass $\| |\bmu|(\Ball^{d'}) \|_2 \le \beta_{\mu}$ such that, for all $t$, the conditional distribution of $\bphi_t'$ given $\cF_{t-1}$ ensures that, for all bounded functions $f$, 
\begin{align}\label{eq:cat_union_lin_exp}
\Exp\left[f(\bphi_t') \mid \cF_{t-1}\right] =  \innerb{\bphi_t}{\int f(\bphi') \rmd \bmu(\bphi')}.
\end{align}
\item $|\eta_t| \le \beta_{\eta}$ with probability 1, and $\Exp[\eta_t \mid \cF_{t-1}] =0$. 
\item There exist a parameter $\bust \in \R^d$ with $\|\bust\|_2 \le \beta_{u}$, a function class $\Fclass$ of functions $f:\R^{d'}\to [-H,H]$, and a function $\fst \in \Fclass$ which may be random and dependent on  $ (\bphi_t, \bphi_t')_{t=1}^T$  such that, for all $t$, $y_t = \langle \bust, \bphi_t \rangle + \fst(\bphi'_t) + \eta_t$. Thus,
\begin{align*}
\Exp[y_t | \cF_{t-1}] = \innerb{\bphi_t}{\bthetast} \quad \text{for} \quad \bthetast = \bust + \int  \fst(\bphi') \rmd \bmu(\bphi') 
\end{align*}
and $\| \bthetast \|_2 \le \beta_u + H \beta_\mu$.
\item $\sigma_t$ are uniformly lower bounded by $\sigma_{\min}$, finite, $\cF_{t-1}$ measurable, and satisfy
\begin{align}\label{eq:cat_union_sighat}
\Exp \left [ \left(  \inner{\bphi_t}{\bust} + \fst(\bphi'_t)   + \eta_t \right)^2 \mid \cF_{t-1} \right ] \le \frac{1}{2}\sigma_t^2 .
\end{align}
\item The covering numbers of $\Fclass$ are \emph{parameteric}, in the sense that there exists a $p \in \N$ and $R > 0$ such that, for $\epsilon > 0$, the $\epsilon$-covering number of $\Fclass$ in the metric $\dist_{\infty}(f,f') := \sup_{\bphi' \in \Ball^{d'}} |f(\bphi') - f'(\bphi')|$ is bounded as $\covnum(\Fclass,\dist_{\infty},\epsilon) \le p\log(1+\frac{2 R}{\epsilon})$, where $\covnum(\Fclass,\dist_{\infty},\epsilon)$ is the $\epsilon$-covering number of $\Fclass$ in the norm $\dist_{\infty}$. 
\end{itemize}
\end{restatable}

Note that \Cref{defn:regression_with_function_approx} strictly generalizes \Cref{defn:heavy_mart_reg} since we can always choose $\fst = 0$ to be a fixed function, and are left only with the noise $\eta_t$. For this most general setting, we attain the following result:

\begin{thm}[Heteroscedastic Catoni Estimation with Function Approximation]\label{cor:catoni} 
Assume that we are in the setting of \Cref{defn:regression_with_function_approx}. Define 
\begin{align*}
d_T := c \cdot (p + d) \cdot \logterm\left(T, \alpha_{\max}^2,  \lambda^{-1},\sigma_{\min}^{-2},\beta_{\mu},\beta_u,\beta_\eta,R,H\right).
\end{align*}
Let $\catonii\left[T \bLambda_T^{-1} \bv\right]$ denote the Catoni estimate applied to $(X_t)_{t=1}^T$ and parameter $\alpha$ given by
\begin{align*}
X_t = T \bv^\top \bLambda_T^{-1} \bphi_t y_t / \sigma_t^2, \quad \alpha =   \min \left \{ \sqrt{  \| T \bLambda_T^{-1} \bv \|_{\bSigma_T}^{-2} \cdot (d_T +  \log 1/\delta)},\, \alphamax \right \} 
\end{align*}
and for $\bLambda_T = \lambda I + \sum_{t=1}^T \sigma_{t}^{-2}\bphi_t \bphi_t^\top$. Then, as long as $T \ge 6 ( \log 1/\delta + d_T)$, with probability at least $1-\delta$, for all $\bv \in \Ball^d$ simultaneously, 
\begin{align}\label{eq:cor_cat_bound}
\left| \catonii\left[\bLambda_T^{-1} \bv\right] - \bv^\top \btheta_{\star} \right | \le   5 \|\bv\|_{\bLambda_T^{-1}} \cdot \left ( \sqrt{\log 1/\delta + d_T} + \sqrt{\lambda} \| \bthetast \|_2 \right ) + \frac{3 (\log\frac{1}{\delta}+d_T)}{\alphamax T} .
\end{align}
\end{thm}

A couple remarks are in order. First, \Cref{cor:catoni_self_norm_no_function} is just the special case obtained by setting $\fst(\cdot) = 0$ to be the zero function, and sole element of $\Fclass = \{\fst\}$. Second, as will be observed, the assumptions in \Cref{defn:regression_with_function_approx} precisely line up with those required for linear RL. The proof of \Cref{cor:catoni}, detailed in \Cref{app:self_norm_catonii}, follows by applying \Cref{lem:catoni2} and carefully union bounding over the parameter space. It invokes a novel perturbation analysis of the Catoni estimator, given in \Cref{app:Catoni_perturb}, which may be of independent interest. Again, we remark $\alphamax$ can be chosen suitably large that estimation error of the Catoni estimator scales primarily as $\|\bv\|_{\bLambda_T^{-1}} \cdot \sqrt{ \log (1/\delta) + d_T}$.

\subsection{Linear Approximation to the Catoni Estimator}\label{sec:lin_approx_cat_sketch}
In the linear RL setting, we will rely on the Catoni estimator to form an optimistic estimate, $Q_h^k(s,a)$, of $\Qst_h(s,a)$. To construct this estimator, we will set $y_t =  V_{h+1}^k(s_{h+1,t})$---thus, $\fst$ will itself be an optimistic $Q$-value estimate. In order to apply \Cref{cor:catoni} directly to the linear RL setting, we therefore need to cover the space of \emph{all} Catoni estimates. It is not clear how to do this in general without covering all $\cO(d T)$ parameters the Catoni estimator takes as input, which will result in suboptimal $K$ dependence in the final regret bound.

To overcome this challenge, we make the critical observation that \eqref{eq:cor_cat_bound} implies that, up to some tolerance, \emph{there exists a linear function which approximates $\catonii [T \bLambda_T^{-1} \bv ]$ for all $\bv$}, namely $\innerb{\bv}{\bthetast}$. As we do not know $\bthetast$, we cannot compute this function directly. However, the following result shows that we can exploit the fact that there \emph{exists} such a linear approximation in order to come up with our own linear approximation:

\begin{lem}\label{lem:cat_lin_approx}
Let $\catonii[\bLambda^{-1} \bv]$ denote a Catoni estimate, as defined in \Cref{lem:catoni2}. Assume that, for all $\bv \in \cV$ for some $\cV \subseteq \R^d$, $\bm{0} \not\in \cV$, we have
\begin{align}\label{eq:lin_approx_cat_good}
| \catonii[\bLambda^{-1} \bv] - \inner{\bv}{\bthetast} | \le  C_1 \| \bv \|_{\bLambda^{-1}} + C_2/T
\end{align}
for some $C_1,C_2$. Set
\begin{align}\label{eq:cat_find_lin_approx}
\bthetahat = \argmin_{\btheta} \sup_{\bv \in \cV} \frac{| \innerb{\btheta}{\bv} - \catonii[\bLambda^{-1} \bv] |}{\| \bv \|_{\bLambda^{-1}} }.
\end{align}
Then, for all $\bv \in \cV$, we have 
\begin{align*}
| \innerb{\bthetahat}{\bv} - \catonii[\bLambda^{-1} \bv] | \le C_1 \| \bv \|_{\bLambda^{-1}} + C_2/T, \qquad | \innerb{\bthetahat}{\bv} - \inner{\bv}{\bthetast} | \le 2 C_1 \| \bv \|_{\bLambda^{-1}} + 2 C_2/T.
\end{align*}
\end{lem}

Given this result, if we approximate our $Q$-functions by Catoni estimates, $\Exphat_h[V_{h+1}^k](s,a)$, instead of directly using $\Exphat_h[V_{h+1}^k](s,a)$ we can rely on a linear approximation to it, $\innerb{\what_h^k}{\bphi(s,a)}$. By \Cref{lem:cat_lin_approx}, this will be an accurate approximation for \emph{all} $s,a$. As we can easily cover the space of $d$-dimensional vectors, this allows us to cover the space of all of our $Q$-function estimates. As we will see, in practice we rely on optimistic $Q$ functions which also depend on some $\bLambda \succeq 0$, so we will ultimately choose $\Fclass$ in \Cref{defn:regression_with_function_approx} so that $p = \cO(d^2)$. 

Note that solving \eqref{eq:cat_find_lin_approx} is not computationally efficient in general, yet as we described in \Cref{sec:comp_eff_alg} and show in more detail in \Cref{sec:lin_approx_cat}, a linear approximation to a Catoni estimator can be found in a computationally efficient manner if we are willing to pay an extra factor of $\sqrt{d}$ in the approximation error.

\section{Regret Bound Proof Sketch}\label{sec:regret_sketch}
We turn now to applying the Catoni estimation results of \Cref{sec:catoni_summary} in the setting of linear RL. We defer the full proofs to \Cref{sec:regret_proofs}.

\subsection{Failure of Least Squares Estimation}\label{sec:ls_breaks}
We first describe in more detail why least squares estimation is insufficient to obtain first-order regret. Building on \cite{jin2020provably}, our goal in the RL setting will be to construct optimistic estimators, $Q_h^k(s,a)$, to the optimal value function, $\Qst_h(s,a)$, satisfying $Q_h^k(s,a) \ge \Qst_h(s,a)$. \cite{jin2020provably} construct such estimators recursively by applying a least-squares value iteration update and solving
\begin{align*}
\wtil_h^k = \argmin_{\bw \in \bbR^d} \sum_{\tau = 1}^{k-1} \left (   V_{h+1}^{k}(s_{h+1,\tau}) - \bw^\top \bphi_{h,\tau} \right )^2 + \lambda \| \bw \|_2^2 .
\end{align*}
Intuitively, if enough data has been collected, this update will produce a $\wtil_h^k$ which accurately approximates the expectation over the next state. Indeed, \cite{jin2020provably} show that, for any $\pi$,\footnote{In fact, \cite{jin2020provably} uses a slightly different update, including $r_{h,\tau}$ in the regression problem so that $\innerb{\wtil_h^k}{\phi(s,a)}$ estimates the reward and next-state expectation. In contrast, the setting of $\wtil_h^k$ given here estimates only the next-state expectation. \cite{jin2020provably} assume that the reward is unknown and is linear, motivating their inclusion of it in the regression problem. The direct extension of their approach to known but nonlinear reward is the update stated above.}
\begin{align*}
\innerb{\wtil_h^k}{\bphi(s,a)} + r_h(s,a) - \Qpi_h(s,a) = \Exp_h[V_{h+1}^k - \Vpi_{h+1}](s,a) + \xi_h(s,a)
\end{align*}
for some $\xi_h(s,a) \lesssim dH \| \bphi(s,a) \|_{\bLamtil_{h,k-1}^{-1}}$, where $\bLamtil_{h,k-1} = \lambda I + \sum_{\tau = 1}^{k-1} \bphi_{h,\tau} \bphi_{h,\tau}^\top$. Applying this estimator, \cite{jin2020provably} are able to construct a value function guaranteed to be optimistic, and ultimately obtains regret of $\cOtil(\sqrt{d^3 H^4 K})$. This is fundamentally a Hoeffding-style estimator, however, and does not scale with the variance of the next-state value function. As such, it does not appear that tighter regret bounds can be obtained using this approach.

A natural modification of this estimator would be the \emph{weighted least squares estimate}:
\begin{align}\label{eq:ls_var_est}
\wtil_h^k = \argmin_{\bw \in \bbR^d} \sum_{\tau = 1}^{k-1} \left (   V_{h+1}^{k}(s_{h+1,\tau}) - \bw^\top \bphi_{h,\tau} \right )^2/\sighatb_{h,\tau}^2 + \lambda \| \bw \|_2^2 
\end{align}
for $\sighatb_{h,\tau}^2$ an upper bound on $\Var_{s' \sim P_h(\cdot | s_{h,\tau},a_{h,\tau})}[ V_{h+1}^{k}(s')]$. An approach similar to this is taken in the linear mixture MDP setting of \cite{zhou2020nearly}, where it is shown that this approach does indeed yield variance-dependent bounds when a Bernstein-style self-normalized bound is applied. However, as noted, this Bernstein-style bound still scales with the \emph{magnitude} of the ``noise'' in its lower-order term, which here will be of order $H/\sigminb$. 
Carrying their analysis through, we see that the leading order term of the regret is at least on order $(\sqrt{d} + \frac{H}{\sigminb}) \sqrt{\sigminb^2 HK} \ge H \sqrt{HK}$. 
Thus, while this approach may yield an improved $d$ and $H$ dependence, it is unable to obtain a first-order scaling of $\cO(\sqrt{\Vst_1 K})$ when $\Vst_1$ is small.

\subsection{From Catoni Estimation to Optimism}
Note that $\wtil_h^k$ in \eqref{eq:ls_var_est} can be written as 
\begin{align*}
\wtil_h^k =  \sum_{\tau = 1}^{k-1} \bLambda_{h,k-1}^{-1} \bphi_{h,\tau}  V_{h+1}^{k}(s_{h+1,\tau})  /\sighat_{h,\tau}^2 .
\end{align*}
In other words, $\wtil_h^k$ is simply the sample mean. This motivates applying the Catoni estimator to the problem. Indeed, consider setting $\Exphat_h[V_{h+1}^k](s,a) = \catonii_{h,k-1}[(k-1)\bLambda_{h,k-1}^{-1} \bphi(s,a)]$. By \Cref{defn:linear_mdp}, we can set $\bthetast$ in \Cref{defn:regression_with_function_approx} as
\begin{align*}
\bthetast \leftarrow \int  V_{h+1}^k(s') \rmd \bmu_h(s')
\end{align*}
and will have that $\innerb{\bphi(s,a)}{\bthetast} =  \Exp_h[V_{h+1}^k](s,a)$. \Cref{cor:catoni} then immediately gives that, for all $s,a,h,k$,
\begin{align}\label{eq:proof_sketch_good_event}
 | \Exphat_h[V_{h+1}^k](s,a) - \Exp_h[V_{h+1}^k](s,a) | \lesssim ( 1 + H\sqrt{\lambda}) \beta \| \bphi(s,a) \|_{\bLambda_{h,k-1}^{-1}} + \sigmin \beta^2 /k^2
\end{align}
where here $\beta = 6 \sqrt{\log 1/\delta + d_T}$ for $d_T = \cOtil(d + p_{\mathrm{mdp}})$ and $p_{\mathrm{mdp}}$ the covering number of the set of functions $V_{h+1}^k(\cdot)$. Recall that we chose $\alphamax = K/\sigmin$ in the linear RL setting. Thus, the lower-order term of $3 \beta^2/(\alphamax k)$ of \Cref{cor:catoni} can be upper bounded as $3\sigmin \beta^2 /k^2$, as in \eqref{eq:proof_sketch_good_event}.

Given this $\Exphat_h[V_{h+1}^k](s,a)$, let $\what_h^k$ denote the linear approximation to $\Exphat_h[V_{h+1}^k](s,a)$, as described in \Cref{lem:cat_lin_approx}. By \Cref{lem:cat_lin_approx}, it follows that for all $s,a$,
\begin{align}\label{eq:what_rl_sketch}
 | \innerb{\what_h^k}{\bphi(s,a)} - \Exp_h[V_{h+1}^k](s,a) | \lesssim ( 1 + H\sqrt{\lambda}) \beta \| \bphi(s,a) \|_{\bLambda_{h,k-1}^{-1}} + \sigmin \beta^2 /k^2 .
\end{align}

\paragraph{Constructing Optimistic Estimators.}
Fix some $h$ and $k$ and assume that \eqref{eq:what_rl_sketch} holds for all $s,a$. Let 
\begin{align*}
Q_h^k(s,a) = \min \left \{r_h(s,a) + \innerb{\bphi(s,a)}{\what_h^k} + 3 (1 + H\sqrt{\lambda} ) \beta \| \bphi(s,a) \|_{\bLambda_{h,k-1}^{-1}} + 3\sigmin \beta^2 /k^2, H \right \} .
\end{align*}
Assume that $V_{h+1}^k(s)$ is optimistic, that is, $V_{h+1}^k(s) \ge \Vst_{h+1}(s)$ for all $s$. Then \eqref{eq:what_rl_sketch} and this assumption imply that 
\begin{align*}
Q_h^k(s,a) & = \min \left \{ r_h(s,a) + \innerb{\bphi(s,a)}{\what_h^k} + 3 (1 + H\sqrt{\lambda} ) \beta \| \bphi(s,a) \|_{\bLambda_{h,k-1}^{-1}} + 3\sigmin \beta^2 /k^2, H \right \} \\
& \ge  \min \left \{ r_h(s,a) + \Exp_h[V_{h+1}^k](s,a) + 2 (1 + H\sqrt{\lambda} ) \beta \| \bphi(s,a) \|_{\bLambda_{h,k-1}^{-1}} + 2\sigmin \beta^2 /k^2, H \right \} \\
& \ge  \min \left \{ r_h(s,a) + \Exp_h[\Vst_{h+1}](s,a) + 2 (1 + H\sqrt{\lambda} ) \beta \| \bphi(s,a) \|_{\bLambda_{h,k-1}^{-1}} + 2\sigmin \beta^2 /k^2, H \right \} \\
& \ge  \min \left \{ r_h(s,a) + \Exp_h[\Vst_{h+1}](s,a) , H \right \} \\
& = \Qst_h(s,a). 
\end{align*}
In other words, given that $\what_h^k$ accurately approximates the next state expectation, \eqref{eq:what_rl_sketch}, and that $V_{h+1}^k(s)$ is optimistic, it immediately follows that $Q_h^k(s,a)$ is also optimistic.

\paragraph{Defining the Function Class.}
It remains to determine the value of $p_{\mathrm{mdp}}$. Applying the above argument inductively, we see that to form an optimistic estimate, it suffices to consider functions in the set
\begin{align*}
\Fclassmdp = \Big \{ f(\cdot) = \min \{ \innerb{\cdot}{\bw} + \bar{\beta} \| \cdot \|_{\bLambda^{-1}} + \bar{c}, H \} \ : \ \| \bw \|_2 \le \betaw, \ \bLambda \succeq \lambda I \Big \} .
\end{align*}
for some $\bar{\beta},\bar{c},$ and $\betaw$. $\Fclassmdp$ depends on two parameters---the $d$-dimensional $\bw$ and $d \times d$ dimensional $\bLambda$. Thus, using standard covering arguments, it's easy to see that $\covnum(\Fclassmdp,\dist_{\infty},\epsilon) = \cO(d^2 \log(1 + 1/\epsilon ) )$, so it suffices to take $p_{\mathrm{mdp}} = \cO(d^2)$. Given this and the definition of $d_T$, we see that in our setting we will have that $d_T = \cOtil(d^2)$, so $\beta = \cOtil(\sqrt{\log 1/\delta + d^2})$.

\subsection{Proving the Regret Bound}
Henceforth, we will assume that \eqref{eq:proof_sketch_good_event} holds for all $s,a,h,$ and $k$. We turn now to showing how the above results can be used to prove a regret bound. The following lemma, which is a simple consequence of \eqref{eq:proof_sketch_good_event}, will be useful in decomposing the regret.

\begin{lem}[Informal]\label{lem:V_recursion_informal}
Let $\delta_h^k = V_h^{k}(s_h^k) - V_h^{\pi_k}(s_h^k)$ and $\zeta_{h+1}^k = \Exp_{h}[\delta_{h+1}^k ](s_{h,k},a_{h,k}) - \delta_{h+1}^k$. Then, with high probability,
\begin{align*}
\delta_h^k \le \delta_{h+1}^k + \zeta_{h+1}^k + \min \{ 5(1 + H\sqrt{\lambda} ) \beta \| \bphi_{h,k} \|_{\bLambda_{h,k-1}^{-1}} + 5\sigmin \beta^2 /k^2  , H \}. 
\end{align*}
\end{lem}

By definition of $\cR_K$, the optimism of $V_h^k(s)$, and \Cref{lem:V_recursion_informal}, we can bound
\begin{align*}
\cR_K & \le \sum_{k=1}^K (\Vst_1(s_1) - V_1^{\pi_k}(s_1)) \\
& \le \sum_{k=1}^K (V_1^k(s_1) - V_1^{\pi_k}(s_1)) \\
& \lesssim \sum_{k=1}^K \sum_{h=1}^H \zeta_h^k + \sum_{k=1}^K \sum_{h=1}^H  \min \{  (1 + H \sqrt{\lambda} )\beta \| \bphi_{h,k} \|_{\bLambda_{h,k-1}^{-1}} + \sigmin \beta^2 /k^2, H \} .
\end{align*}
$ \sum_{k=1}^K \sum_{h=1}^H \zeta_h^k$ is a martingale-difference sequence and can be bounded using Freedman's Inequality to obtain the desired $\Vst_0$ dependence. In particular, we have, with high probability
\begin{align*}
\sum_{k=1}^K \sum_{h=1}^H \zeta_h^k  \lesssim \sqrt{H^2 \Vstval K \cdot \log 1/\delta} + (\text{lower order terms}).
\end{align*}
In addition, $\sigmin \beta^2 /k^2$ sums to a term that is $\poly(d,H)$, so we ignore it for future calculations. We focus our attention on the term: 
\begin{align*}
\sum_{k=1}^K \sum_{h=1}^H  \min \{  (1 + H \sqrt{\lambda} )\beta \| \bphi_{h,k} \|_{\bLambda_{h,k-1}^{-1}} , H \}
\end{align*}
which can be expressed as:
\begin{align}\label{eq:sketch_reg_decomp1}
\sum_{k=1}^K & \sum_{h=1}^H \sighat_{h,k} \min \{  (1 + H \sqrt{\lambda} )\beta \| \bphi_{h,k}/\sighat_{h,k} \|_{\bLambda_{h,k-1}^{-1}} , H/\sighat_{h,k} \} .
\end{align}
Typically, terms such as this are handled via the Elliptic Potential Lemma. However, to apply the Elliptic Potential Lemma \citep{abbasi2011improved} here, we need to choose $\lambda = 1/\sigmin^2$ to guarantee $\lambda \ge \max_{h,k} \| \bphi_{h,k}/\sighat_{h,k} \|_2^2$. Due to the $\sqrt{\lambda}$ dependence, this will result in a $1/\sigmin$ scaling in the final regret bound, which is prohibitively large. To overcome this, we instead apply the following result, to control the number of times $ \| \bphi_{h,k}/\sighat_{h,k} \|_{\bLambda_{h,k-1}^{-1}}$ can be large:
\begin{lem}\label{lem:elip_pot_bad_event_informal}
Consider a sequence of vectors $( \bx_t)_{t=1}^T, \bx_t \in \bbR^d$, and assume that $\| \bx_t \|_2 \le a$ for all $t$. Let $\bV_t = \lambda I + \sum_{s=1}^t \bx_s \bx_s^\top$ for some $\lambda > 0$. Then, we will have that $\| \bx_t \|_{\bV_{t-1}^{-1}} > b$ at most
\begin{align*}
d \log ( 1 + a^2 T/\lambda)/\log ( 1 + b)
\end{align*}
times. 
\end{lem}
Let $\cK_h = \{ k \ : \  \| \bphi_{h,k}/\sighat_{h,k} \|_{\bLambda_{h,k-1}^{-1}} \le 1 \} $. Then we can bound \eqref{eq:sketch_reg_decomp1} as
\begin{align*}
\eqref{eq:sketch_reg_decomp1} & \lesssim  \sum_{h=1}^H \sum_{k \in \cK_h}  (1 + H \sqrt{\lambda} )\beta \sighat_{h,k} \min \{  \| \bphi_{h,k}/\sighat_{h,k} \|_{\bLambda_{h,k-1}^{-1}} , 1 \} + \sum_{h = 1}^H H | \cK_h^c | \\
& \lesssim  \sum_{h=1}^H \sum_{k \in \cK_h}  (1 + H \sqrt{\lambda} )\beta \sighat_{h,k} \min \{  \| \bphi_{h,k}/\sighat_{h,k} \|_{\bLambda_{h,k-1}^{-1}} , 1 \} + d H^2 \log(1 + K/(\lambda \sigmin^2))
\end{align*}
where the first inequality holds by definition of $\cK_h$, and the second holds by \Cref{lem:elip_pot_bad_event_informal}. By Cauchy-Schwarz, the first term can be bounded as
\begin{align*}
& \lesssim (1 + H \sqrt{\lambda} )\beta \sqrt{\sum_{h=1}^H \sum_{k = 1}^K   \sighat_{h,k}^2} \sqrt{\sum_{h=1}^H \sum_{k=1}^K \min \{  \| \bphi_{h,k}/\sighat_{h,k} \|_{\bLambda_{h,k-1}^{-1}}^2 , 1 \}} .
\end{align*}
As we take the $\min$ over $1$ and $\| \bphi_{h,k}/\sighat_{h,k} \|_{\bLambda_{h,k-1}^{-1}}^2$, regardless of the choice of $\lambda$ we can now apply the Elliptic Potential Lemma to get
\begin{align*}
\sum_{h=1}^H \sum_{k=1}^K \min \{  \| \bphi_{h,k}/\sighat_{h,k} \|_{\bLambda_{h,k-1}^{-1}}^2 , 1 \} \lesssim d H \log(1 + K/(d\lambda \sigmin^2)) .
\end{align*}
Choosing $\lambda = 1/H^2$, we then have that the regret is bounded as
\begin{align*}
\lesssim \beta \sqrt{d H \log(1 + HK/(d \sigmin^2))} \sqrt{\sum_{h=1}^H \sum_{k = 1}^K   \sighat_{h,k}^2} + \poly(d,H,\log K) .
\end{align*}
It remains to bound $\sighat_{h,k}^2$. After some manipulation, and using the definition of $\sighat_{h,k}$ given in \algname, we can bound 
\begin{align*}
\sum_{h=1}^H \sum_{k = 1}^K   \sighat_{h,k}^2 \lesssim H^2 \Vstval K + ( \mathrm{lower \ order \ terms}) .
\end{align*}
Putting this together yields a final regret bound of
\begin{align*}
 \beta \sqrt{d H \log(1 + HK/(d \sigmin^2))} \sqrt{H^2 \Vstval K} + ( \mathrm{lower \ order \ terms}).
\end{align*}
In the proof, slightly more care must be taken with handling $\sighat_{h,k}^2$ to avoid a lower order $K^{1/4}$ term, but we defer the details of this to the appendix.


\section{Conclusion}
In this work we have shown that it is possible to obtain first-order regret in reinforcement learning with large state spaces. Our algorithm, \algname, critically relies on the robust Catoni estimator, and our analysis establishes novel results on uniform Catoni estimation in general martingale regression settings, which may be of independent interest.

Several questions remain open for future work. First, while we show that it is possible to obtain a computationally efficient version of \algname, doing so incurs an additional $\sqrt{d}$ factor. Removing this factor while maintaining computational efficiency would be an interesting direction and may require new techniques. More broadly, obtaining a computationally efficient algorithm with regret scaling as $\sqrt{d^2}$ would be an interesting future direction. \cite{zanette2020learning} show that it is possible to obtain a $\sqrt{d^2}$ scaling, but their algorithm is computationally inefficient. 
In addition, obtaining optimal $H$ dependence is of much interest.
While \algname will achieve this for $\Vstval \le 1$, technical challenges remain to showing this holds in general. We believe our use of the Catoni estimator could be a key step towards achieving this, but leave this for future work. 
Finally, developing first-order regret bounds for more general function approximation settings \cite{jiang2017contextual,du2021bilinear} is an exciting direction. The results in this work rely strongly on the linearity of the MDP, yet, as a first step, it may be possible to extend our techniques to bilinear classes \cite{du2021bilinear}, which also exhibit a certain linear structure.

\iftoggle{arxiv}{
\subsection*{Acknowledgements}}{
\vspace{-0.5em}
\paragraph{Acknowledgements.}}
The work of AW is supported by an NSF GFRP Fellowship DGE-1762114. The work of SSD is in part supported by grants NSF IIS-2110170. The work of KJ was funded in part by the AFRL and NSF TRIPODS 2023166.

%

\newpage
\bibliographystyle{icml2022}
\bibliography{bibliography.bib}

\newpage
\appendix

\newcommand{\catround}{\catonii_{\mathsf{rnd}}}
\newcommand{\epsround}{\epsilon_{\mathsf{rnd}}}
\newcommand{\alpharound}{\alpha_{\mathsf{rnd}}}
\newcommand{\cAround}{\cA_{\mathsf{rnd}}}
\newcommand{\Lf}{L_f}
\newcommand{\Lg}{L_g}
\newcommand{\zetabar}{\bar{\zeta}}
\newcommand{\sigmamin}{\sigma_{\min}}

\section{Technical Results}\label{sec:technical}

\subsection{Covering and Elliptical Potential Lemmas}
\begin{defn}[Covering Number]\label{defn:cov_num} Let $\calX$  be a set with metric $\dist(\cdot,\cdot)$. Given $\epsilon > 0$, the $\epsilon$-covering number of $\calX$ in $\dist$, $\covnum(\calX,\dist,\epsilon)$, is defined as the minimal cardinality of a set $\mathcal{N} \subset \calX$ such that, for all $x \in \calX$, there exists an $x' \in \calN$ with $\dist(x,x') \le \epsilon$. 
\end{defn}
\begin{lem}[\cite{vershynin2010introduction}]\label{lem:euc_ball_cover}
For any $\epsilon > 0$, the $\epsilon$-covering number of the Euclidean ball $\mathcal{B}^d(R) := \{\bx \in \R^d: \|\bx\|_2 =1\}$ with radius $R > 0$ in the Euclidean metric is upper bounded by $(1 + 2R/\epsilon)^d$. 
\end{lem}

\begin{lem}[Lemma D.6 of \cite{jin2020provably}]\label{lem:q_fun_cover}
Consider the class of functions from $\bbR^d$ to $\bbR$ of the form
\begin{align*}
f(\bphi) = \min \left \{ \innerb{\bw}{\bphi} + \beta \| \bphi \|_{\bLambda^{-1}}, H \right \}
\end{align*}
where the parameters $\bw,\beta,\bLambda$ satisfy $\| \bw \|_2 \le \betaw$, $\beta \in [0,B]$, and $\bLambda \succeq \lambda I$. Let $\cN_{\epsilon}$ be an $\epsilon$-covering of this set with respect to the norm $\dist_{\infty}(f,f') := \sup_{\bphi \in \Ball^{d}} |f(\bphi) - f'(\bphi)|$. Then,
\begin{align*}
\log | \cN_\epsilon | \le d \log(1 + 4\betaw/\epsilon) + d^2 \log ( 1 + 8 \sqrt{d} B^{2} / (\lambda \epsilon^2)) .
\end{align*} 
\end{lem}

\begin{lem}[Elliptic Potential Lemma, Lemma 11 of \cite{abbasi2011improved}]\label{lem:elip_pot}
Under the same assumptions as \Cref{lem:elip_pot_bad_event_informal}, for any choice of $\lambda > 0$, we will have that
\begin{align*}
\sum_{t=1}^T \min \{ 1, \| \bx_t \|_{\bV_{t-1}^{-1}}^2 \} \le 2 d \log ( 1 + a^2 T / (d \lambda)) .
\end{align*}
Furthermore, if $\lambda \ge \max \{ 1, a^2 \}$,
\begin{align*}
\sum_{t=1}^T \| \bx_t \|_{\bV_{t-1}^{-1}}^2 \le 2 d \log ( 1 + a^2 T / (d \lambda)) .
\end{align*}
\end{lem}

\begin{lem}[Freedman's Inequality \citep{freedman1975tail}]\label{lem:freedman}
$\cF_0 \subset \cF_1 \subset \ldots \subset \cF_T$ be a filtration and let $X_1,X_2,\ldots,X_T$ be real random variables such that $X_t$ is $\cF_t$-measurable, $\Exp[X_t | \cF_{t-1}] = 0$, $| X_t | \le b$ almost surely, and $\sum_{t=1}^T \Exp[X_t^2 | \cF_{t-1}] \le V$ for some fixed $V > 0$ and $b > 0$. Then for any $\delta \in (0,1)$, we have with probability at least $1-\delta$,
\begin{align*}
\sum_{t=1}^T X_t \le 2\sqrt{V \log(1/\delta)} + b \log(1/\delta) .
\end{align*}
\end{lem}

\begin{proof}[Proof of \Cref{lem:elip_pot_bad_event_informal}]
Our goal is to bound the number of times that $\| \bx_t \|_{\bV_{t-1}^{-1}} > b$. A now-standard determinant computation (see, e.g. \citet{abbasi2011improved}) based on the Sherman-Morrison identity yields
\begin{align*}
\det(\bV_t) = \det(\bV_{t-1}) ( 1 + \| \bx_t \|_{\bV_{t-1}^{-1}}^2) \implies \| \bx_t \|_{\bV_{t-1}^{-1}}^2 = \frac{\det(\bV_t)}{\det(\bV_{t-1})} - 1 .
\end{align*}
It follows that, whenever $\| \bx_t \|_{V_{t-1}^{-1}} > b$, it must also be the case that
\begin{align*}
\frac{\det(\bV_t)}{\det(\bV_{t-1})} - 1 > b \iff \det(\bV_t) > (1 + b) \det(\bV_{t-1}) .
\end{align*}
In particular, if $N$ denotes the number of times that $\| \bx_t \|_{\bV_{t-1}^{-1}} > b$ for $t \in \{ 1,\ldots , T \}$, then it follows that $\det(\bV_T) > (1+b)^N \det(\bV_0) = (1+b)^N \lambda^d$. At the same time, 
\begin{align*}
\det(\bV_T) & = \det \left ( \lambda I + \sum_{s=1}^T \bx_s \bx_s^\top \right ) \\
& \le \left (  \| \lambda I + \sum_{s=1}^T \bx_s \bx_s^\top \|_{\op} \right ) ^d \\
& \le ( \lambda  + a^2 T)^d .
\end{align*}
Combining these inequalities gives:
\begin{align*}
(1+b)^N \lambda^d < (\lambda + a^2 T)^d \iff N < \frac{d \log ( \lambda + a^2 T) - d \log (\lambda)}{\log ( 1 + b)} . 
\end{align*}
\end{proof}

We remark that a variant of \Cref{lem:elip_pot_bad_event_informal} appeared in concurrent work \citep{kim2021improved}, and originally as an exercise in \cite{lattimore2020bandit}. 

\subsection{Martingale Catoni Estimation}\label{app:cat_est}


\begin{lem}[Martingale Catoni Estimator]\label{lem:martingale_catoni}
Let $\cF_0 \subset \cF_1 \subset \ldots \subset \cF_T$ be a filtration and let $X_1,X_2,\ldots,X_T$ be square-integrable real random variables such that $X_t$ is $\cF_t$-measurable, and
\begin{itemize}
\item Conditional means $\Exp[X_t | \cF_{t-1}] = \zeta_t$ for some (possibly random) $\zeta_t$. 
\item $|\zeta_t| \le \zetabar$ with probability 1 for some fixed (non-random) $\zetabar$.
\item Average conditional mean $\zeta := \frac{1}{T} \sum_{t=1}^T \zeta_t$. 
\item Conditional variances $\sum_{t=1}^T \Exp[ (X_t - \zeta_t)^2 | \cF_{t-1}] \le V$ for some (possibly random) $V > 0$. 
\end{itemize}
 Then for any confidence $\delta \in (0,1)$, fixed $\alpha > 0$, and sample size $T \ge \alpha^2 ( V + \sum_{t=1}^T ( \zeta_t - \zeta)^2) + 2 \log \frac{1}{\delta}$, we have with probability at least $1-2\delta$, the Catoni estimator $\catonii_{T,\alpha}$ satisfies
\begin{align*}
\left| \catonii_{T,\alpha} - \zeta \right| \le \frac{\alpha \left( V + \sum_{t=1}^T (\zeta_t - \zeta)^2\right)}{T} + \frac{2 \log \frac{8 T( 1 +   \zetabar \alpha )^2}{\delta}}{\alpha T} .
\end{align*}
\end{lem}

\begin{proof}
The primary difference between this result and that of Lemma 13 of \cite{wei2020taking} is that we allow for random $\zeta_t$ and $V$, while \cite{wei2020taking} assume they are fixed.
The first portion of this proof follows closely the proof of Lemma 13 of \cite{wei2020taking}. Given the similarity, we omit several algebraic calculations that appear in \cite{wei2020taking}.

\paragraph{Analysis for Fixed $z$.}
Throughout the proof we let $\Exp_t[\cdot] := \Exp[\cdot \mid \cF_{t-1}]$. Note that $\psicat(y) \le \log(1+y+y^2/2)$, so with some calculation we can upper bound, for any fixed $z$, 
\begin{align*}
\Exp_t[\exp (\psicat(\alpha(X_t - z)))] \le \exp \left ( \alpha ( \zeta_t - z) + \frac{\alpha^2 \Exp_t[(X_t - \zeta_t)^2] + \alpha^2 (\zeta_t - z)^2}{2} \right ).
\end{align*}
Define recursively the random variable $Z_0 = 1$, and
\begin{align*}
Z_t = Z_{t-1} \exp (\psicat(\alpha(X_t - z)))\exp \left ( - \left ( \alpha ( \zeta_t - z) + \frac{\alpha^2 \Exp_t[(X_t - \zeta_t)^2] + \alpha^2 (\zeta_t - z)^2}{2} \right ) \right ).
\end{align*}
The previous calculation shows that $\Exp_t[Z_t] \le Z_{t-1}$, which further implies $\Exp[Z_T] \le \Exp_[Z_{T-1}] \le \ldots \le \Exp[Z_1] = 1$. Define
\begin{align*}
g(z) := T \alpha (\zeta - z) + \frac{1}{2} \alpha^2 \sum_{t=1}^T (\zeta_t - z)^2 + \frac{1}{2} \alpha^2 V + \log \frac{1}{\delta} 
\end{align*}
and note that $\fcatoni(z) \ge g(z)$ implies
\begin{align*}
\sum_{t=1}^T \psicat (\alpha(X_t - z)) \ge \sum_{t=1}^T \left ( \alpha (\zeta_t - z) + \frac{\alpha^2 (\zeta_t - z)^2 + \alpha^2 \Exp_t[(X_t - \zeta_t)^2]}{2} \right ) + \log \frac{1}{\delta}.
\end{align*}
This further implies $Z_T \ge 1/\delta$. By Markov's Inequality, we have for fixed $z$ that 
\begin{align*}
\Pr[\fcatoni(z) \ge g(z)] \le \Pr[Z_T \ge 1/\delta] \le \Pr[Z_T \ge \Exp[Z_T]/\delta] \le \delta.
\end{align*}

\paragraph{Covering the Space of $z$.}
Let
\begin{align*}
z_0 = \zeta + \frac{1}{\alpha} \left ( 1 - \sqrt{1 - \frac{\alpha^2 ( V + \sum_{t=1}^T (\zeta_t - \zeta)^2)}{T} - \frac{2}{T} \log \frac{1}{\delta}} \right ) 
\end{align*}
and note that $z_0$ is a root of $g(z)$. Now, if $z_0$ were non-random as in \cite{wei2020taking}, since $\fcatoni$ is monotonic and $\fcatoni(\catonii_{T,\alpha}) = 0$, we would have
\begin{align*}
\Pr[\catonii_{T,\alpha} \ge z_0] = \Pr[\fcatoni(z_0) \ge 0] = \Pr[\fcatoni(z_0) \ge g(z_0)] \le \delta
\end{align*}
which would complete the proof. However, in our setting $z_0$ is random, and it is not clear that the final inequality holds. 

Note that $|z_0| \le \zetabar + 1/\alpha$ with probability 1. Let $\cZ_{\epsilon} = \{ -\zetabar - 1/\alpha, - \zetabar - 1/\alpha + \epsilon, \ldots, \zetabar + 1/\alpha - \epsilon, \zetabar + 1/\alpha \}$ for some $\epsilon$ to be chosen and note that $|\cZ_\epsilon| \le \frac{2(\zetabar + 1/\alpha)}{\epsilon}$. Assume that $|f(z) - f(z')| \le \Lf | z - z' | $ and $|g(z) - g(z')| \le \Lg | z - z'|$ for $z,z' \in [-\zetabar - 1/\alpha, \zetabar + 1/\alpha]$ and some deterministic $\Lf, \Lg$. Note that, similar to the above calculation, for deterministic $x$,
\begin{align*}
\Pr[\fcatoni(z) \ge g(z) - x] \le \Pr[Z_T \ge \exp(-x)/\delta ] \le \Pr[Z_T \ge \Exp[Z_T] \exp(-x)/\delta ] \le \delta \exp(x).
\end{align*}
Fixing $\epsilon = 1/(\Lf + \Lg)$ and union bounding over all $z \in \cZ_\epsilon$, it follows that with probability at least $1 - |\cZ_\epsilon| e \delta$, for all $z \in \cZ_\epsilon$ simultaneously, 
\begin{align*}
\fcatoni(z) \le g(z) - 1.
\end{align*}
By the construction of $\cZ_\epsilon$, for any $z_0$ satisfying $|z_0| \le \zetabar + 1/\alpha$, there exists some $z \in |\cZ_\epsilon|$ such that $|z_0 - z| \le \epsilon = 1/(\Lf + \Lg)$. Thus, it follows that if $\fcatoni(z) \le g(z) - 1$ for all $z \in \cZ_\epsilon$, then 
\begin{align*}
\fcatoni(z_0) - \Lf \epsilon \le g(z_0) + \Lg \epsilon - 1 \implies \fcatoni(z_0) \le g(z_0).
\end{align*}
So it follows that on the event that $\fcatoni(z) \le g(z) - 1$ for all $z \in \cZ_\epsilon$, we have $\fcatoni(z_0) \le g(z_0)$, which further implies $\catonii_{T,\alpha} \ge z_0$. As this event occurs with probability at least $1-|\cZ_\epsilon| e \delta$, it follows that $\catonii_{T,\alpha} \ge z_0$ holds with probability at least $1-|\cZ_\epsilon| e \delta$.

\paragraph{Bounding the Lipschitz Constant of $\fcatoni(z)$ and $g(z)$.}
We have
\begin{align*}
\frac{\rmd}{\rmd z} g(z) = - T\alpha - \alpha^2 \sum_{t=1}^T (\zeta_t - z).
\end{align*}
For $|z| \le \zetabar + 1/\alpha$ and $|\zeta_t| \le \zetabar$, we can then bound
\begin{align*}
| \frac{\rmd}{\rmd z} g(z) | \le T \alpha + \alpha^2 T (2\zetabar + 1/\alpha) =: \Lg. 
\end{align*}
We also have
\begin{align*}
\frac{\rmd}{\rmd z} \fcatoni(z) = -\alpha \sum_{t=1}^T \psicat'(\alpha(X_t - z))
\end{align*}
for
\begin{align*}
\psicat'(y) = \begin{cases} \frac{1+y}{1+y+y^2/2} & y \ge 0 \\
\frac{1-y}{1-y+y^2/2} & y < 0 \end{cases}.
\end{align*}
Note that $|\psicat'(y)| \le 1$ for all $y$. Thus, we can bound
\begin{align*}
| \frac{\rmd}{\rmd z} \fcatoni(z) | \le \alpha T =: \Lf.
\end{align*}
The final result covers by plugging in these values of $\Lf$ and $\Lg$, rescaling $\delta$, and repeating the same calculation in the opposite direction.

\end{proof}

\subsubsection{Heteroscedastic Catoni Estimation}

We recall the heteroscedastic heavy-tailed martingale linear regression setting as defined in \Cref{defn:heavy_mart_reg}.

\defnmartlinreg*

\heterodirectcatoni*
\begin{proof}
We apply \Cref{lem:martingale_catoni} to the scalar data $X_t := \bv^\top  \bphi_t y_t/\sigma_t^2 $. Note that with this choice of $X_t$, 
\begin{align*}
\Exp[X_t| \cF_{t-1}] = \frac{1}{\sigma^2_t} \bv^\top  \bphi_t \bphi_t^\top \btheta_{\star} =: \zeta_t[\bv]
\end{align*} 
so we will have that
\begin{align*}
\frac{1}{T} \sum_{t=1}^T \Exp[X_t | \cF_{t-1}] & = \frac{1}{T} \sum_{t=1}^T \frac{1}{\sigma_t^2} \bv^\top \bphi_t \bphi_t^\top \bthetast= \zeta[\bv].
\end{align*}
Furthermore, we can bound $|\zeta_t[\bv]| \le \| \bv \| \| \bthetast \| / \sigmamin^2$.

Applying \Cref{lem:martingale_catoni} gives that, with probability at least $1-\delta$, 
\begin{align*}
| \catonii[\bv] - \zeta[\bv] | \le \frac{\alpha \left ( V + \sum_{t=1}^T (\zeta_t[\bv] - \zeta[\bv])^2 \right )}{T} + \frac{2 \log \frac{8 T( 1 +   \| \bv \| \| \bthetast \| \alpha / \sigmamin^2 )^2}{\delta}}{\alpha T},
\end{align*}
where $V > 0$  is any upper bound on the quantity
\begin{align*}
V \ge \sum_{t=1}^ \Exp[ (X_t - \zeta_t[\bv])^ 2 \mid \cF_{t-1}] &=  \sum_{t=1}^T \Exp \left [ \sigma_{t}^{-4}\left ( \bv^\top  \bphi_t y_t  - \bv^\top  \bphi_t \bphi_t^\top \btheta_{\star}  \right )^2 | \cF_{t-1} \right ]  \\
&= \sum_{t=1}^T \left ( \bv^\top \bphi_t/\sigma_t \right )^2 \Exp \left [  (  y_t  -  \bphi_t^\top \btheta_{\star}   )^2 / \sigma_t^2 | \cF_{t-1} \right ]
\end{align*}
Our assumption on $\sigma_t$ ensures that
\begin{align}
1 \ge \sigma_t^{-2}\Exp[y_t^2 \mid \cF_{t-1}] &=  \sigma_t^{-2}\Var[y_t] + \sigma_t^{-2}\Exp[y_t \mid \cF_{t-1}]^2 \nonumber \\
&= \Exp\left [  (  y_t  -  \bphi_t^\top \btheta_{\star}   )^2 / \sigma_t^2 \mid \cF_{t-1} \right ]  + \sigma_t^{-2} (\bphi_t^\top \btheta_{\star})^2, \label{eq:sig_bound_inter}
\end{align}
 so it suffices that we select $V$ to be
\begin{align*}
V &= \sum_{t=1}^T \left ( \bv^\top  \bphi_t/\sigma_t \right )^2  = \|\bv\|_{\bSigma_T}^2. 
\end{align*}
Furthermore, since $\zeta[\bv]$ is the average of the terms $\zeta_t[\bv]$, we can upper bound
\begin{align*}
\sum_{t=1}^T (\zeta_t[\bv] - \zeta[\bv])^2 & \le \sum_{t=1}^T \zeta_t[\bv]^2 = \sum_{t=1}^T(\bv^\top  \bphi_t \bphi_t^\top \btheta_{\star} / \sigma_t^2 )^2 \\
&=  \sum_{t=1}^T (\bv^\top \bphi_t / \sigma_t)^2 (\bphi_t^\top \bthetast / \sigma_t )^2 
\end{align*}
Again, our assumption on $\sigma_t$ ensures $(\bphi_t^\top \bthetast / \sigma_t )^2 \le 1$ via \Cref{eq:sig_bound_inter}, so we can bound
\begin{align*}
\sum_{t=1}^T (\zeta_t[\bv] - \zeta[\bv])^2\le \sum_{t=1}^T (\bv^\top \bphi_t / \sigma_t)^2 = \|\bv\|_{\bSigma_T}^2.
\end{align*}
Putting these two bounds together, we have that
\begin{align}
| \catonii[\bv] - \zeta[\bv] | &\le \frac{\alpha \left( V + \sum_{t=1}^T (\zeta_t[\bv] - \zeta[\bv])^2\right)}{T} + \frac{2 \log \frac{16 T( 1 +   \| \bv \| \| \bthetast \| \alpha / \sigmamin^2 )^2}{\delta}}{\alpha T} . \nonumber\\
&\le \frac{2 \alpha \|\bv\|^2_{\bSigma_T}}{T} + \frac{2 \log\frac{16 T( 1 +   \| \bv \| \| \bthetast \| \alpha / \sigmamin^2 )^2}{\delta}}{\alpha T}. \label{eq:alpha_bound_intermediate}
\end{align}
provided that 
\begin{align*}
T &\ge  2\alpha^2 \|\bv\|^2_{\bSigma_T} + 2 \log \frac{2}{\delta} \ge \alpha^2 ( V + \sum_{t=1}^T ( \zeta_t[\bv] - \zeta[\bv])^2) + 2 \log \frac{2}{\delta}.
\end{align*}
Introduce $\alpha_0 := \sqrt{ \frac{ \log \frac{16 T( 1 +   \| \bv \| \| \bthetast \| \alpha / \sigmamin^2 )^2}{\delta}}{ \|\bv\|^2_{\bSigma_T}}}$, so that $\alpha = \min\{\gamma \alpha_0,\alpha_{\max}\}$ (recall $\gamma \ge 1$ is a possibly random scalar but that $\alpha$ is deterministic). Then, $\gamma \alpha_0 \ge \alpha$. Hence, it is enough that
\begin{align*}
T &\ge  2\gamma^2 \alpha_0^2  \|\bv\|^2_{\bSigma_T} + 2 \log \frac{2}{\delta} = (2+2\gamma^2) \log \frac{2}{\delta}. 
\end{align*}
Moreover, using $\alpha = \min\{\gamma\alpha_0,\alpha_{\max}\}$ and $\gamma \ge 1$, we can continue the bound in \Cref{eq:alpha_bound_intermediate} via
\begin{align*}
| \catonii[\bv] - \zeta[\bv] | 
&\le \frac{2 \min\{\alpha_{\max},\gamma\alpha_0\} \|\bv\|_{\bSigma_T}^2}{T} + \frac{2 \log \frac{16 T( 1 +   \| \bv \| \| \bthetast \| \alpha / \sigmamin^2 )^2}{\delta}}{\min\{\alpha_{\max},\gamma\alpha_0\} T}\\
&\le  \frac{2 \gamma\alpha_0 \|\bv\|_{\bSigma_T}^2}{T} + \frac{2 \log \frac{16 T( 1 +   \| \bv \| \| \bthetast \| \alpha_{\max} / \sigmamin^2 )^2}{\delta}}{\gamma \alpha_0 T} + \frac{2 \log \frac{16 T( 1 +   \| \bv \| \| \bthetast \| \alpha_{\max} / \sigmamin^2 )^2}{\delta}}{\alpha_{\max} T}\\
&\le  (2+2\gamma) \|\bv\|_{\bSigma_T} \sqrt{ \frac{ \log \frac{16 T( 1 +   \| \bv \| \| \bthetast \| \alpha_{\max} / \sigmamin^2 )^2}{\delta}}{T^2}} + \frac{2 \log \frac{16 T( 1 +   \| \bv \| \| \bthetast \| \alpha_{\max} / \sigmamin^2 )^2}{\delta}}{\alpha_{\max} T}.
\end{align*}
\end{proof}
\subsection{Self-Normalized Catoni Estimation \label{app:self_norm_catonii}}

\regfunapprox*

We now state an intermediate technical proposition, from which derive our main self-normalized guarantee as a special case:
\begin{prop}\label{lem:catoni_union} Let $c > 0$ denote a universal constant, take parameters $\lambda > 0$ and $\alpha_{\max} \ge 1$, and consider the regression with function approximation of \Cref{defn:regression_with_function_approx} with parameters $d,p,\sigma_{\min}^2,\beta_u,\beta_{\mu},R,H$. For a sample size $T \in \N$ introduce the effective dimension
\begin{align*}
d_T := c\cdot (p+d)\cdot\logterm\left(T, \alpha_{\max}^2,  \lambda^{-1},\sigma_{\min}^{-2},\beta_{\mu},\beta_u,\beta_\eta,R,H\right).
\end{align*} 
For vectors $\bvtil \in \R^d$, define the  mean parameter
\begin{align*}
\zeta[\bvtil] := \frac{1}{T}\bvtil^\top \bSigma_T \cdot\btheta_{\star}  
\end{align*}
and let $\catonii[\bvtil]$ denote the Catoni estimator  using features and $\alpha[\bvtil]$ parameter 
\begin{align*}
X_t = \bvtil^\top \bphi_t y_t / \sigma_t^2, \quad \alpha[\bvtil] = \min \left \{ \sqrt{  \tfrac{d_T + \log 1/\delta}{\|\bvtil\|_{\bSigma_T}^2 }},\, \alphamax \right \}.
\end{align*} 
Then, if $T \ge 6(\log\frac{1}{\delta} + d_T)$, with probability $1 - \delta$, it holds that $\forall \bv \in \Ball^d$ and for all $\bLambda \succeq \lambda I$,
\begin{align*}
 \left| \catonii\left[\bLambda^{-1} \bv\right] - \zeta\left[\bLambda^{-1} \bv\right]\right| \le   4\|\bLambda^{-1}\bv\|_{\bSigma_T} \sqrt{ \frac{\log \frac{1}{\delta} + d_T  }{T^2}} + \frac{3 (\log\frac{1}{\delta}+d_T)}{\alphamax T}.
\end{align*}
\end{prop}

\begin{proof}[Proof of \Cref{cor:catoni}] We instantiate \Cref{lem:catoni_union} with $\bLambda = \frac{1}{T}(\lambda I + \bSigma_T) \succeq \frac{1}{T}\bSigma_T$. For this choice of $\bLambda$, it holds that
\begin{align*}
\|\bLambda^{-1}\bv\|_{\bSigma_T}^2  = \bv^\top \bLambda^{-1}\bSigma_T \bLambda^{-1}\bv \le T\|\bv\|_{\bLambda^{-1}}^2.
\end{align*}
Moreover, $\bLambda \succeq \frac{1}{T}\cdot\lambda I$, so taking $\lambda \gets \lambda/T$, $d_T$ still has the same form for a possibly larger constant $c> 0$. Next,
\begin{align*}
\zeta[\bLambda^{-1}\bv] &= \frac{1}{T}\bv^\top  \bLambda^{-1}\bSigma_T \cdot\btheta_{\star}\\
&= \bv^\top  (\bSigma_T + \lambda I)^{-1}\bSigma_T \cdot\btheta_{\star}\\
&= \bv^\top \btheta_{\star} - \lambda  \bv^\top (\bSigma_T + \lambda I)^{-1} \cdot\btheta_{\star}. 
\end{align*}
It follows that
\begin{align*}
\left| \catonii\left[\bLambda_T^{-1} \bv\right] - \bv^\top \btheta_{\star} \right | & \le \left| \catonii\left[\bLambda_T^{-1} \bv\right] - \zeta[\bLambda^{-1} \bv] \right | + | \lambda  \bv^\top (\bSigma_T + \lambda I)^{-1} \cdot\btheta_{\star} | .
\end{align*}
We bound $\left| \catonii\left[\bLambda_T^{-1} \bv\right] - \zeta[\bLambda^{-1} \bv] \right |$ by \Cref{lem:catoni_union} and bound
\begin{align*}
| \lambda  \bv^\top (\bSigma_T + \lambda I)^{-1} \cdot\btheta_{\star} |  & \le \lambda \| \bv^\top (\bSigma_T + \lambda I)^{-1/2} \|_2 \| (\bSigma_T + \lambda I)^{-1/2} \|_\op  \| \bthetast \|_2  \\
& \le \sqrt{\lambda} \| \bv \|_{\bLambda^{-1}} \| \bthetast \|_2 .
\end{align*}

\end{proof}

\begin{proof}[Proof of \Cref{lem:catoni_union}] 
The proof requires a careful covering of directions $\bv^\top \bLambda^{-1}$, and regression functions $f \in \Fclass$.

\paragraph{Notation.} Let us establish some notation to facilliate the covering.  Given $f \in \Fclass$, we define the associated targets
\begin{align*}
\ytil_t(f) := \inner{\bphi_t}{\bust} + f(\bphi'_t) + \eta_t, \quad  \bthettil(f) := \bust +  \int   f(\bphi')  \rmd \bmu(\bphi').
\end{align*}
Given $\bvtil \in \R^d$ and $f \in \Fclass$, define
\begin{align*}
\quad \zeta[\bvtil,f] := \bvtil^\top \bSigma_T \btheta_{\star},
\end{align*}
and let $\catonii[\bvtil,f,\tilde\alpha]$ to denote the Catoni estimator using parameter $\tilde\alpha$ and features
\begin{align} 
 X_t[\bvtil,f] := \frac{1}{\sigma^2_t}\bvtil^\top \bphi_t \ytil_t(f)\label{eq:X_t_def}
\end{align}
Over loading notation, define $\catonii[\bvtil,f]$ to denote the following  estimate using the correct, data-dependent : 
\begin{align}
\catonii[\bvtil,f] = \catonii[\bvtil,f,\alpha[\bvtil]], \quad  \alpha[\bvtil] =\min \left \{\frac{ \sqrt{\log \frac{2M c_T}{\delta}} }{\|\bvtil\|_{\bSigma_T}}, \alphamax \right \}, \label{eq:cat_alpha_2}
\end{align}
where $M$ is some value we will set later, and $c_T$ is as in \Cref{lem:catoni2}, but with $\| \bv \|$ replaced by $1/\lambda$, which is a bound on the norm of $\bvtil$. Note that the correspondence between the original notation parameterized by direction $\bLambda^{-1}\bv$ and the new notation is given by 
\begin{align}
\catonii[\bvtil,f_{\star}] = \catonii[\bLambda^{-1}\bv], \quad \zeta[\bvtil,f_{\star}] = \zeta^\star[\bLambda^{-1}\bv], \quad \bvtil =\bLambda^{-1}\bv.  \label{eq:notation_equiv}
\end{align}

We note that by the assumption that $\|\bv\|_2 \le 1$ and $\bLambda \succeq \lambda I$, it suffices to consider $\bvtil$ in the set
\begin{align*}
\cV := \{\bvtil: \|\bvtil\|_2 \le \beta_{\tilde{v}}, \quad \beta_{\tilde{v}} := 1/\lambda\}.
\end{align*}
Lastly, we define the interval
\begin{align*}
\cA := \{\alpha: \frac{\lambda}{T} \le \alpha \le \alpha_{\max}\}
\end{align*}

\paragraph{Rounding $\alpha$.}
\newcommand{\roundop}{\mathsf{round}}
To handle that $\alpha[\bvtil]$ is data-dependent, we will build a cover using the Catoni estimator with rounded values of $\alpha[\bvtil]$. Note that this step is purely for the analysis, and does not need to be incorporated into the algorithm. For $\epsilon > 0$  and scalar $k$, set 
\begin{align*}
\roundop(x,\epsilon) := \inf\{(1+\epsilon)^k: (1+\epsilon)^k \ge x, \quad k \in \N\}.
\end{align*}
Fixing an $\epsround \in (0,1/4)$ to be chosen, set 
\begin{equation}
\label{eq:catoniiround}
\begin{aligned}
\catround[\bvtil,f] &= \catonii[\bvtil,f,\alpharound[\bvtil]], \quad \\
&\text{where }\alpharound[\bvtil] = \min \left \{\roundop\left(\frac{ \sqrt{\log \frac{2M c_T}{\delta}} }{\|\bvtil\|_{\bSigma_T}},\epsround \right), \alphamax \right \}
\end{aligned}
\end{equation}
Note that since $\|\bphi_t\|_2 \le 1$ and $\|\bvtil\|_2 \le 1/\lambda$ for $\bvtil \in \cV$, we have that $\alpha[\bvtil] \in \cA$ for $\bvtil \in \cV$. Note then that the rounded Catoni parameters lie in the finite set
\begin{align}
\alpharound[\bvtil] \in \cAround, \quad \text{where } \cAround := \left\{(1+\epsround)^k : \frac{\lambda}{T} \le (1+\epsround)^k \le \alpha_{\max} \right\} \cup \{\alpha_{\max}\}.  \label{eq:Around}
\end{align}
Furthermore, the cardinality of $\cAround$ can be crudely bounded by
\begin{equation} 
\begin{aligned}
|\cAround| \le \log_{1+\epsround}(T\alpha_{\max}/\lambda) &= 1+ \frac{\log(T\alpha_{\max}/\lambda)}{\log(1+\epsround)} \\
&\le 1 + (T\alpha_{\max}/\lambda ) \cdot 2/\epsround  \\
&\le 1+\frac{2T\alpha_{\max}}{\lambda\epsround}. \label{eq:cAround_card}
\end{aligned}
\end{equation} 
where we used the crude bound $\log x \le1+ x$ for $x \ge 1$, and $\log (1+\epsilon) \ge \epsilon /2$ for $\epsilon \in (0,1/4)$.

\paragraph{Uniform bound on a cover. } Let $\calN_{1} \subset \cV \subset \R^{d}$ and  $\calN_{2} \subset \Fclass$ denote fixed (deterministic), finite sets whose product $\calN = \calN_1 \times \calN_2$ has cardinality at most $|\cAround| \cdot |\calN|\le M$. We use \Cref{lem:catoni2} to establish a uniform bound on the errors $| \catonii[\bvtil,f,\tilde\alpha] - \zeta[\bvtil,f,\tilde\alpha] |$ of the Catoni esimator corresponding to pairs $(\bvtil, f) \in \calN$ and $\tilde \alpha \in \cAround$. To do this, we have to be somewhat careful, because we require that conditional variances are upper bounded by $\sigma_t^2$. To this end, we argue a bound on the Catoni error when the following \emph{random event} holds:
\begin{align*}
\cE_f &:= \{\Exp \left [   \ytil_t(f)^2 \mid \cF_{t-1} \right ] \le \sigma_t^2, ~\forall t \},
\end{align*}
Note that $\cE_f$ is indeed random because $\sigma_t^2$ are random. Using linearity of expectation, we have
\begin{align*}
\Exp[\ytil_t(f) | \cF_{t-1}] & = \inner{\bphi_t}{\bust} +  \inner{\bphi_t}{\int   f(\bphi') \rmd \bmu(\bphi')}\\
&= \inner{\bphi_t}{\bust +  \int   f(\bphi')  \rmd \bmu(\bphi')}\\
&= \inner{\bphi_t}{\bthettil(f)}.
\end{align*}
so the linearity of expectation assumption required by \Cref{lem:catoni2} will be met. Hence, for all pairs $(\bvtil,f) \in \calN$  such that $\cE_{f}$ holds, and all $\tilde{\alpha} \in \cAround$ such that 
\begin{itemize}
    \item $\tilde{\alpha} \ge \alpha[\bvtil]$
    \item $\tilde{\alpha}$ can be expressed as  $ \min \left \{\tilde{\gamma} \cdot \frac{ \sqrt{\log \frac{2M c_T}{\delta}} }{\|\bvtil\|_{\bSigma_T}}, \alphamax \right \}$
    \item $T \ge (2+2\tilde\gamma) \log\frac{2M c_T}{\delta}$. 
\end{itemize}
then it holds that with probability $1 - \frac{\delta}{M}$,
\begin{align*}
| \catonii[\bvtil,f,\tilde{\alpha}] - \zeta[\bvtil,f] | \le  (2+2\tilde \gamma)\|\bvtil\|_{\bSigma_T} \sqrt{ \frac{\log \frac{2M c_T}{\delta} }{T^2}} + \frac{2 \log\frac{2M c_T}{\delta}}{\alphamax T}, \quad \text{ whenever } \cE_{f} \text{ holds.}
\end{align*}
In particular, selecting $\tilde{\alpha} = \alpharound[\bvtil]$, we can choose $\tilde \gamma = 1+\epsround$. As $\epsround$ was chosen such that $\epsround \le 1/4$, $2 + 2\tilde{\gamma} \le 5$.  Hence, we find that with probability at least $1 - \delta$,
\begin{align*}
| \catround[\bvtil,f] - \zeta[\bvtil,f] | \le  5\|\bvtil\|_{\bSigma_T} \sqrt{ \frac{\log \frac{2M c_T}{\delta} }{T^2}} + \frac{2 \log\frac{2M c_T}{\delta}}{\alphamax T}, \quad \forall (f, \bvtil) \in \calN \text{ such that } \cE_{f} \text{ holds} ,
\end{align*}
provided that  $T \ge (2+2\tilde\gamma^2) \log\frac{2M c_T}{\delta}$. Given our setting of $\tilde{\gamma}$, it suffices to take $T \ge 6\log\frac{2M c_T}{\delta}$.

\paragraph{Approximation by covering.} Having achieved a pointwise bound, we observe that, on the $1-\delta$ event above, for any $(\bvtil,f) \in \cV \times \Fclass$, and any $(\bvtil_0,f_0) \in \calN$ for which $\cE_{f_0}$ holds,
\begin{align}
&| \catonii[\bvtil,f] - \zeta[\bvtil,f] | \nonumber \\
&\quad\le |\catonii[\bvtil,f] - \catround[\bvtil_0,f_0]| + |\zeta[\bvtil,f] - \zeta[\bvtil_0,f_0]| + | \catround[\bvtil_0,f_0] - \zeta[\bvtil_0,f_0] | \nonumber\\
&\quad\le |\catonii[\bvtil,f] - \catround[\bvtil_0,f_0]| + |\zeta[\bvtil,f] - \zeta[\bvtil_0,f_0]| + 5\|\bvtil_0\|_{\bSigma_T} \sqrt{ \frac{\log \frac{2M c_T}{\delta} }{T^2}} + \frac{2 \log\frac{2M c_T}{\delta}}{\alphamax T} \nonumber\\
&\quad\le  5\|\bvtil\|_{\bSigma_T} \sqrt{ \frac{\log \frac{2M c_T}{\delta} }{T^2}} + \frac{2 \log\frac{2M c_T}{\delta}}{\alphamax T}  
+ \underbrace{5 \sqrt{ \frac{\log \frac{2M c_T}{\delta}  }{T^2}}\cdot\big{|} \|\bvtil\|_{\bSigma_T} - \|\bvtil_0\|_{\bSigma_T}\big{|}}_{(i)}\nonumber\\
&\qquad\qquad+\underbrace{|\zeta[\bvtil,f] - \zeta[\bvtil_{0},f_{0}]|}_{(ii)} + \underbrace{|\catonii[\bvtil,f] - \catround[\bvtil_{0},f_{0}]|}_{(iii)} \label{eq:approx_bound_inter}
\end{align}
Recall $\beta_u = \|\bust\|_2$, $\beta_{\mu} = \| |\bmu|(\Ball^{d'}) \|_2$, $\sigma_{\min}^2 \le \sigma_t^2$, and $\dist_{\infty}(f,f_0) := \sup_{\bphi' \in \Ball_{d'}}|f(\bphi') - f_0(\bphi')|$. We further assume that $\|\bvtil_0\|_2,\|\bvtil\|_2 \le \beta_{\tilde{v}}
$.  We show that, for these scalings, it suffices to ensure that $\|\bvtil - \bvtil_0\|_2$ and $\dist_{\infty}(f,f_0)$ are at most polynomial in relevant problem parameters:
\begin{lem}\label{lem:approximation_lemma} There exists a constant $\Cpoly = \poly(\sigma_{\min}^{-2},T,\beta_{\tilde v},\alpha_{\max},H,\beta_u,\beta_{\mu},\beta_{\eta})$ such that, if 
\begin{align*}
\max\{\|\bvtil - \bvtil_0\|_2, \dist_{\infty}(f,f_0),\epsround\} \le 1/\Cpoly,
\end{align*}
then
\begin{align*}
 \text{\normalfont Term $(i)$}  +  \text{\normalfont Term $(ii)$}  +  \text{\normalfont Term $(iii)$}  \le\frac{ \log\frac{2Mc_T}{\delta}}{\alphamax T}.
\end{align*}
\end{lem}We defer the proof of \Cref{lem:approximation_lemma} to the end of the section. 
In addition, we show that if $\dist_{\infty}(f,f_0)$ is sufficiently small, then the element $f_0$ from the covering satisfies the desired variance upper bound:
\begin{lem}\label{lem:variance_condition} Suppose that $f$ satisfies the variance bound in \Cref{eq:cat_union_sighat}, that is, 
\begin{align}
\Exp  [ \left(  \inner{\bphi_t}{\bust} + f(\bphi'_t) + \eta_t \right )^2 \mid \cF_{t-1} ] \le \frac{1}{2}\sigma_t^2. \label{eq:f_var_desired}
\end{align}
Then, if $\dist_{\infty}(f,f_0) \le 1/\Cpoly$ for an appropriate choice of $\Cpoly$ as in \Cref{lem:approximation_lemma},  $\cE_{f_0}$ holds, i.e., 
\begin{align*}\Exp [ \left(  \inner{\bphi_t}{\bust} + f_0(\bphi'_t) + \eta_t \right )^2 \mid \cF_{t-1} ] \le \sigma_t^2 .
\end{align*}
\end{lem}
\begin{proof} 
For any $f$ satisfying \Cref{eq:f_var_desired},
\begin{align*}
\Exp [ \left(  \inner{\bphi_t}{\bust} + f_0(\bphi'_t) + \eta_t \right )^2 \mid \cF_{t-1} ]  &\le 2\Exp [ \left(  \inner{\bphi_t}{\bust} + f(\bphi'_t) + \eta_t \right )^2 \mid \cF_{t-1} ] + 2\Exp [ \left(   f(\bphi'_t)  - f_0(\bphi'_t) \right )^2 \mid \cF_{t-1} ]\\
&\le \frac{1}{2}\sigma_t^2 + 2\dist_{\infty}(f,f_0)^2.
\end{align*}
Since $\sigma_t^2 \ge \sigma_{\min}^2$, it is enough that $\dist_{\infty}(f,f_0)^2 \le \frac{\sigma_{\min}^2}{2}$, which is ensured by an appropriate choice of $\Cpoly$. 
\end{proof}

\paragraph{Concluding the proof}
Let us summarize our current findings.
We see that if $\calN \subset \cV \times \Fclass$ is a collection of pairs $(\bvtil_0,f_0)$ satisfying
\begin{itemize}
    \item The cardinality bound $|\cAround| \cdot |\calN|\le M$
    \item The approximation bound that, 
    \begin{align*}
    \forall \bvtil \in \cV, f \in \Fclass,  \quad \exists(\bvtil_0,f_0) \in \calN \text{ such that } \max\{\|\bvtil - \bvtil_0\|_2,\, \dist_{\infty}(f,f_0)\} \le 1 /\Cpoly,
    \end{align*}
    and that $\epsround \le 1/\Cpoly$, where again 
    \begin{align*}
    \Cpoly = \poly(\sigma_{\min}^{-2},T,\beta_{\tilde{v}},\alpha_{\max},H,\beta_\mu,\beta_u,\beta_{\eta}) = \poly(\sigma_{\min}^{-2},T,1/\lambda,\alpha_{\max},H,\beta_u,\beta_{\mu},\beta_{\eta}).
    \end{align*}
\end{itemize}
Then, \Cref{eq:approx_bound_inter}, the fact that $\fst$ satisfies \Cref{eq:f_var_desired}, and \Cref{lem:approximation_lemma,lem:variance_condition} imply that with probability $1 - \delta$,
\begin{align*}
\quad \forall \bvtil \in \calV, \quad | \catonii(\bvtil,f_{\star}) - \zeta(\bvtil,f_{\star}) | \le   5\|\bvtil\|_{\bSigma_T} \sqrt{ \frac{\log \frac{2M c_T}{\delta}   }{T^2}} + \frac{3 \log\frac{2M c_T}{\delta}}{\alphamax T},
\end{align*}
provided that $T \ge 6 \log\frac{2M c_T}{\delta}$. We now find an $M$ sufficiently large to ensure the covering conditions hold. To this end, it suffices to ensure that $\calN = \calN_1 \times \calN_2$, where $\calN_1$ is an $\epsilon = 1/\Cpoly$ net of $\cV$, and $\calN_2$ is an $\epsilon = 1/\Cpoly$ net of $\Fclass$ in the norm $\dist_{\infty}(\cdot,\cdot)$. By  \Cref{lem:euc_ball_cover} and the fact that $\cV$ is a Euclidean ball of radius $\beta_{\tilde{v}} = 1/\lambda$, it suffices to take
\begin{align*}
\log|\calN_1| \le d \log (1+\frac{2\Cpoly}{\lambda }).
\end{align*}
Similarly, by assumption that the covering numbers of $\Fclass$ are $\covnum(\Fclass,\dist_{\infty}(\cdot,\cdot),\epsilon) \le p \log(1 + \frac{R}{\epsilon})$,
\begin{align*}
\log|\calN_2| \le p \log (1+2 R\Cpoly).
\end{align*}
Finally, using the bound on $|\cAround|$ from \Cref{eq:cAround_card}, 
\begin{align*}
\log |\cAround| \le \log (1+\frac{2T\alpha_{\max}}{\lambda\epsround}) \le \log (1+\frac{2\Cpoly T\alpha_{\max}}{\lambda}).
\end{align*}
Hence, we can bound, for universal constants $c',c>0$,
\begin{align*}
\log\frac{2Mc_T}{\delta} &= \log\frac{2|\cAround||\calN_1||\calN_2|c_T}{\delta}\\
&\le \log\frac{1}{\delta} + \log 2+ (p+d)\cdot c'\,\logterm\left( R, T,\alpha_{\max},\lambda^{-1},\Cpoly, \| \bthetast \|_2, \sigmamin^{-1} \right) \\
&= \log\frac{1}{\delta} + \underbrace{(p+d)\cdot c\cdot\logterm\left(\frac{1}{\delta},T, R, \lambda^{-1},\sigma_{\min}^{-2}, \alpha_{\max}^2,H,\beta_{\mu},\beta_u, \beta_\eta\right)}_{:= d_T}.
\end{align*}
For this choice of $M$,
\begin{align*}
\quad \forall \bvtil \in \calV, \quad | \catonii(\bvtil,f_{\star}) - \zeta(\bvtil,f_{\star}) | \le    5\|\bvtil\|_{\bSigma_T} \sqrt{ \frac{\log \frac{2}{\delta}  + d_T }{T^2}} + \frac{3 \log\frac{2M c_T}{\delta}}{\alphamax T},
\end{align*}
which, returning to the orginal notation parameterized by $\vLam$ and noting the equivalence of notation in \Cref{eq:notation_equiv}, we see that with probability $1 - \delta$, it holds that $\forall \bv \in \Ball^d$ and $\bLambda \succeq \lambda I$ (ensuring $\bLambda^{-1}\bv \in \cV$)
\begin{align*}
 | \catonii[\bLambda^{-1} \bv] - \zeta[\bLambda^{-1} \bv] | \lesssim    \| \bLambda^{-1} \bv \|_{\bSigma_T}\sqrt{ \frac{\log \frac{1}{\delta} + d_T }{T^2}} + \frac{ \log\frac{1}{\delta}+d_T}{\alphamax T}.
\end{align*}

\end{proof}

\subsubsection{Proofs supporting \Cref{lem:catoni_union}}

\begin{proof}[Proof of \Cref{lem:approximation_lemma}]

Recall $\beta_u = \|\bust\|_2$, $\beta_{\mu} = \| |\bmu|(\Ball^{d'})\|_2$, $\sigma_{\min}^2 \le \sigma_t^2$, $\beta_{\eta} \ge |\eta_t|$ with probability 1, and $\dist_{\infty}(f,f_0) := \sup_{\bphi' \in \Ball_{d'}}|f(\bphi') - f_0(\bphi')|$. We further assume that $\bvtil_0,\bvtil \in \calV$, i.e. 
\begin{align*}
\|\bvtil_0\|_2,\|\bvtil\|_2 \le \beta_{\tilde{v}},
\end{align*}
and that we may choose $\Cpoly = \poly(\sigma_{\min}^{-2},T,\beta_{\tilde v},\alpha_{\max},H,\beta_u,\beta_u,\beta_{\eta})$ to be an aritrary polynomial in these quantities. We move term by term, showing we can make each at most $\frac{\log (2M c_T/\delta)}{3\alpha_{\max} T}$ by selecting $\Cpoly$ appropriately. Throughout, we use the fact that, for $\delta \in (0,1/2)$ and $M \ge 1$, $\log (2M c_T/\delta) \ge 1$.

\begin{claim}[Bounding Term $(i)$]\label{claim:term_i} {\normalfont Term $(i)$} is at most $\sqrt{\frac{\log(2M c_T/\delta)}{T\sigma_{\min}^2}}\cdot\|\bvtil - \bvtil_{0}\|_2$. Hence, for an appropriate choice of $\Cpoly$, the above is at most {\normalfont Term $(i)$} is at most  $\frac{\sqrt{\log (2M c_T/\delta)}}{3\alpha_{\max} T} \le \frac{\log (2M c_T/\delta)}{3\alpha_{\max} T} $.
\end{claim}
\begin{proof} We have
\begin{align*}
\text{\normalfont Term $(i)$} &:= 5 \sqrt{ \frac{\log \frac{2M c_T}{\delta} \cdot }{T^2}}\cdot\left| \|\bvtil\|_{\bSigma_T} - \|\bvtil_0\|_{\bSigma_T}\right|\\
& \le 5 \sqrt{ \frac{\log \frac{2M c_T}{\delta} \cdot }{T^2}}\cdot \sqrt{\|\bSigma_T\|_{\op}}\cdot \|\bvtil - \bvtil_0\|_2.
\end{align*}
Since $\|\bphi_t\| \le 1$ and $\sigma_t \ge \sigma_{\min}$ by assumption, $\|\bSigma_T\|_{\op} = \|\sum_{t=1}^T \sigma_t^{-2}\bphi_t\bphi_t^\top\|_{\op} \le T\sigma_{\min}^{-2}$. The bound follows.
\end{proof}
\begin{claim}[Bounding Term $(ii)$] We can bound
\begin{align}
 \text{\normalfont Term (ii)} := |\zeta[\bvtil,f] - \zeta[\bvtil_{0},f_{0}]| \le \frac{1}{\sigma_{\min}^2}\left(\beta_{\mu}\beta_{\tilde{v}} \cdot \dist_{\infty}(f,f_{0}) + (\beta_u + H \beta_{\mu})\|\bvtil - \bvtil_{0}\|_2 \right). 
 \end{align}
 Hence, the appropriate choice of $\Cpoly$ ensures {\normalfont Term $(ii)$} is at most  $ \frac{\log (2M c_T/\delta)}{3\alpha_{\max} T} $.
\end{claim}
\begin{proof} We expand
\begin{align*}
\zeta[\bvtil,f] - \zeta[\bvtil_{0},f_{0}] &=   \frac{1}{T} \left(\bvtil^\top\bSigma_T\bthettil(f) -\bvtil_0^\top\bSigma_T\bthettil(f_0)\right)\\
&=   \frac{1}{T} \bvtil^\top\bSigma_T\left(\bthettil(f) - \bthettil(f_0)\right) + \frac{1}{T} \left(\bvtil - \bvtil_0\right)^\top\bSigma_T\bthettil(f_0).
\end{align*}
Hence, using the bound $\|\bSigma_T\|_{\op} \le T/\sigma_{\min}^2$ developed above, 
\begin{align*}
|\zeta[\bvtil,f] - \zeta[\bvtil_{0},f_{0}] | \le \frac{1}{\sigma_{\min}^2} (\|\bvtil\||\bthettil(f) - \bthettil(f_0)| + |\bthettil(f_0)|\|\bvtil - \bvtil_0\| ).
\end{align*}
Note that $\|\bvtil\| \le \beta_{\tilde{v}}$ by assumption. Further, we bound
\begin{align}\label{eq:theta_dif}
\| \bthettil(f) -\bthettil(f_{0})\|_{2} = \left\|  \int   \left(f -f_0\right)(\bphi')  \rmd \bmu(\bphi') \right\| \le \| |\bmu|(\Ball_{d'})\|_2 \cdot \|f - f_0\|_{\mathcal{L}_{\infty}(\Ball_{d'})} \le \beta_{\mu}\dist_{\infty}(f,f_0).
\end{align}
and moreover,
\begin{align}\label{eq:theta_unif}
\| \bthettil(f_{0})\|_{2} \le \|\bust\|_2 + \left\|  \int  f_0(\bphi')  \rmd \bmu(\bphi') \right\| \le \beta_u + \| |\bmu|(\Ball_{d'}) \|_2 \max_{\bphi' \in \Ball_{d'}} |f_0(\bphi')| \le \beta_u + H \beta_{\mu}. 
\end{align}
Combining the bounds concludes the proof.
\end{proof}

\begin{claim}[Bounding Term $(iii)$]
An appropriate choice of $\Cpoly$ ensures {\normalfont Term $(iii)$} is at most  $ \frac{\log (2M c_T/\delta)}{3\alpha_{\max} T} $.
\end{claim}
\begin{proof} Recall that $\text{\normalfont Term $(iii)$} = |\catonii[\bvtil,f] - \catround[\bvtil_{0},f_{0}]|$.  Recall that $\catonii[\bvtil,f]$ uses the data and parameter
\begin{align*}
X_t[\bvtil,f] := \frac{1}{\sigma^2_t}\bvtil^\top \bphi_t \ytil_t(f) \quad 
\alpha[\bvtil] =\min \left \{ \frac{ \sqrt{\log 2M c_T/\delta}}{\|\bvtil\|_{\bSigma_T}}, \alphamax \right \},
\end{align*}
whereas $\catround[\bvtil_{0},f_{0}]$ uses the same $X_t$, replaced with $\bvtil_0$ and $f_0$, and uses the rounded version $\alpharound[\bvtil]$. To compute the sensitivity bound, we consider differences between various quantities of interest. Throughout, we use
\begin{align*}
|\ytil_t(f)| := |\inner{\bphi_t}{\bust} + f(\bphi'_t) + \eta_t| \le \beta_{u} + H + \beta_{\eta}
\end{align*}
\paragraph{Difference in scalar data.} We have
\begin{align*}
\left|X_t[\bvtil,f] - X_t[\bvtil_0,f_0]\right| &= \left|\frac{1}{\sigma^2_t}\bvtil^\top \bphi_t \ytil_t(f) - \frac{1}{\sigma^2_t}\bvtil_0^\top \bphi_t \ytil_t(f_0)\right|\\
&\le \frac{\|\bphi_t\|_2}{\sigma^2_t} \cdot\left(|\ytil_t(f_0)|\|\bvtil - \bvtil_0\|_2 + \|\bvtil\|_2 |\ytil_t(f) - \ytil_t(f_0)|\right)\\
&\le \frac{1}{\sigma^2_{\min}} \cdot\left(H\|\bvtil - \bvtil_0\|_2 + \|\bvtil\|_2 |f(\bphi_t') - f_0(\bphi_t')|\right)\\
&\le \epsilon_X := \frac{1}{\sigma^2_{\min}} \cdot\left((\beta_{u} + H + \beta_{\eta})\|\bvtil - \bvtil_0\|_2 + \beta_{\tilde{v}} \dist_{\infty}(f,f_0)\right). 
\end{align*}

\paragraph{Difference in Catoni parameters.}
Setting $c = \sqrt{\log 2M c_T/\delta}$ and $a(\bvtil) := \|\bvtil\|_{\bSigma_T}$, we can express 
\begin{align*}
\alpha[\bvtil] = \min\left\{\frac{c}{a(\bvtil)},\alpha_{\max}\right\}  = \frac{c }{\max\{a(\bvtil),c/\alpha_{\max}\}}.
\end{align*} 
This gives that the difference between the unrounded parameter with $\bvtil$, $\alpha[\bvtil]$, and the \emph{also unrounded} parameter $\alpha[\bvtil_0]$ with $\bvtil_0$ are bounded as
\begin{align*}
|\alpha[\bvtil] - \alpha[\bvtil_0]|  &= c\left|\frac{1 }{\max\{a(\bvtil),c/\alpha_{\max}\}} - \frac{1 }{\max\{a(\bvtil),c/\alpha_{\max}\}}\right|\\
&= c\cdot\left|\frac{\max\{a(\bvtil),\frac{c}{\alpha_{\max}}\} - \max\{a(\bvtil_0),\frac{c}{\alpha_{\max}}\} }{\max\{a(\bvtil),\frac{c}{\alpha_{\max}}\} \cdot \max\{a(\bvtil_0),\frac{c}{\alpha_{\max}}\} }  \right|\\
&\le c \frac{|a(\bvtil) - a(\bvtil_0)|}{\frac{c^2}{\alpha_{\max}^2}} = \frac{\alpha_{\max}^2 |a(\bvtil) - a(\bvtil_0)|}{c}\\
&\le \alpha_{\max}^2\sqrt{\sigma_{\min}^{-2} T}\cdot\|\bvtil - \bvtil_{0}\|_2 
\end{align*}
where the last line uses the definition of $c$, and the argument of \Cref{claim:term_i} to bound $|\alpha[\bvtil] - \alpha[\bvtil_0]|$, as well as $\log (2M c_T/\delta) \le 1$. 

Note however the $\catround[\bvtil_0,f_0]$ uses the \emph{rounded} parameter $\alpharound[\bvtil_0]$. Directly from its definition, we can see that
\begin{align*}
\alpha[\bvtil_0] \le \alpharound [\bvtil_0] \le (1+\epsround)\alpha[\bvtil_0],
\end{align*}
so that 
\begin{align*}
|\alpha[\bvtil_0] - \alpharound [\bvtil_0] | \le \epsround \alpha[\bvtil_0] \le \epsround \alpha_{\max}.
\end{align*}
By the triangle inequality, we therefore conclude
\begin{align*}
|\alpha[\bvtil] - \alpharound[\bvtil_0]| &\le |\alpha[\bvtil] - \alpha[\bvtil_0]| + |\alpha[\bvtil_0] - \alpharound [\bvtil_0] | \\
&\le  \underbrace{\alpha_{\max}^2\sqrt{\sigma_{\min}^{-2} T}\cdot\|\bvtil - \bvtil_{0}\|_2 + \epsround \alpha_{\max}}_{:= \epsilon_{\alpha}}.
\end{align*}

\paragraph{Upper bounding data norms and lower bound $\alpha[\bvtil]$.} We have
\begin{align*}
\max\{|X_t[\bvtil,f]|, |X_t[\bvtil_0,f_0]|\} &\le\sigma_t^{-2}\|\bphi\|_2\max\{\|\bvtil\|_2,\|\bvtil_0\|\}\cdot\max\{|\ytil_t(f),\ytil_f(f_0)|\}\\
&\le\sigma^{-2}_{\min}\beta_{\tilde{v}}(H + \beta_u+\beta_{\eta}) := \gamma_X.
\end{align*}
and, upper bounding $\|\bvtil\|_{\bSigma_T}^2 \le \beta_{\tilde{v}}^2 T/\sigma_{\min}^2$,
\begin{align*}
\alpha[\bvtil] =\min \left \{ \frac{\sqrt{ \log 2M c_T/\delta}}{\|\bvtil\|_{\bSigma_T}}, \alphamax \right \}  &\ge \min \left \{ \sqrt{\frac{ \log 2M c_T/\delta}{\beta_{\tilde{v}}^2 T/\sigma_{\min}^2}}, \alphamax \right \}\\
&\ge \alpha_{-} := \min\left\{\frac{1}{\beta_{\tilde{v}}\sqrt{T\sigma_{\min}^2}}, \alphamax \right \},
\end{align*}

We now invoke a perturbation bound for the Catoni estimator (\Cref{lem:catoni_perturb}), which ensures that, as long as
\begin{align*}
\epsilon :=  \alpha(\bvtil) \epsilon_X + 3  \gamma_X \epsilon_{\alpha} \le \frac{1}{18}\min \left \{ 1,  \alpha(\bvtil)^2 \gamma^2 \right \}, 
\end{align*}
for which it suffices that 
\begin{align*}
\alpha_{\max} \epsilon_X + 3  \gamma_X \epsilon_{\alpha} \le \frac{1}{18}\min \left \{ 1,  \frac{\gamma_X^2}{\alpha_{\max}^2},  \sigma_{\min}^{-2}(H + \beta_u)^2/T \right \}, 
\end{align*}
we will have
\begin{align*}
| \zst - \ztilst | \le \frac{1 + 2 \alpha(\bvtil) \gamma}{\alpha(\bvtil)} \epsilon + \sqrt{\frac{2 \epsilon}{ \alpha(\bvtil)^2}} \le \frac{1 + 2 \alpha_{\max} \gamma}{\alpha_-} \epsilon + \sqrt{\frac{2 \epsilon}{ \alpha_-}}
\end{align*}
Examining the above bounds, we have that $|\catonii(\bvtil,f) - \catonii(\bvtil_{0},f_{0})| \le \epsilon_0$ provided
\begin{align}
\max\{\|\bvtil - \bvtil_0\|_2, \dist_{\infty}(f,f_0),\epsround\} \le \epsilon_0^2 \cdot 1/\poly(\sigma_{\min}^{-2},T,\beta_{\tilde v},\alpha_{\max},H,\beta_u,\beta_{\eta}), \quad \epsilon_0 \le 1. 
\end{align}
The bound follows by taking $\epsilon_0$ to be  $ \frac{1}{3\alpha_{\max} T} \le \frac{\log (2M c_T/\delta)}{3\alpha_{\max} T} $.
\end{proof}
\end{proof}

\subsection{Linear Approximation to Catoni}\label{sec:lin_approx_cat}
\begin{proof}[Proof of \Cref{lem:cat_lin_approx}]
By definition of $\bthetahat$ and \eqref{eq:lin_approx_cat_good}, we will have that
\begin{align*}
\sup_{\bv  \in \cV} \frac{| \innerb{\bthetahat}{\bv} - \catonii[\bLambda^{-1} \bv] |}{\| \bv \|_{\bLambda^{-1}} } & = \min_{\btheta} \sup_{\bv \in \cV} \frac{| \innerb{\btheta}{\bv} - \catonii[\bLambda^{-1} \bv] |}{\| \bv \|_{\bLambda^{-1}} } \\
& \le \sup_{\bv  \in \cV} \frac{| \innerb{\bthetast}{\bv} - \catonii[\bLambda^{-1} \bv] |}{\| \bv \|_{\bLambda^{-1}} }  \\
&  \le C_1 + \frac{C_2}{T \| \bv \|_{\bLambda^{-1}}} .
\end{align*}
Rearranging this implies that for all $\bv \in \cV$, 
\begin{align*}
 \frac{| \innerb{\bthetahat}{\bv} - \catonii[\bLambda^{-1} \bv] |}{\| \bv \|_{\bLambda^{-1}} } \le C_1 + \frac{C_2}{T \| \bv \|_{\bLambda^{-1}}} \iff | \innerb{\bthetahat}{\bv} - \catonii[\bLambda^{-1} \bv] | \le C_1 \| \bv \|_{\bLambda^{-1}} + \frac{C_2}{T} .
\end{align*}
Similarly,
\begin{align*}
\sup_{\bv  \in \cV}\frac{| \innerb{\bthetahat}{\bv} - \zeta[\bLambda^{-1} \bv] |}{\| \bv \|_{\Lambda^{-1}} } & \le \sup_{\bv \in \cV} \frac{| \innerb{\bthetahat}{\bv} - \catonii[\bLambda^{-1} \bv] | + | \catonii[\bLambda^{-1} \bv] - \zeta[\bLambda^{-1} \bv] |}{\| \bv \|_{\bLambda^{-1}} } \le 2C_1 + \frac{2C_2}{T \| \bv \|_{\bLambda^{-1}}}  . 
\end{align*}
\end{proof}

\begin{lem}\label{lem:cat_lin_approx_efficient}
Assume that, for all $\bv \in \cV$ we have
\begin{align*}
| \catonii[\bLambda^{-1} \bv] -\zeta[\bLambda^{-1} \bv] | \le C_1 \| \bv \|_{\bLambda^{-1}} + C_2/T
\end{align*}
for some $\cV \subseteq \R^d$, $\bm{0} \not\in \cV$, and $\zeta[\bLambda^{-1} \bv] = \inner{\bv}{\bthetast}$. Let $\bu_1,\ldots,\bu_d$ denote the eigenvectors of $\bLambda$, and set
\begin{align*}
\bthetahat = [\bu_1, \ldots, \bu_d] \cdot [ \catonii[\bLambda^{-1} \bu_1], \ldots, \catonii[\bLambda^{-1} \bu_d]]^{\top} .
\end{align*}
Then, for all $\bv \in \cV$,
\begin{align*}
& | \innerb{\bv}{\bthetahat} - \catonii[\bLambda^{-1} \bv] | \le (\sqrt{d}  + 1) C_1 \| \bv \|_{\bLambda^{-1}} + (\sqrt{d} \cdot \sup_{\bv' \in \cV} \| \bv' \|_2 + 1) C_2/T, \\
& | \innerb{\bv}{\bthetahat} - \zeta[\bLambda^{-1} \bv] | \le (\sqrt{d} + 2) C_1 \| \bv \|_{\bLambda^{-1}} + (\sqrt{d} \cdot \sup_{\bv' \in \cV} \| \bv' \|_2 + 2) C_2/T .
\end{align*}
\end{lem}
\begin{proof}
Let 
\begin{align*}
\bthetatil = \argmin_{\btheta} \max_{\bv \in \cV} \frac{| \innerb{\btheta}{\bv} - \catonii[\bLambda^{-1} \bv] |}{\| \bv \|_{\bLambda^{-1}} }.
\end{align*}
Fix some $\bv \in \cV$, and express $\bv$ as $\bv = \sum_{i=1}^d a_i \bu_i$. Then,
\begin{align*}
|\innerb{\bv}{\bthetahat} - \innerb{\bv}{\bthetatil}| = |\sum_{i=1}^d a_i \innerb{\bu_i}{\bthetahat - \bthetatil} | \le \sum_{i=1}^d | a_i | | \innerb{\bu_i}{\bthetahat - \bthetatil} |  .
\end{align*}
By construction, we will have that $\innerb{\bu_i}{\bthetahat} = \catonii[\bLambda^{-1} \bu_i]$. Furthermore, by \Cref{lem:cat_lin_approx}, $\bthetatil$ will satisfy $| \innerb{\bu_i}{\bthetatil} - \catonii[\bLambda^{-1} \bu_i] | \le C_1 \| \bu_i \|_{\bLambda^{-1}} + C_2/T$. Thus,
\begin{align*}
 \sum_{i=1}^d | a_i | | \innerb{\bu_i}{\bthetahat - \bthetatil} | & =  \sum_{i=1}^d | a_i | | \catonii[\bLambda^{-1} \bu_i] - \innerb{\bu_i}{\bthetatil} | \\
 & \le \sum_{i=1}^d | a_i | (C_1 \| \bu_i \|_{\bLambda^{-1}} + C_2/T) \\
 & \le C_1 \sqrt{d} \sqrt{\sum_{i=1}^d a_i^2 \| \bu_i \|_{\bLambda^{-1}}^2} + \frac{C_2}{T}  \sum_{i=1}^d |a_i|
\end{align*}
where the last inequality follows by Cauchy-Schwarz. Since $\bu_i$ are the eigenvectors of $\bLambda_T$ and are therefore orthogonal, we will have that 
\begin{align*}
\sqrt{\sum_{i=1}^d a_i^2 \| \bu_i \|_{\bLambda_T^{-1}}^2} = \sqrt{( \sum_{i=1}^d a_i \bu_i)^\top \bLambda_T^{-1} ( \sum_{i=1}^d a_i \bu_i)} = \| \bv \|_{\bLambda_T^{-1}} .
\end{align*}
Finally, we can bound
\begin{align*} 
\sum_{i=1}^d | a_i | \le \sqrt{d} \sqrt{\sum_{i=1}^d a_i^2} = \sqrt{d} \sqrt{ \| \bv \|_2^2} \le \sqrt{d} \cdot \sup_{\bv' \in \cV} \| \bv' \|_2 .
\end{align*}
The first result follows by upper bounding
\begin{align*}
| \innerb{\bv}{\bthetahat} - \catonii[\bLambda^{-1} \bv] | \le | \innerb{\bv}{\bthetahat} - \innerb{\bv}{\bthetatil} | + | \innerb{\bv}{\bthetatil} - \catonii[\bLambda^{-1} \bv] |
\end{align*}
and again applying \Cref{lem:cat_lin_approx}. The second result follows since
\begin{align*}
| \innerb{\bv}{\bthetahat} - \zeta[\bLambda^{-1} \bv] | \le |\innerb{\bv}{\bthetahat} - \catonii[\bLambda^{-1} \bv]| + | \catonii[\bLambda^{-1} \bv] - \zeta[\bLambda^{-1} \bv]| 
\end{align*}
and using the first result and the assumption on $| \catonii[\bLambda^{-1} \bv] - \zeta[\bLambda^{-1} \bv]|$. 
\end{proof}

\begin{lem}\label{lem:thetahat_norm_bound}
Let $\catonii[\bLambda^{-1} \bv]$ denote a Catoni estimate as defined in \Cref{lem:catoni2}, and assume that the data used to form $\catonii[\bLambda^{-1} \bv]$, $\{ X_t(\bv,\bLambda) \}_{t=1}^T$, satisfies $|X_t(\bv,\bLambda)| \le \gamma \| \bv \|_{\bLambda^{-1}}$ for all $t$ and $\bv$, and that $\| \bv \|_{\bLambda^{-1}} \le \| \bv \|_2 / \sqrt{\lambda}$. Then for $\bthetahat$ as defined in \Cref{lem:cat_lin_approx}, if we have $\cS^{d-1} \subseteq \cV$, for $\cS^{d-1}$ the unit sphere, we have:
\begin{align*}
\| \bthetahat \|_2 \le 2\gamma/\sqrt{\lambda}  
\end{align*}
and for $\bthetahat$ as defined in \Cref{lem:cat_lin_approx_efficient} and any $\cV$:
\begin{align*}
\| \bthetahat \|_2 \le \sqrt{d}\gamma/\sqrt{\lambda}  . 
\end{align*}
\end{lem}
\begin{proof}
By \Cref{claim:cat_sln_bound} and our assumption that $|X_t(\bv,\bLambda)| \le \gamma \| \bv \|_2$, we can bound $|\catonii[\bLambda^{-1} \bv]| \le \gamma \| \bv \|_2$. First consider setting $\bthetahat$ as in \Cref{lem:cat_lin_approx}. Fix $\bv \in \cS^{d-1}$. By assumption $\bv \in \cV$, so we have
\begin{align*}
| \innerb{\bthetahat}{\bv} | & \le \| \bv \|_{\bLambda^{-1}} \frac{| \innerb{\bthetahat}{\bv} - \catonii[\bLambda^{-1} \bv] |}{\| \bv \|_{\bLambda^{-1}}} + | \catonii[\bLambda^{-1} \bv] | \\
& \le \| \bv \|_{\bLambda^{-1}} \left [ \sup_{\bv' \in \cV} \frac{| \innerb{\bthetahat}{\bv'} - \catonii[\bLambda^{-1} \bv'] |}{\| \bv' \|_{\bLambda^{-1}}} \right ]+ | \catonii[\bLambda^{-1} \bv] | \\
& = \| \bv \|_{\bLambda^{-1}} \left [ \min_{\btheta} \sup_{\bv' \in \cV} \frac{| \innerb{\btheta}{\bv'} - \catonii[\bLambda^{-1} \bv'] |}{\| \bv' \|_{\bLambda^{-1}}} \right ]+ | \catonii[\bLambda^{-1} \bv] | \\
& \le \| \bv \|_{\bLambda^{-1}} \left [  \sup_{\bv' \in \cV} \frac{| \catonii[\bLambda^{-1} \bv'] |}{\| \bv' \|_{\bLambda^{-1}}} \right ]+ | \catonii[\bLambda^{-1} \bv] | \\
& \le \| \bv \|_{\bLambda^{-1}}\left [  \sup_{\bv' \in \cV} \frac{\gamma \| \bv' \|_{\bLambda^{-1}}}{\| \bv' \|_{\bLambda^{-1}}} \right ]+ \gamma \| \bv \|_{\bLambda^{-1}} \\
& \le \gamma \| \bv \|_{\bLambda^{-1}}  + \gamma \| \bv \|_{\bLambda^{-1}} \\
& \le 2 \gamma/\sqrt{\lambda}.
\end{align*}
As this holds for all $\bv \in \cS^{d-1}$, it follows that $\| \bthetahat \|_2 \le 2 \gamma/\sqrt{\lambda}$.

If $\bthetahat$ is set as in \Cref{lem:cat_lin_approx_efficient}, we have that
\begin{align*}
\| \bthetahat \|_2 \le \sqrt{\sum_{i=1}^d \catonii[\bLambda^{-1} \bu_i]^2} \le \sqrt{ \sum_{i=1}^d \gamma^2 \| \bu_i \|_{\bLambda^{-1}}^2} = \sqrt{d/\lambda} \gamma .
\end{align*} 
\end{proof}

\subsection{Catoni Perturbation Analysis}\label{app:Catoni_perturb}

\begin{lem}\label{lem:catoni_perturb}
Consider some fixed $\cX := \{ X_t \}_{t=1}^T, \cXtil := \{ \Xtil_t \}_{t=1}^T$ satisfying $| X_t | \le \gamma, | \Xtil_t | \le \gamma$ for all $t$, and some fixed $\alpha > 0, \alphatil > 0$. Let $\zst$ denote the root of the function $\fcatoni(z; \cX, \alpha)$ and $\ztilst$ the root of $\fcatoni(z;\cXtil,\alphatil)$. Then, assuming that
\begin{align*}
\epsilon := \frac{1}{T}\sum_{t=1}^T \alpha | X_t - \Xtil_t | + 3  \gamma | \alpha - \alphatil |  \le  \frac{1}{18}\min \left \{ 1,  \alpha^2 \gamma^2 \right \} 
\end{align*}
we will have
\begin{align*}
| \zst - \ztilst | \le \frac{1 + 2 \alpha \gamma}{\alpha} \epsilon + \sqrt{\frac{2 \epsilon}{ \alpha^2}} .
\end{align*}

\end{lem}
\begin{proof}
For simplicity, we will denote $f(z) := \fcatoni(z ; \cX, \alpha)$ and $\ftil(z) := \fcatoni(z; \cXtil, \alphatil)$. Fix some $\Delta \ge 0$ with $\Delta \le \gamma$. Note that $f(z)$ is differentiable, even at $z=0$, and
\begin{align*}
\frac{d}{dz} f(z) = \sum_{t=1}^T -\alpha \psicat'(\alpha ( X_t - z)), \qquad \psicat'(y) = \left \{ \begin{matrix} \frac{1 + y}{1 + y + y^2/2} & y \ge 0 \\
\frac{1-y}{1 - y + y^2/2} & y < 0 \end{matrix} \right . 
\end{align*}
By the Mean Value Theorem,
\begin{align*}
f ( \zst + \Delta) = f ( \zst + \Delta) - f(\zst) = f'(y) \Delta
\end{align*}
for some $y \in [\zst, \zst + \Delta]$ which implies that
\begin{align*}
0 \ge f ( \zst + \Delta) - \Delta \sup_{z \in [\zst, \zst + \Delta]} f'(z) . 
\end{align*}
Note that $\psicat'(y) \ge 0$ for all $y$, that $\psicat'(y)$ decreases as $|y|$ increases, and that $\psicat'(y) = \psicat'(-y)$. It follows that 
\begin{align*}
\sup_{z \in [\zst, \zst + \Delta]} -\alpha \psicat'(\alpha(X_t - z)) & \le - \alpha \psicat'( \alpha | X_t| + \alpha | \zst|  + \alpha \Delta ).
\end{align*}

\begin{claim}\label{claim:cat_sln_bound}
$ \zst \in [-\gamma,\gamma]$.
\end{claim}
\begin{proof}[Proof of \Cref{claim:cat_sln_bound}]
Recall that, by assumption, $| X_t | \le \gamma$. Furthermore, note that if $X_t - \zst < 0$ for all $t$, then $f(\zst) < 0$, and similarly, if $X_t - \zst > 0$ for all $t$, then $f(\zst) > 0$. Since $f(\zst) = 0$, this implies that $\max_t X_t \ge \zst$ and $\min_t X_t \le \zst$, which implies that $\zst \in [-\gamma,\gamma]$, and so $|\zst| \le \gamma$. 
\end{proof}

By \Cref{claim:cat_sln_bound}, we can upper bound
\begin{align*}
- \alpha \psicat'( \alpha | X_t| + \alpha | \zst|  + \alpha \Delta ) & \le - \alpha \psicat'( 2\alpha \gamma + \alpha \Delta ) \\
&  = -\alpha \cdot \frac{1 + \alpha (2 \gamma + \Delta)}{1 +  \alpha (2 \gamma + \Delta) +  \alpha^2 (2 \gamma + \Delta)^2/2}
\end{align*}
which implies that
\begin{align*}
\sup_{z \in [\zst, \zst + \Delta]} f'(z) & \le \sum_{t=1}^T \sup_{z \in [\zst, \zst + \Delta]} -\alpha \psicat'(\alpha(X_t - z)) \\
& \le - \frac{T \alpha + T \alpha^2 (2 \gamma + \Delta)}{1 +  \alpha (2 \gamma + \Delta) +  \alpha^2 (2 \gamma + \Delta)^2/2}
\end{align*}
so
\begin{align}\label{eq:cat_pert_ineq1}
0 = f(\zst) \ge f(\zst + \Delta) + \frac{T \alpha \Delta ( 1  +  \alpha (2 \gamma + \Delta))}{1 +  \alpha (2 \gamma + \Delta) +  \alpha^2 (2 \gamma + \Delta)^2/2} . 
\end{align}
Note that $|\psicat'(y)| \le 1$ for all $y$, which implies that $|\psicat(y) - \psicat(y')| \le | y - y'|$. It follows that
\begin{align*}
| f(\zst + \Delta) - \ftil(\zst + \Delta) | & \le \sum_{t=1}^T | \psicat(\alpha(X_t - \zst - \Delta)) - \psicat(\alphatil(\Xtil_t - \zst - \Delta)) | \\
& \le \sum_{t=1}^T | \alpha(X_t - \zst - \Delta) - \alphatil(\Xtil_t - \zst - \Delta) | \\
& \le \sum_{t=1}^T \alpha | X_t - \Xtil_t | + \sum_{t=1}^T | \alpha - \alphatil | | \Xtil_t | + T | \alpha - \alphatil | | \zst + \Delta | \\
& \le \sum_{t=1}^T \alpha | X_t - \Xtil_t | + 3 T \gamma | \alpha - \alphatil | \\
& =: \epsilon \cdot T
\end{align*}
Thus,
\begin{align*}
f(\zst + \Delta) & + \frac{T \alpha \Delta ( 1  +  \alpha (2 \gamma + \Delta))}{1 +  \alpha (2 \gamma + \Delta) +  \alpha^2 (2 \gamma + \Delta)^2/2} \\
& \ge \ftil(\zst + \Delta) + \frac{T \alpha \Delta ( 1  +  \alpha (2 \gamma + \Delta))}{1 +  \alpha (2 \gamma + \Delta) +  \alpha^2 (2 \gamma + \Delta)^2/2} - \epsilon T.
\end{align*}
If 
\begin{align}\label{eq:suf_del_cat}
\frac{ \alpha \Delta ( 1  +  \alpha (2 \gamma + \Delta))}{1 +  \alpha (2 \gamma + \Delta) +  \alpha^2 (2 \gamma + \Delta)^2/2} - \epsilon  \ge 0,
\end{align}
then by \eqref{eq:cat_pert_ineq1} it follows that $0 \ge \ftil(\zst + \Delta)$. Since $\ftil$ is monotonically decreasing in $z$ and $\ftil(\ztilst) = 0$, $\ztilst \le \zst + \Delta$. 

It remains to determine what choice of $\Delta$ is sufficient. Solving \eqref{eq:suf_del_cat} for $\Delta$, we will have that \eqref{eq:suf_del_cat} is met as long as
\begin{align*}
\Delta & \ge \frac{(1 + 2 \alpha \gamma)  \epsilon - (1 + 2 \alpha \gamma)   + \sqrt{- \epsilon^2 + 2  \epsilon  + (1 + 4 \alpha \gamma + 4 \alpha^2 \gamma^2)} }{2\alpha  - \alpha \epsilon} .
\end{align*}
By assumption we have that $1 \ge   \epsilon$, so $-\epsilon^2 + 2 \epsilon  + (1 + 4 \alpha \gamma + 4 \alpha^2 \gamma^2) $ is non-negative. We can then bound
\begin{align*}
& \frac{(1 + 2 \alpha \gamma)  \epsilon - (1 + 2 \alpha \gamma)   + \sqrt{- \epsilon^2 + 2  \epsilon  + (1 + 4 \alpha \gamma + 4 \alpha^2 \gamma^2) } }{2\alpha  - \alpha \epsilon}  \\
& \qquad \le \frac{(1 + 2 \alpha \gamma)  \epsilon - (1 + 2 \alpha \gamma)   + \sqrt{- \epsilon^2 + 2  \epsilon } + \sqrt{(1 + 4 \alpha \gamma + 4 \alpha^2 \gamma^2)} }{2\alpha  - \alpha \epsilon} \\
& \qquad \le \frac{(1 + 2 \alpha \gamma)  \epsilon   + \sqrt{- \epsilon^2 + 2 \epsilon }  }{\alpha } \\
& \qquad \le \frac{1 + 2 \alpha \gamma}{\alpha } \epsilon + \sqrt{\frac{2 \epsilon}{ \alpha^2 }} .
\end{align*}
A sufficient condition to meet \eqref{eq:suf_del_cat} is then
\begin{align*}
\Delta = \frac{1 + 2 \alpha \gamma}{\alpha } \epsilon + \sqrt{\frac{2 \epsilon}{ \alpha^2 }} .
\end{align*}
Thus, 
\begin{align*}
\ztilst \ge \zst + \frac{1 + 2 \alpha \gamma}{\alpha } \epsilon + \sqrt{\frac{2 \epsilon}{ \alpha^2 }}.
\end{align*}
We have required that $\Delta \le \gamma$, but note that this is met for this choice of $\Delta$ since we have assumed that
\begin{align*}
\epsilon \le \min \left \{ \frac{1}{6}, \frac{\alpha \gamma }{3}, \frac{\alpha^2 \gamma^2 }{18} \right \}
\end{align*}
and $\epsilon$ satisfying this will ensure that $\Delta \le \gamma$. Notice that the above condition is satisfied when
\begin{align*}
\epsilon \le \frac{1}{18}\min \left \{1, \alpha^2 \gamma^2\right\}.
\end{align*}
The result follows by repeating this argument in the opposite direction. 
\end{proof}


\section{Regret Analysis}\label{sec:regret_proofs}
We will consider a slightly more general setup here than that considered in the main text. In particular, we will allow for the reward function to be time-varying: at episodes $k$, the reward is specified by $r_h^k(s,a)$. We will make several assumptions on this reward.

\begin{asm}[Time-Varying Reward]\label{asm:time_var_reward}
The reward function $r_h^k(s,a) \in [0,1]$ is $\cF_{k-1}$-measurable, and non-increasing in $k$: $r_h^k(s,a) \le r_h^{k-1}(s,a)$ for all $s,a,h,k$. Furthermore, for each $h,k$, $r_h^k \in \Rclass$ for some function class $\Rclass$, and $\Rclass$ has covering number bounded as $\covnum(\Rclass,\dist_{\infty},\epsilon) \le \dR \log(1 + \frac{2\RR}{\epsilon})$. 
\end{asm}

As the reward function changes at each step, we will denote the value function for policy $\pi$ at episode $k$ by $Q_h^{k,\pi}(s,a)$ (and similarly $V_h^{k,\pi}(s)$). We will also redefine regret as
\begin{align*}
\cR_K := \sum_{k=1}^K (V_1^{k,\star} - V_1^{k,\pi_k}) 
\end{align*}
for $V_1^{k,\star}$ the optimal value function for reward $r^k$. To accommodate time-varying reward in \algname, the update of the optimistic $Q$-estimate on \Cref{line:update_q} must be changed to:
$$Q_h^{k}(\cdot , \cdot) \leftarrow \min \{ r_h^k(\cdot,\cdot) + \innerb{\bphi(\cdot ,\cdot)}{\what_h^{k}} + 6\beta \| \bphi(\cdot , \cdot) \|_{\bLambda_{h,k-1}^{-1}} + 12\sigmin \beta^2/k^2, H \}$$
and the following settings of $\Kinit$ and $\beta$ must be used:
\begin{align*}
\Kinit &\leftarrow c \left ( (d^2 + \dR) \log( \max \{ d, \sigmin^{-1}, K, H , \RR \} ) + \log (2HK/\delta) \right ) \\
\beta &\leftarrow 6 \sqrt{c (d^2 + \dR) \log \left ( \max \{ d, \sigmin^{-1}, H,  K, \RR \} \right ) + \log(2HK/\delta)} .
\end{align*}
We then have the following result.

\begin{thm}[Regret Bound for Time-Varying Reward]\label{thm:reg_bound_time_var}
Fix a failure probability $\delta \in (0,1)$ and $K \in \N$, and assume that the reward satisfies \Cref{asm:time_var_reward}. Then, the regret of \algname, modified to handle time-varying rewards as outlined above, satisfies the following bound with probability at least $1-3\delta$:
\begin{align*}
\cR_K & \le c_1 \sqrt{d(d^2 + \dR) H^3 \cdot \log(\RR HK/\delta)  \log^2(HK/\delta) \cdot \sum_{k=1}^K V_1^{k,\star} } \\
& \qquad \qquad + c_2 \sqrt{d} (d^2 + \dR)^{3/2} H^3 \log^{3/2}(\RR HK/\delta)  \log^{2}(HK/\delta) 
\end{align*}
for universal constants $c_1,c_2$. Furthermore, if we use the computationally efficient update as outlined in \Cref{lem:regret_bound_eff}, with probability at least $1-3\delta$, the regret is bounded by
\begin{align*}
\cR_K & \le c_1 \sqrt{d^2(d^2 + \dR) H^3 \cdot \log(\RR HK/\delta)  \log^2(HK/\delta) \cdot \sum_{k=1}^K V_1^{k,\star} } \\
& \qquad \qquad + c_2 d (d^2 + \dR)^{3/2} H^3 \log^{3/2}(\RR HK/\delta)  \log^{2}(HK/\delta) 
\end{align*}
and computation will scale polynomially in  $d,H,K,$ and $\min \{ |\cA|, \cO(2^d) \}$. 
\end{thm}

\Cref{lem:regret_bound} and \Cref{lem:regret_bound_eff} are direct corollaries of \Cref{thm:reg_bound_time_var}, where we simply set $r_h^k = r_h$ for all $k$, replace $ \sum_{k=1}^K V_1^{k,\star}$ with $K \Vst_1$, and note that since the reward is deterministic in this case, no cover over reward functions is necessary, so the regret scales independently of $\dR$ and $\RR$. Throughout the remainder of this section, we will consider this more general time-varying reward setting.

\subsection{Preliminaries and Notation}
Define the following events:
\begin{align*}
A_{k,h} & := \bigg \{  \left | \catonii_{h,k} \left [ (k-1)  \bLambda_{h,k-1}^{-1}\bphi_{h,k}\right ] - \Exp_h[V_{h+1}^{k}](s_{h,k}, a_{h,k}) \right | \le \beta \| \bphi_{h,k} \|_{\bLambda_{h,k-1}^{-1}} + \sigmin \beta^2/k^2 \bigg \} \\
B_{k,h} & := \bigg \{ \forall \bv \in \Ball^d \ : \ \left | \Exphat_h[V_{h+1}^k](\bv) - \Exp_h[V_{h+1}^{k}] (\bv) \right |  \le \beta \| \bv \|_{\bLambda_{h,k-1}^{-1}} + \sigmin \beta^2/k^2 \bigg  \} \\
\cE & := \bigcap_{k = \Kinit}^K \bigcap_{h=1}^H (B_{k,h} \cap A_{k,h})
\end{align*}
where we denote $\Exphat_h[V_{h+1}^k](\bv) = \catonii_{h,k}[(k-1)\bLambda_{h,k-1}^{-1} \bv]$, $\beta = 6 \sqrt{\Cmdp + \log(2HK/\delta)}$, and
\begin{align*}
\Cmdp := c (d^2 + \dR) \cdot \logterm \left (  d, \sigmin^{-1}, H, 1/\lambda, K, \RR  \right ) 
\end{align*}
for a universal constant $c$. Here we overload notation slightly and define:
\begin{align*}
\Exp_h[V_{h+1}^{k}](\bv)  = \innerb{\bv}{\int   V_{h+1}^k(s') \rmd \bmu_h(s')}  .
\end{align*}
We will also define $r_h(\bv) = r_h(s,a)$ if $\bv = \bphi(s,a)$, and 0 otherwise. Throughout this section, we will also denote $\Exphat_h[V_{h+1}^k](s,a) := \Exphat_h[V_{h+1}^k](\bphi(s,a))$. 

The analysis of the computationally inefficient and computationally efficient versions of \algname are nearly identical, and we therefore prove them in tandem. To facilitate this, we will define the parameter
\begin{align*}
\betatil := \left \{ \begin{matrix}  2\beta & \efficient = \false \\
(\sqrt{d} + 2) \beta & \efficient = \true
\end{matrix} \right . 
\end{align*}
where the $\efficient$ flag corresponds to which version of the algorithm we are running: $\efficient = \false$ corresponds to running the version of \algname as stated in \Cref{alg:catoni_regret}, and $\efficient = \true$ corresponds to running the computationally efficient version as described in \Cref{sec:comp_eff_alg}. Given the definition of $\betatil$, we can then write the update to $Q_h^k$ as
\begin{align*}
Q_h^{k}(\cdot , \cdot) \leftarrow \min \{  r_h^k(\cdot,\cdot) + \innerb{\bphi(\cdot ,\cdot)}{\what_h^{k}} + 3\betatil \| \bphi(\cdot , \cdot) \|_{\bLambda_{h,k-1}^{-1}} + 3\sigmin \betatil^2/k^2, H \},
\end{align*}
and this update holds in either the efficient or inefficient case. We will use $\betatil$ throughout the analysis, and set $\sigmin = 1/K$, $\alphamax = K/\sigmin$ as in \algname.

\subsection{Catoni Estimation is Correct for Linear MDPs}

\begin{lem}\label{lem:linear_mdp_reg}
Consider the function class
\begin{align*}
\Fclassmdp = \Big \{ f(\cdot) = \min \{ r(\cdot) + \innerb{\cdot}{\bw} + 3\betatil \| \cdot \|_{\bLambda^{-1}} + \bar{c}, H \} \ : \ \| \bw \|_2 \le 4H\sqrt{d}/(\lambda \sigmin^2), \ \bLambda \succeq \lambda I, r \in \Rclass \Big \} 
\end{align*}
and assume $\bar{c} \ge 0$, $\covnum(\Rclass,\dist_{\infty},\epsilon) \le \dR \log(1 + \frac{2 \RR}{\epsilon} )$. Then, 
$$\covnum(\Fclassmdp,\dist_{\infty},\epsilon) \le (4d^2 + \dR) \log \left ( 1 + \frac{\sqrt{d} (288 \betatil^2 + 8H/\sigmin^2)/\lambda + 4 \RR}{\epsilon} \right ) .$$
Furthermore, conditioned on the event $\cap_{\tau = \Kinit}^{k-2} \cap_{h'=h+1}^H (B_{\tau , h'} \cap B_{k-1,h'}) \cap A_{\tau,h}$, the Catoni estimation problems on \Cref{line:sighat} at episode $k$ of \algname
are instances of the regression with function approximation setting of \Cref{defn:regression_with_function_approx} for 
\begin{align*}
\beta_\mu = \sqrt{d}, \quad \beta_u = 0, \quad \Fclass = \Fclassmdp .
\end{align*}
Similarly, conditioned on the event $\cap_{\tau = \Kinit}^{k-1} \cap_{h'=h+1}^H (B_{\tau , h'} \cap B_{k,h'}) \cap A_{\tau,h}$, the Catoni estimation problems on \Cref{line:qhat_lin1} and in the computationally efficient update of \Cref{eq:efficient_update}, are instances of the regression with function approximation setting of \Cref{defn:regression_with_function_approx} for 
\begin{align*}
\beta_\mu = \sqrt{d}, \quad \beta_u = 0, \quad \Fclass = \Fclassmdp .
\end{align*}
\end{lem}
\begin{proof}
We will instantiate \Cref{defn:regression_with_function_approx} with $\bphi_\tau = \bphi_{h,\tau}$, $\bphi_\tau' = \bphi_{h+1,\tau}$, $\bmu = \bmu_h$, $\sigma_\tau = \sighat_{h,\tau}$, and the function class $\Fclass = \Fclassmdp$. \algname solves two different forms of regression problems. In the first setting, when solving for $\sighat_{h,k-1}^2$ on \Cref{line:sighat}, we consider $y_\tau = V_{h+1}^{k-1}(s_{h+1,\tau})$, and  $\bust = 0$. In the second, when solving either $\catonii_{h,k}[(k-1)\bLambda_{h,k-1}^{-1} \bv]$ or $\catonii_{h,k}[(k-1)\bLambda_{h,k-1}^{-1} \bu_i]$, we set $y_\tau = V_{h+1}^k(s_{h+1,\tau})$ and set $\bust = 0$. 

We verify that this meets the criteria of \Cref{defn:regression_with_function_approx}. First, note that by definition of $\cF_{h,\tau}$, we will have that $\bphi_{h,\tau}$ is $\cF_{h,\tau}$-measurable and that $\bphi_{h+1,\tau}$ is $\cF_{h+1,\tau}$-measurable. In addition, $\| \bphi_{h,\tau} \|_2 \le 1$ and $\| \bphi_{h+1,\tau} \|_2 \le 1$ by assumption. Given the linear MDP structure of \Cref{defn:linear_mdp}, for any bounded function $f$,
\begin{align*}
\Exp[ f(\bphi_{h+1,\tau}) \mid \cF_{h,\tau}] = \innerb{\bphi_{h,\tau}}{ \int  f(\bphi(s',\pi_{h+1}^\tau(s'))) \rmd \bmu_h(s')}. 
\end{align*}
Note that we can think of $\bmu_h(\cdot)$ as a measure over $\R^d$, as required by \Cref{defn:regression_with_function_approx}, by associating $s'$ with $\bphi(s',\pi_{h+1}^\tau(s'))$, and putting a measure of 0 on all vectors $\bv$ such that there does not exist $s,a$ with $\bv = \bphi(s,a)$. In addition, by assumption $\| |\bmu_h|(\cS) \|_2 \le \sqrt{d}$, so we can take $\beta_\mu = \sqrt{d}$.

In both settings, since $\bust = 0$, it suffices to take $\beta_u = 0$. Note that for any $s,a,h,k,\tau$, we can bound
\begin{align*}
 | \bphi(s,a)^\top & \bLambda_{h,k-1}^{-1} \bphi_{h,\tau}   V_{h+1}^k(s_{h+1,\tau}) / \sighat_{h,\tau}^2 | \\
& \le \| \bphi(s,a) \|_{\bLambda_{h,k-1}^{-1}}  \cdot \| \bLambda_{h,k-1}^{-1/2} \|_\op \| \bphi_{h,\tau} \|_2  | V_{h+1}^k(s_{h+1,k})| /  \sighat_{h,\tau}^2 \\
& \le \| \bphi(s,a) \|_{\bLambda_{h,k-1}^{-1}} \cdot \frac{ H}{\sqrt{\lambda} \sigmin^2} 
\end{align*}
so by \Cref{lem:thetahat_norm_bound}, we will have that $\| \what_{h+1}^k \|_2 \le \frac{4 H \sqrt{d}}{\lambda \sigmin^2}$. It follows that, by construction of $V_{h+1}^k(\cdot)$, we will have $V_{h+1}^k(\cdot) \in \Fclassmdp$.

It remains to show that the condition on $\sighat_{h,\tau}$, \eqref{eq:cat_union_sighat}, is met at round $k$. In our setting, for the Catoni estimation on \Cref{line:sighat} at episode $k$, \eqref{eq:cat_union_sighat} is equivalent to
\begin{align*}
\Exp_h[(V_{h+1}^{k-1})^2](s_{h,\tau},a_{h,\tau}) \le \frac{1}{2} \sighat_{h,\tau}^2 .
\end{align*}
However, by \Cref{lem:sigma_valid2}, this holds for all $\tau \ge \Kinit$ on the $\cap_{\tau = \Kinit}^{k-2} \cap_{h'=h+1}^H (B_{\tau , h'} \cap B_{k-1,h'}) \cap A_{\tau,h}$. For $\tau \le \Kinit$ it trivially as we set $\sighat_{h,\tau}^2 = 2H^2$ and since $V_{h+1}^{k}(s') \in [0,H]$. For the Catoni estimation on \Cref{line:qhat_lin1} or in \Cref{eq:efficient_update}, \eqref{eq:cat_union_sighat} is equivalent to
\begin{align*}
\Exp_h[(V_{h+1}^{k})^2](s_{h,\tau},a_{h,\tau}) \le \frac{1}{2} \sighat_{h,\tau}^2 .
\end{align*}
Again by \Cref{lem:sigma_valid2}, this holds on the event $\cap_{\tau = \Kinit}^{k-1} \cap_{h'=h+1}^H (B_{\tau , h'} \cap B_{k,h'}) \cap A_{\tau,h}$ for $\tau \ge \Kinit$. For $\tau \le \Kinit$ this trivially holds since $\sighat_{h,\tau}^2 = 2H^2$.

Finally, we bound the covering number of $\Fclassmdp$. Consider $f_1,f_2 \in \Fclassmdp$, then 
\begin{align*}
\dist_{\infty}(f_1,f_2) & = \sup_{\bphi \in \cB^d} | f_1(\bphi) - f_2(\bphi)| \\
& = \sup_{\bphi \in \cB^d} \Big | \min \{ r_1(\bphi) + \innerb{\bphi}{\bw_1} + 3\betatil \| \bphi \|_{\bLambda_1^{-1}} + \bar{c}, H \} \\
& \qquad \qquad \qquad - \min \{ r_2(\bphi) + \innerb{\bphi}{\bw_2} + 3\betatil \| \bphi \|_{\bLambda_2^{-1}} + \bar{c}, H \} \Big | .
\end{align*}
Assume that $f_1(\bphi) = f_2(\bphi) = H$, then we can clearly bound
\begin{align*}
| f_1(\bphi) - f_2(\bphi)| \le | r_1(\bphi) - r_2(\bphi) | + |  \min \{  \innerb{\bphi}{\bw_1} + 3\betatil \| \bphi \|_{\bLambda_1^{-1}} , H \} -  \min \{  \innerb{\bphi}{\bw_2} + 3\betatil \| \bphi \|_{\bLambda_2^{-1}}, H \} | .
\end{align*}
If $f_1(\bphi) < H, f_2(\bphi) < H$, using that $\bar{c}, r_1(\bphi), r_2(\bphi) \ge 0$, we can bound
\begin{align*}
| f_1(\bphi) - f_2(\bphi)| & =  |  r_1(\bphi) + \innerb{\bphi}{\bw_1} + 3\betatil \| \bphi \|_{\bLambda_1^{-1}}  - (r_2(\bphi) + \innerb{\bphi}{\bw_2} + 3\betatil \| \bphi \|_{\bLambda_2^{-1}} ) | \\
& =  |  r_1(\bphi) + \min \{ \innerb{\bphi}{\bw_1} + 3\betatil \| \bphi \|_{\bLambda_1^{-1}},H \}  - (r_2(\bphi) + \min \{ \innerb{\bphi}{\bw_2} + 3\betatil \| \bphi \|_{\bLambda_2^{-1}}, H \} ) | \\
& \le | r_1(\bphi) - r_2(\bphi) | + |  \min \{  \innerb{\bphi}{\bw_1} + 3\betatil \| \bphi \|_{\bLambda_1^{-1}} , H \} -  \min \{  \innerb{\bphi}{\bw_2} + 3\betatil \| \bphi \|_{\bLambda_2^{-1}}, H \} |.
\end{align*}
If $f_1(\bphi) = H, f_2(\bphi) < H$,
\begin{align*}
| f_1(\bphi) - f_2(\bphi)| & \le | r_1(\bphi) + \min \{ \innerb{\bphi}{\bw_1} + 3\betatil \| \bphi \|_{\bLambda_1^{-1}} , H \} - (r_2(\bphi) + \innerb{\bphi}{\bw_2} + 3\betatil \| \bphi \|_{\bLambda_2^{-1}} ) | \\
& =  |  r_1(\bphi) + \min \{ \innerb{\bphi}{\bw_1} + 3\betatil \| \bphi \|_{\bLambda_1^{-1}},H \}  - (r_2(\bphi) + \min \{ \innerb{\bphi}{\bw_2} + 3\betatil \| \bphi \|_{\bLambda_2^{-1}}, H \} ) | \\
& \le | r_1(\bphi) - r_2(\bphi) | + |  \min \{  \innerb{\bphi}{\bw_1} + 3\betatil \| \bphi \|_{\bLambda_1^{-1}} , H \} -  \min \{  \innerb{\bphi}{\bw_2} + 3\betatil \| \bphi \|_{\bLambda_2^{-1}}, H \} |.
\end{align*}
The same argument holds of $f_1(\bphi) < H, f_2(\bphi) = H$. Altogether then, 
\begin{align*}
\dist_{\infty}(f_1,f_2) & \le \sup_{\bphi \in \cB^d} | r_1(\bphi) - r_2(\bphi) | \\
& \qquad + \sup_{\bphi \in\cB^d} |  \min \{  \innerb{\bphi}{\bw_1} + 3\betatil \| \bphi \|_{\bLambda_1^{-1}} , H \}  - \min \{  \innerb{\bphi}{\bw_2} + 3\betatil \| \bphi \|_{\bLambda_2^{-1}} , H \}| .
\end{align*}
It follows that we can construct $\epsilon/2$-nets of $\Rclass$ and the class
\begin{align*}
\widetilde{\Fclass} := \Big \{ f(\cdot) = \min \{  \innerb{\cdot}{\bw} + 3\betatil \| \cdot \|_{\bLambda^{-1}}, H \} \ : \ \| \bw \|_2 \le 4H\sqrt{d}/(\lambda \sigmin^2), \ \bLambda \succeq \lambda I \Big \} 
\end{align*}
separately, and the union of these nets will serve as an $\epsilon$-net of $\Fclassmdp$. By assumption, we have $\covnum(\Rclass,\dist_{\infty},\epsilon/2) \le \dR \log(1 + \frac{4 \RR}{\epsilon} )$. Furthermore, $\widetilde{\Fclass}$ is identical to the function class considered in \Cref{lem:q_fun_cover}, so 
\begin{align*}
\covnum(\widetilde{\Fclass},\dist_{\infty},\epsilon/2) & \le d \log \left ( 1 + \frac{8 H \sqrt{d}}{\lambda \sigmin^2 \epsilon} \right ) + d^2 \log \left ( 1 + \frac{288 \sqrt{d} \betatil^2}{\lambda \epsilon^2} \right ) \\
& \le 4 d^2 \log \left ( 1 + \frac{\sqrt{d} (288 \betatil^2 + 8H/\sigmin^2)}{\lambda \epsilon} \right ) .
\end{align*}
This implies that (since log-covering numbers are additive)
\begin{align*}
\covnum(\Fclassmdp,\dist_{\infty},\epsilon) & \le 4 d^2 \log \left ( 1 + \frac{\sqrt{d} (288 \betatil^2 + 8H/\sigmin^2)}{\lambda \epsilon} \right ) + \dR \log \left (1 + \frac{4 \RR}{\epsilon} \right ) \\
& \le (4d^2 + \dR) \log \left ( 1 + \frac{\sqrt{d} (288 \betatil^2 + 8H/\sigmin^2)/\lambda + 4 \RR}{\epsilon} \right ) .
\end{align*}

\end{proof}

\begin{lem}\label{lem:good_event}
Assume we are in the linear MDP setting and are running \Cref{alg:catoni_regret} with $\lambda \le 1/H^2$. Then as long as $K \ge \Kinit$, we will have that $\Pr[\cE] \ge 1-\delta$. 
\end{lem}
\begin{proof}
First, note that
\begin{align*}
\cE^c = \bigcup_{k = \Kinit}^K \bigcup_{h=1}^H (B_{k,h}^c \cup A_{k,h}^c) .
\end{align*}
Then,

\begin{claim}\label{claim:set_inclusion}
\begin{align*}
\bigcup_{k = \Kinit}^K \bigcup_{h=1}^H (B_{k,h}^c \cup A_{k,h}^c)  & = \bigcup_{k=\Kinit}^K \bigcup_{h=1}^H \left [ B_{k,h}^c \cap \left ( \cap_{k' = \Kinit}^{k-1} \cap_{h'=1}^H (B_{k',h'} \cap A_{k',h'}) \right ) \cap \left ( \cap_{h' = h+1}^H B_{k,h'} \right )  \right ] \\
& \qquad  \cup  \bigcup_{k=\Kinit}^K \bigcup_{h=1}^H \left [ A_{k,h}^c \cap \left ( \cap_{k' = \Kinit}^{k-1} \cap_{h'=1}^H (B_{k',h'} \cap A_{k',h'}) \right ) \cap \left ( \cap_{h' = h+1}^H B_{k,h'} \right )  \right ] .
\end{align*}
\end{claim}

\Cref{claim:set_inclusion} and a union bound imply that
\begin{align*}
\Pr[\cE^c] & = \Pr \left [ \bigcup_{k = \Kinit}^K \bigcup_{h=1}^H (B_{k,h}^c \cup A_{k,h}^c)  \right ] \\
& \le \sum_{k=\Kinit}^K \sum_{h=1}^H \Pr \left [ B_{k,h}^c \cap \left ( \cap_{k' = \Kinit}^{k-1} \cap_{h'=1}^H (B_{k',h'} \cap A_{k',h'}) \right ) \cap \left ( \cap_{h' = h+1}^H B_{k,h'} \right )  \right ]  \\
& \qquad + \sum_{k=\Kinit}^K \sum_{h=1}^H \Pr \left [ A_{k,h}^c \cap \left ( \cap_{k' = \Kinit}^{k-1} \cap_{h'=1}^H (B_{k',h'} \cap A_{k',h'}) \right ) \cap \left ( \cap_{h' = h+1}^H B_{k,h'} \right )  \right ] \\
& \le \sum_{k=\Kinit}^K \sum_{h=1}^H \Pr \left [ B_{k,h}^c | \left ( \cap_{k' = \Kinit}^{k-1} \cap_{h'=1}^H (B_{k',h'} \cap A_{k',h'}) \right ) \cap \left ( \cap_{h' = h+1}^H B_{k,h'} \right )  \right ]  \\
& \qquad + \sum_{k=\Kinit}^K \sum_{h=1}^H \Pr \left [ A_{k,h}^c | \cap_{k' = \Kinit}^{k-1} \cap_{h'=1}^H (B_{k',h'} \cap A_{k',h'})  \cap \left ( \cap_{h' = h+1}^H B_{k,h'} \right ) \right ] .
\end{align*}

We first bound 
$$\Pr \left [ A_{k,h}^c | \cap_{k' = \Kinit}^{k-1} \cap_{h'=1}^H (B_{k',h'} \cap A_{k',h'}) \cap \left ( \cap_{h' = h+1}^H B_{k,h'} \right )   \right ].$$
By \Cref{lem:linear_mdp_reg}, we have the regression estimate $\catonii_{h,k} \left [  (k-1) \bLambda_{h,k-1}^{-1} \bphi_{h,k} \right ]$ satisfies \Cref{defn:regression_with_function_approx} conditioned on the event $\cap_{k' = \Kinit}^{k-1} \cap_{h'=1}^H (B_{k',h'} \cap A_{k',h'}) \cap \left ( \cap_{h' = h+1}^H B_{k,h'} \right )$. We now apply \Cref{cor:catoni}. First note that since $\alphamax = K/\sigmin$, $\| \btheta_h \|_2 \le \sqrt{d}$, and using the values for $\beta_u$ and $\beta_\mu$ from \Cref{lem:linear_mdp_reg}, as well as the covering number bound of $\Fclassmdp$, it follows that $\Cmdp$ upper bounds $d_T$\footnote{Note that if $k \ge \Kinit = 4 ( \Cmdp + \log (2HK/\delta))$, then $K \ge \betatil$, so we can remove $\betatil$ from the definition of $d_T$ as it will be dominated by $K$.}. Since we have assumed $K \ge \Kinit$, by our choice of $\Kinit = c(\log(2HK/\delta) + \Cmdp)$, it follows that the minimum sample condition of \Cref{cor:catoni} is met for $k \ge \Kinit$. Finally, note that in this setting, using the definition of linear MDPs, \Cref{defn:linear_mdp}, we will have that
\begin{align*}
\bthetast = \int  V_{h+1}^k(s') \rmd \bmu_h(s') .
\end{align*}
Thus, by \Cref{cor:catoni}, with probability at least $1-\delta/(2HK)$,
\begin{align*}
 \bigg | \catonii_{h,k} & \left [ (k-1)  \bLambda_{h,k-1}^{-1} \bphi_{h,k} \right ]  - \innerb{\bphi_{h,k}}{\int  V_{h+1}^{k}(s') d\bmu_h(s')}  \bigg | \\
& \le 5 \| \bphi_{h,k} \|_{\bLambda_{h,k-1}^{-1}} \left ( \sqrt{\Cmdp + \log(2HK/\delta)} +  \sqrt{\lambda} \| \bthetast \|_2 \right ) + \frac{3(\Cmdp + \log(2HK/\delta))}{\alphamax (k-1)}
\end{align*}
Note that,
\begin{align*}
\| \bthetast \|_2 = \| \int  V_{h+1}^k(s') \rmd \bmu_h(s')  \|_2 \le H \| \int   |\rmd \bmu_h(s')| \|_2  \le H \sqrt{d}
\end{align*}
where the last inequality follows by \Cref{defn:linear_mdp}. 
It follows that for $\lambda \le 1/H^2$ and proper choice of the universal constant in $\Cmdp$, we can bound 
\begin{align*}
\sqrt{\lambda} \| \bthetast \|_2 \le \frac{1}{5} \sqrt{\Cmdp + \log(2HK/\delta)} .
\end{align*}
As $\innerb{\bphi_{h,k}}{\int \bmu_h(s') V_{h+1}^{k}(s') ds'} = \Exp_h[V_{h+1}^{k}](s_{h,k},a_{h,k})$ by \Cref{defn:linear_mdp}, by our choice of $\beta$ we conclude that with probability at least $1-\delta/(2HK)$\footnote{We have replaced $1/(k-1)$ in the lower order term with $1/k$ for future notational convenience. Note that this is valid since $k \ge \Kinit > 1$ so we can accommodate this change by slightly increasing the constant in $\beta$.},
\begin{align*}
& \bigg |  \catonii_{h,k}  \left [  (k-1) \bLambda_{h,k-1}^{-1} \bphi_{h,k} \right ]  - \Exp_h[V_{h+1}^{k}](s_{h,k},a_{h,k}) \bigg |   \le  \beta \| \bphi_{h,k} \|_{\bLambda_{h,k-1}^{-1}} + \sigmin \beta^2/k^2 .
\end{align*}
This is precisely the definition of $A_{k,h}$, however, so it follows that
\begin{align*}
\Pr \left [ A_{k,h}^c | \cap_{k' = \Kinit}^{k-1} \cap_{h'=1}^H (B_{k',h'} \cap A_{k',h'}) \cap \left ( \cap_{h' = h+1}^H B_{k,h'} \right )  \right ] \le \frac{\delta}{2HK} .
\end{align*}

The bound on 
$$ \Pr \left [ B_{k,h}^c | \left ( \cap_{k' = \Kinit}^{k-1} \cap_{h'=1}^H (B_{k',h'} \cap A_{k',h'}) \right ) \cap \left ( \cap_{h' = h+1}^H B_{k,h'} \right )  \right ]  $$
can be shown almost identically. As such, we omit the calculation and conclude that
\begin{align*}
\Pr \left [ B_{k,h}^c | \left ( \cap_{k' = \Kinit}^{k-1} \cap_{h'=1}^H (B_{k',h'} \cap A_{k',h'}) \right ) \cap \left ( \cap_{h' = h+1}^H B_{k,h'} \right )  \right ] \le \frac{\delta}{2HK} . 
\end{align*}
Combining these bounds gives that $\Pr[\cE^c] \le \delta$.
\end{proof}

\begin{lem}\label{lem:Qtil_decreases}
Fix $h$, $k \ge \Kinit$, and $k' > k$. Then if $B_{k,h'}$ and $B_{k',h'}$ hold for all $h' \in [h,H]$, we will have
\begin{align*}
5 Q_{h'}^k(s,a) \ge Q_{h'}^{k'}(s,a)
\end{align*}
for all $h' \in [h,H]$. In particular, $5 V_{h'}^k(s) \ge V_{h'}^{k'}(s)$.
\end{lem}
\begin{proof}
We will prove this by induction. In the base case, take $h' = H$. On $B_{k,H} \cap B_{k',H}$, we have
\begin{align*}
5 Q_{H}^k(s,a) & = \min \Big \{ 5 r_H^k(s,a) + 5 \innerb{\bphi(s,a)}{\what_H^k} + 15 \betatil \| \bphi(s,a) \|_{\bLambda_{H,k-1}^{-1}} + 15 \sigmin \betatil^2/k^2, 5 H \Big \} \\
& \overset{(a)}{\ge} \min \Big \{ 5 r_h^k(s,a) + \Exp_{H}[5 V_{H+1}^k](s,a) + 5 \betatil \| \bphi(s,a) \|_{\bLambda_{H,k-1}^{-1}} + 5 \sigmin \betatil^2/k^2, 5 H \Big \} \\
& \overset{(b)}{=} \min \Big \{ 5 r_h^k(s,a) + 5 \betatil \| \bphi(s,a) \|_{\bLambda_{H,k-1}^{-1}} + 5 \sigmin \betatil^2/k^2, 5 H \Big \} \\
& \overset{(c)}{\ge} \min \Big \{ r_h^{k'}(s,a) + 5 \betatil \| \bphi(s,a) \|_{\bLambda_{H,k'-1}^{-1}} + 5 \sigmin \betatil^2/(k')^2,  H \Big \} \\
& \overset{(d)}{\ge} \min \Big \{ r_h^{k'}(s,a) +  \innerb{\bphi(s,a)}{\what_{H}^{k'}} +3\betatil \| \bphi(s,a) \|_{\bLambda_{H,k'-1}^{-1}}  + 3 \sigmin \betatil^2/(k')^2,  H \Big \} \\
& = Q_H^{k'}(s,a)
\end{align*}
where $(a)$ follows since we are on $B_{k,H}$ and by \Cref{lem:lin_approx_good_mdp}, $(b)$ follows since $V_{H+1}^k(s') = 0$ by definition, $(c)$ follows since reward is non-increasing in $k$ so $r_h^{k'}(s,a) \le r_h^k(s,a)$, and since $\bLambda_{H,k'-1} \succeq \bLambda_{H,k}$, and $(d)$ follows since we are on $B_{k',H}$, and by \Cref{lem:lin_approx_good_mdp}. This implies that, for all $s$, 
\begin{align}\label{eq:var_dom_V}
5 V_H^k(s) = 5 Q_H^k(s,\pi_H^k(s)) \ge 5 Q_H^k(s,\pi_H^{k'}(s)) \ge Q_H^{k'}(s,\pi_H^{k'}(s)) = V_H^{k'}(s) .
\end{align}

For the inductive step, assume that $5 V_{h' + 1}^k(s) \ge V_{h'+1}^{k'}(s)$ for all $s$ and that $B_{k,h'} \cap B_{k',h'}$ holds. Then we can repeat the above calculation, but now lower bounding
\begin{align*}
 \Exp_{h'}[5 V_{h'+1}^k](s,a) \ge \Exp_{h'}[V_{h'+1}^{k'}](s,a).
\end{align*}
In full detail,
\begin{align*}
5 Q_{h'}^k(s,a) & = \min \Big \{ 5 r_{h'}^k(s,a) + 5\innerb{\bphi(s,a)}{\what_{h'}^k} + 15\betatil \| \bphi(s,a) \|_{\bLambda_{h',k-1}^{-1}} + 15 \sigmin \betatil^2/k^2, \cvar H \Big \} \\
& \ge \min \Big \{ 5 r_{h'}^{k}(s,a) + \Exp_{h'}[5V_{h'+1}^k](s,a) + 5 \betatil \| \bphi(s,a) \|_{\bLambda_{h',k-1}^{-1}}  + 5 \sigmin \betatil^2/k^2, 5 H \Big \} \\
& \ge \min \Big \{ 5 r_{h'}^{k}(s,a) + \Exp_{h'}[V_{h'+1}^{k'}](s,a)  + 5 \betatil \| \bphi(s,a) \|_{\bLambda_{h',k-1}^{-1}}  + 5 \sigmin \betatil^2/k^2, 5 H \Big \} \\
& \ge \min \Big \{ r_{h'}^{k'}(s,a) + \innerb{\bphi(s,a)}{\what_{h'}^{k'}} + 3 \betatil \| \bphi(s,a) \|_{\bLambda_{h',k'-1}^{-1}}  + 3 \sigmin \betatil^2/(k')^2,  H \Big \} \\
& = Q_{h'}^{k'}(s,a) .
\end{align*}
It follows that $5 V_{h'}^k(s) \ge V_{h'}^{k'}(s)$ by the same argument as in \eqref{eq:var_dom_V}. This proves the inductive step, so the result follows. 
\end{proof}

\begin{lem}\label{lem:sigma_valid2}
Set
\begin{align*}
\sighat_{h,k}^2 = \max \bigg \{  & 20 H  \catonii_{h,k} \left [(k-1) \bLambda_{h,k-1}^{-1} \bphi_{h,k} \right ] + 20 H  \beta \| \bphi_{h,k} \|_{\bLambda_{h,k-1}^{-1}}  + 20H  \sigmin \beta^2/k^2  , \sigmin^2 \bigg \} .
\end{align*}
Then $\sighat_{h,k}^2$ is $\cF_{h,k}$-measurable, and, for any $k' \ge k$, on the event $\cap_{h' = h+1}^H (B_{k,h'} \cap B_{k',h'}) \cap A_{k,h}$, we have 
\begin{align}\label{eq:sighat_lb}
\begin{split}
& \Exp_h[(V_{h+1}^{k'})^2](s_{h,k}, a_{h,k})  \le \frac{1}{2} \sighat_{h,k}^2, \\
&  4H \Exp_h[V_{h+1}^{k'}](s_{h,k},a_{h,k}) \le \sighat_{h,k}^2
\end{split}
\end{align}
and
\begin{align}\label{eq:sighat_ub}
\begin{split}
\sighat_{h,k}^2 \le \max \bigg \{  & 20H \Exp_h[V_{h+1}^{k}](s_{h,k},a_{h,k})  + 40 H \beta \| \bphi_{h,k} \|_{\bLambda_{h,k-1}^{-1}}  + 40H \sigmin \beta^2/k^2 , \sigmin^2 \bigg \} .
\end{split}
\end{align}
\end{lem}
\begin{proof}
By definition $\bphi_{h,k}, \bLambda_{h,k-1}$, and $r_{h,k}$ are $\cF_{h,k}$-measurable. As we only rely on data up to episode $k-1$, it follows that $\bphi_{h,\tau}$ and $s_{h+1,\tau}$ are also $\cF_{h,k}$-measurable. Finally, we see from the definition of \Cref{alg:catoni_regret} that $V_{h+1}^k$ is formed using only data up to and including episode $k-1$. It follows that $\sighat_{h,k}$ is $\cF_{h,k}$-measurable. 

Note that we can trivially bound
\begin{align*}
 \Exp_h[( V_{h+1}^{k'})^2](s_{h,k}, a_{h,k}) &  \le  H \Exp_h[V_{h+1}^{k'}](s_{h,k}, a_{h,k})
\end{align*}
where the last inequality follows since $V_{h+1}^{k'}(s') \in [0,H]$. By \Cref{lem:Qtil_decreases}, on the event $\cap_{h' = h+1}^H (B_{k,h'} \cap B_{k',h'})$, we will have that
\begin{align*}
H \Exp_h[V_{h+1}^{k'}](s_{h,k}, a_{h,k}) \le 5 H \Exp_h[V_{h+1}^{k}](s_{h,k}, a_{h,k}) .
\end{align*}
On the event $A_{k,h}$, we can bound
\begin{align*}
\Exp_h[V_{h+1}^{k}](s_{h,k}, a_{h,k}) & \le \catonii_{h,k} \left [ (k-1)  \bLambda_{h,k-1}^{-1} \bphi_{h,k} \right ]  +  \beta \| \bphi_{h,k} \|_{\bLambda_{h,k-1}^{-1}} + \sigmin \beta^2/k^2 .
\end{align*}
The lower bound \eqref{eq:sighat_lb} follows by our choice of $\sighat_{h,k}^2$. The upper bound \eqref{eq:sighat_ub} follows since, on $A_{k,h}$, we have
\begin{align*}
\catonii_{h,k} \left [ (k-1)  \bLambda_{h,k-1}^{-1} \bphi_{h,k} \right ] & \le \Exp_h[V_{h+1}^{k}](s_{h,k}, a_{h,k})    +  \beta \| \bphi_{h,k} \|_{\bLambda_{h,k-1}^{-1}} + \sigmin \beta^2/k^2 .
\end{align*}
\end{proof}

\begin{lem}\label{lem:lin_approx_good_mdp}
On the event $B_{k,h}$, if we are running \Cref{alg:catoni_regret}, we will have that
\begin{align*}
& | \innerb{\bphi(s,a)}{\what_h^k} - \Exphat_h[V_{h+1}^k](s,a) | \le \betatil \| \bphi(s,a) \|_{\bLambda_{h,k-1}^{-1}} +  \sigmin \betatil^2 / k^2 , \\
& | \innerb{\bphi(s,a)}{\what_h^k}  - \Exp_h[V_{h+1}^{k}](s,a) | \le \betatil \| \bphi(s,a) \|_{\bLambda_{h,k-1}^{-1}} +  \sigmin \betatil^2 / k^2 
\end{align*}
for all $s$ and $a$. 
\end{lem}
\begin{proof}
This follows directly from \Cref{lem:cat_lin_approx} and \Cref{lem:cat_lin_approx_efficient}, the definition of $\betatil$ and $B_{k,h}$, and since $\Exp_h[V_{h+1}^{k}](s,a)$ is linear in $\bphi(s,a)$ and we assume that $\| \bphi(s,a) \|_2 \le 1$ for all $s,a$ and that there does not exist $s,a$ such that $\bphi(s,a) = \bm{0}$. 
\end{proof}

\begin{proof}[Proof of \Cref{claim:set_inclusion}]
Clearly,
\begin{align*}
& \bigcup_{k = \Kinit}^K \bigcup_{h=1}^H (B_{k,h}^c \cup A_{k,h}^c) \\
& \qquad = \bigcup_{k = \Kinit}^K \bigcup_{h=1}^H \Big [ (B_{k,h}^c \cup A_{k,h}^c) \backslash \left ( (\cup_{k' = \Kinit}^{k-1} \cup_{h' =1}^H (B_{k',h'}^c \cup A_{k',h'}^c)) \cup (\cup_{h'=h+1}^H (B_{k,h'}^c \cup A_{k,h'}^c)) \right ) \Big ] \\
&  \qquad = \bigcup_{k = \Kinit}^K \bigcup_{h=1}^H \Big [ B_{k,h}^c  \backslash \left ( (\cup_{k' = \Kinit}^{k-1} \cup_{h' =1}^H (B_{k',h'}^c \cup A_{k',h'}^c)) \cup (\cup_{h'=h+1}^H (B_{k,h'}^c \cup A_{k,h'}^c)) \right ) \Big ] \\
& \qquad \qquad \cup  \bigcup_{k = \Kinit}^K \bigcup_{h=1}^H \Big [  A_{k,h}^c \backslash \left ( (\cup_{k' = \Kinit}^{k-1} \cup_{h' =1}^H (B_{k',h'}^c \cup A_{k',h'}^c)) \cup (\cup_{h'=h+1}^H (B_{k,h'}^c \cup A_{k,h'}^c)) \right ) \Big ] .
\end{align*}
Noting that
\begin{align*}
& X \backslash \left ( (\cup_{k' = \Kinit}^{k-1} \cup_{h' =1}^H (B_{k',h'}^c \cup A_{k',h'}^c)) \cup (\cup_{h'=h+1}^H (B_{k,h'}^c \cup A_{k,h'}^c)) \right ) \\
& = X \cap \left ( (\cap_{k' = \Kinit}^{k-1} \cap_{h' =1}^H (B_{k',h'} \cap A_{k',h'})) \cap (\cap_{h'=h+1}^H (B_{k,h'} \cap A_{k,h'})) \right )
\end{align*}
for any $X$ completes the proof.
\end{proof}

\subsection{Optimism}

\begin{lem}\label{lem:unroll_q}
On the event $\cE$, for all $s,a,h$, and $k \ge \Kinit$ and any $\pi$, we have
\begin{align*}
\Exphat_h[V_{h+1}^k](s,a) + r_h^k(s,a) - Q^{k,\pi}_h(s,a) = \Exp_h[ V_{h+1}^k - V^{k,\pi}_{h+1}](s,a)  + \xi_{h}^k(s,a)
\end{align*}
where $\xi_h^k(s,a)$ satisfies $|\xi_h^k(s,a)| \le  \beta \| \bphi(s,a) \|_{\bLambda_{h,k-1}^{-1}} + \sigmin \beta^2/k^2$.
\end{lem}
\begin{proof}
By definition, we have that
\begin{align*}
Q_h^{k,\pi}(s,a) = r_h^k(s,a) + \Exp_h[V^{k,\pi}_{h+1}](s,a) .
\end{align*}
On $\cE$, we have that 
\begin{align*}
\left | \Exphat_h[V_{h+1}^k](s,a)  - \Exp_h[V_{h+1}^k](s,a) \right | \le \beta \| \bphi(s,a) \|_{\bLambda_{h,k-1}^{-1}} + \beta^2/k 
\end{align*}
so we can therefore write
\begin{align*}
\Exphat_h[V_{h+1}^k](s,a) =  \Exp_h[V_{h+1}^k](s,a) + \xi_h^k(s,a)
\end{align*}
for a term $\xi_h^k(s,a)$ satisfying
\begin{align*}
| \xi_h^k(s,a) | \le \beta \| \bphi(s,a) \|_{\bLambda_{h,k-1}^{-1}} + \sigmin \beta^2/k^2 . 
\end{align*}
It follows that
\begin{align*}
\Exphat_h[V_{h+1}^k](s,a) + r_h^k(s,a) - \Qpi_h(s,a) = \Exp_h[ V_{h+1}^k - V^{k,\pi}_{h+1}](s,a) + \xi_h^k(s,a) .
\end{align*}
\end{proof}

\begin{lem}\label{lem:optimism}
On the event $\cE$, for all $s,a,h$, and $k \ge \Kinit$, we have that $Q_h^k(s,a) \ge Q^{k,\star}_h(s,a)$.
\end{lem}
\begin{proof}
We will prove this by induction for a fixed $k$. First, take $h = H$. Since $V_{H+1}(s) = \Vst_{H+1}(s) = 0$ by definition, by \Cref{lem:unroll_q} we have
\begin{align*}
| \Exphat_H[V_{H+1}^k](s,a) + r_H^k(s,a) - Q^{k,\star}_H(s,a) | \le  \beta \| \bphi(s,a) \|_{\bLambda_{H,k-1}^{-1}} + \sigmin \beta^2/k^2
\end{align*}
which implies
\begin{align*}
Q^{k,\star}_H(s,a) & \le \min \{ r_h^k(s,a) + \Exphat_H[V_{H+1}^k](s,a) + \beta \| \bphi(s,a) \|_{\bLambda_{H,k-1}^{-1}} + \sigmin \beta^2/k^2, H \} \\
& \le \min \{ r_h^k(s,a) + \innerb{\bphi(s,a)}{\what_H^k} +  (\beta + \betatil) \| \bphi(s,a) \|_{\bLambda_{H,k-1}^{-1}} + \sigmin (\beta^2 + \betatil^2)/k^2, H \} \\
& \le \min \{ r_h^k(s,a) + \innerb{\bphi(s,a)}{\what_H^k} +  3 \betatil \| \bphi(s,a) \|_{\bLambda_{H,k-1}^{-1}} + 3 \sigmin \betatil^2/k^2, H \} \\
& = Q_H^k(s,a) 
\end{align*}
where we have used \Cref{lem:lin_approx_good_mdp} and that $\beta \le \betatil$. Now assume that $Q_{h+1}^k(s,a) \ge Q^{k,\star}_{h+1}(s,a)$ for all $(s,a)$ and some $h$. Again by \Cref{lem:unroll_q}, we have that 
\begin{align*}
| \Exphat_h[V_{h+1}^k](s,a) - Q^{k,\star}_h(s,a) - \Exp_h[ V_{h+1}^k - V^{k,\star}_{h+1}](s,a) | \le \beta \| \bphi(s,a) \|_{\bLambda_{h,k-1}^{-1}} + \sigmin \beta^2/k^2 .
\end{align*}
By the inductive hypothesis $\Exp_h[ V_{h+1}^k - V^{k,\star}_{h+1}](s,a) \ge 0$, so 
\begin{align*}
Q^{k,\star}_h(s,a) & \le \min \{ r_h^k(s,a) + \Exphat_h[V_{h+1}^k](s,a) + \beta \| \bphi(s,a) \|_{\bLambda_{h,k-1}^{-1}} + \sigmin \beta^2/k^2, H \} \\
& \le \min \{ r_h^k(s,a) + \innerb{\bphi(s,a)}{\what_h^k} +  (\beta + \betatil) \| \bphi(s,a) \|_{\bLambda_{h,k-1}^{-1}} +  \sigmin (\beta^2+\betatil^2)/k^2, H \} \\
& = \min \{ r_h^k(s,a) + \innerb{\bphi(s,a)}{\what_h^k} +  3 \betatil \| \bphi(s,a) \|_{\bLambda_{h,k-1}^{-1}} +  3 \sigmin \betatil^2/k^2, H \} \\
& = Q_h^k(s,a) .
\end{align*}
This proves the inductive hypothesis so the result follows. 
\end{proof}

\begin{lem}[Formal version of \Cref{lem:V_recursion_informal}] \label{lem:V_recursion}
Let $\delta_h^k = V_h^{k}(s_h^k) - V_h^{k,\pi_k}(s_h^k)$ and $\zeta_{h+1}^k = \Exp_{h}[\delta_{h+1}^k ](s_{h,k},a_{h,k}) - \delta_{h+1}^k$. Then, on the event $\cE$, for any $k \ge \Kinit$,
\begin{align*}
\delta_h^k \le \delta_{h+1}^k + \zeta_{h+1}^k + \min \{ 5\betatil \| \bphi_{h,k} \|_{\bLambda_{h,k-1}^{-1}} + 5 \sigmin \betatil^2/k^2, H \}. 
\end{align*}
\end{lem}
\begin{proof}
We have
\begin{align*}
Q_h^k(s,a) - Q^{k,\pi_k}_h(s,a) & \overset{(a)}{=} \min \{ r_h^k(s,a) + \innerb{\bphi(s,a)}{\what_h^k} + 3 \beta \| \bphi(s,a) \|_{\bLambda_{h,k-1}^{-1}} + 3 \sigmin \beta^2/k^2, H \} - Q^{k,\pi_k}_h(s,a) \\
& \overset{(b)}{\le}  \min \{ r_h^k(s,a) +  \innerb{\bphi(s,a)}{\what_h^k}- Q^{k,\pi_k}_h(s,a) + 3 \betatil \| \bphi(s,a) \|_{\bLambda_{h,k-1}^{-1}} + 3 \sigmin \betatil^2/k^2, H \} \\
& \overset{(c)}{\le}  \min \{ r_h^k(s,a) + \Exphat_h[V_{h+1}^k](s,a) - Q^{k,\pi_k}_h(s,a) + 4 \betatil \| \bphi(s,a) \|_{\bLambda_{h,k-1}^{-1}} + 4 \sigmin \betatil^2/k^2, H \} \\
& \overset{(d)}{\le} \min \{ \Exp_h[V_{h+1}^k - V_{h+1}^{k,\pi_k}](s,a) + 5 \betatil \| \bphi(s,a) \|_{\bLambda_{h,k-1}^{-1}}  + 5 \sigmin \betatil^2/k^2, H \} \\
& \overset{(e)}{\le} \Exp_h[V_{h+1}^k - V_{h+1}^{k,\pi_k}](s,a) +  \min \{ 5 \betatil \| \bphi(s,a) \|_{\bLambda_{h,k-1}^{-1}} + 5 \sigmin \betatil^2/k^2, H \}
\end{align*}
where $(a)$ is by definition of $Q_h^k(s,a)$, $(b)$ holds since $Q_h^{k,\pi_k}(s,a) \ge 0$, $(c)$ holds by \Cref{lem:lin_approx_good_mdp}, $(d)$ follows by \Cref{lem:unroll_q}, and $(e)$ follows since $\Exp_h[V_{h+1}^k - V_{h+1}^{k,\pi_k}](s,a) \ge 0$ by \Cref{lem:optimism}. 

Now note that since at episode $k$ we play action $a_h^k = \argmax_a Q_h^k(s_h^k,a)$, we will have that 
\begin{align*}
\delta_h^k = Q_h^k(s_h^k,a_h^k) - Q_h^{k,\pi_k}(s_h^k,a_h^k). 
\end{align*}
The result follows by the definition of $V_h^k(s)$ and $V_h^{k,\pi_k}(s)$. 
\end{proof}

\subsection{Regret Bound}
\begin{lem}\label{lem:mart_diff_reg_bound}
With probability at least $1-\delta$, we can bound
\begin{align*}
\sum_{k=\Kinit}^K \sum_{h=1}^H \zeta_h^k \le 2 \sqrt{ 8 H \sum_{k=\Kinit}^K \sum_{h=1}^H \Exp_{h-1}[V_h^k](s_{h-1,k},a_{h-1,k})  \cdot \log 1/\delta} + 2 H \log 1/\delta . 
\end{align*}
\end{lem}
\begin{proof}
This is a direct application of \Cref{lem:freedman}, Freedman's inequality. Recall that
\begin{align*}
\zeta_h^k = \Exp_{h-1}[\delta_h^k](s_{h-1,k},a_{h-1,k}) - \delta_h^k = \Exp_{h-1}[V_h^k - V_h^{k,\pi_k}](s_{h-1,k},a_{h-1,k})  - (V_h^k(s_h^k) - V_h^{k,\pi_k}(s_h^k)).
\end{align*}
Thus, we can bound
\begin{align*}
| \zeta_h^k | \le 2 H
\end{align*}
since the value function will always be bounded in $[0,H]$. Next, note that
\begin{align*}
\Exp_{h-1}[(\zeta_h^k)^2 ](s_{h-1,k},a_{h-1,k}) & \le 2 \Exp_{h-1}[(V_h^k - V_h^{k,\pi_k})^2](s_{h-1,k},a_{h-1,k})  \\
& \le 4 \Exp_{h-1}[(V_h^k)^2 + (V_h^{k,\pi_k})^2](s_{h-1,k},a_{h-1,k})  \\
& \le 8 \Exp_{h-1}[(V_h^k)^2](s_{h-1,k},a_{h-1,k}) \\
& \le 8 H \Exp_{h-1}[V_h^k](s_{h-1,k},a_{h-1,k})
\end{align*}
where the second to last inequality uses \Cref{lem:optimism}. Using these bounds the result then follows directly from \Cref{lem:freedman}.
\end{proof}

\begin{lem}\label{lem:vtil_to_vst}
With probability at least $1-\delta$, we have
\begin{align*}
\sum_{k=1}^K \sum_{h=1}^H \Exp_{h-1}[V_h^k](s_{h-1,k},a_{h-1,k})  \le H \cdot \left ( \sum_{k=\Kinit}^K V_1^{k,\star} +  \cRtil_K  + 2 \sqrt{ \left ( \sum_{k=\Kinit}^K V_1^{k,\star} +  \cRtil_K \right ) \cdot \log 1/\delta} +  \log 1/\delta \right ) 
\end{align*}
where $\cRtil_K = \sum_{k=1}^K (V_1^k(s_1) - V_1^{k,\pi_k}(s_1))$.
\end{lem}
\begin{proof}
By \Cref{lem:lin_approx_good_mdp}, on $\cE$,
\begin{align*}
| \innerb{\bphi(s,a)}{\what_h^k}  - \Exp_h[V_{h+1}^{k}](s,a) | \le \betatil \| \bphi(s,a) \|_{\bLambda_{h,k-1}^{-1}} +  \sigmin \betatil^2 / k^2 
\end{align*}
which implies that
\begin{align*}
 \Exp_h[V_{h+1}^{k}](s,a) \le \innerb{\bphi(s,a)}{\what_h^k} + \betatil \| \bphi(s,a) \|_{\bLambda_{h,k-1}^{-1}} +  \sigmin \betatil^2 / k^2 .
\end{align*}
Thus,
\begin{align*}
Q_h^k(s,a) & = \min \{ r_h^k(s,a) + \innerb{\bphi(s,a)}{\what_h^k} + 3 \betatil \| \bphi(s,a) \|_{\bLambda_{h,k-1}^{-1}} + 3 \sigmin \betatil^2 / k^2, H \} \\
& \ge \min \{ r_h^k(s,a) + \Exp_h[V_{h+1}^{k}](s,a), H \} .
\end{align*}
Since $\pi_h^k(s) = \argmax_a Q_h^k(s,a)$, we have that $V_{h+1}^k(s') = Q_{h+1}^k(s', \pi_{h+1}^k(s'))$. Using that reward is always nonnegative, we can therefore unroll $V_1^k(s_1)$ backwards as:
\begin{align*}
V_1^k(s_1) & = Q_1^k(s_1,\pi_1^k(s_1)) \\
& \ge \min \{ r_1^k(s_1,\pi_1^k(s_1)) + \Exp_{s_2}[Q_2^k(s_2,\pi_2^k(s_2)) \mid s_1, \pi_1^k(s_1)], H \} \\
& \ge \min \{  \Exp_{s_2}[Q_2^k(s_2,\pi_2^k(s_2)) \mid s_1, \pi_1^k(s_1)], H \} \\
& = \Exp_{s_2}[Q_2^k(s_2,\pi_2^k(s_2)) \mid s_1, \pi_1^k(s_1)] \\
& \ge  \Exp_{s_2}[\min \{ r_2^k(s_2,\pi_2^k(s_2)) + \Exp_{s_3}[Q_3^k(s_3,\pi_3^k(s_3)) \mid s_2,\pi_2^k(s_2)], H \} \mid s_1, \pi_1^k(s_1)] \\
& \ge  \Exp_{s_2}[ \Exp_{s_3}[Q_3^k(s_3,\pi_3^k(s_3)) \mid s_2,\pi_2^k(s_2)] \mid s_1, \pi_1^k(s_1)] \\
& = \Exp_{\pi^k}[Q_3^k(s_3,\pi_3^k(s_3))] \\
& \vdots \\
& \ge  \Exp_{\pi^k}[Q_h^k(s_h, \pi_h^k(s_h))] \\
& = \Exp_{\pi^k}[V_h^k(s_h)] 
\end{align*}
where here $\Exp_{s'}[\cdot \mid s,a]$ denotes taking the expectation over the next state $s'$ given that we are in $(s,a)$, and $\Exp_{\pi^k}[\cdot]$ denotes the expectation over trajectories generated by $\pi_k$. We conclude
\begin{align*}
V_1^k(s_1) \ge \Exp_{\pi^k}[V_h^k(s_h)]
\end{align*}
for any $h$. Given this, since we play policy $\pi_k$ at episode $k$, we will have that
\begin{align*}
\Exp_{\pi_k} \Exp_{h-1}[V_h^k](s_{h-1,k},a_{h-1,k}) & = \Exp_{\pi_k}[V_h^k(s_h^k)  ] \le \Exp_{\pi_k}[V_1^k(s_1^k)] = V_1^k(s_1)
\end{align*}
which allows us to bound
\begin{align*}
 & \sum_{k=\Kinit}^K  \sum_{h=1}^H\Exp_{h-1}[V_h^k](s_{h-1,k},a_{h-1,k})\\
& = \sum_{k=\Kinit}^K \sum_{h=1}^H \Exp_{\pi_k} \Exp_{h-1}[V_h^k](s_{h-1,k},a_{h-1,k})  + \sum_{k=\Kinit}^K \sum_{h=1}^H (\Exp_{h-1}[V_h^k](s_{h-1,k},a_{h-1,k})- \Exp_{\pi_k} \Exp_{h-1}[V_h^k](s_{h-1,k},a_{h-1,k})) \\
& \le H \sum_{k=\Kinit}^K V_1^k(s_1) + \sum_{k=\Kinit}^K \sum_{h=1}^H (\Exp_{h-1}[V_h^k](s_{h-1,k},a_{h-1,k}) - \Exp_{\pi_k}\Exp_{h-1}[V_h^k](s_{h-1,k},a_{h-1,k})).
\end{align*}
By definition of $\cRtil_K$,
\begin{align*}
 H \sum_{k=\Kinit}^K V_1^k(s_1) =  H \sum_{k=\Kinit}^K V_1^{k,\pi_k}(s_1) + H \cRtil_K \le H \sum_{k=\Kinit}^K V_1^{k,\star}(s_1) + H \cRtil_K . 
\end{align*}
It remains to bound
\begin{align*}
\sum_{k=\Kinit}^K \sum_{h=1}^H (\Exp_{h-1}[V_h^k](s_{h-1,k},a_{h-1,k})- \Exp_{\pi_k} \Exp_{h-1}[V_h^k](s_{h-1,k},a_{h-1,k}) ) .
\end{align*}
Note first that $|\Exp_{h-1}[V_h^k](s_{h-1,k},a_{h-1,k}) - \Exp_{\pi_k}\Exp_{h-1}[V_h^k](s_{h-1,k},a_{h-1,k})| \le H$ almost surely, 
$$\Exp_{\pi_k}[\Exp_{h-1}[V_h^k](s_{h-1,k},a_{h-1,k}) - \Exp_{\pi_k} \Exp_{h-1}[V_h^k](s_{h-1,k},a_{h-1,k})] = 0,$$
and 
\begin{align*}
\Exp_{\pi_k}[(\Exp_{h-1}[V_h^k](s_{h-1,k},a_{h-1,k}) - \Exp_{\pi_k}\Exp_{h-1}[V_h^k](s_{h-1,k},a_{h-1,k}) )^2] & \le  \Exp_{\pi_k}[\Exp_{h-1}[V_h^k](s_{h-1,k},a_{h-1,k})^2] \\
& \le H \Exp_{\pi_k}[\Exp[\Exp_{h-1}[V_h^k](s_{h-1,k},a_{h-1,k})] \\
& \le H V_1^k(s_1)
\end{align*}
where the last inequality follows by what we have shown above. Applying Freedman's inequality (\Cref{lem:freedman}), we can then bound, with probability at least $1-\delta$,
\begin{align*}
 \sum_{k=\Kinit}^K & \sum_{h=1}^H (\Exp_{h-1}[V_h^k](s_{h-1,k},a_{h-1,k})- \Exp_{\pi_k} \Exp_{h-1}[V_h^k](s_{h-1,k},a_{h-1,k}) ) \\
 & \le 2 \sqrt{H^2 \sum_{k=\Kinit}^K V_1^k(s_1) \cdot \log 1/\delta} + H \log 1/\delta \\
& \le 2 \sqrt{(H^2 \sum_{k=\Kinit}^K V_1^{k,\star}(s_1) + H^2 \cRtil_K) \cdot \log 1/\delta} + H \log 1/\delta
\end{align*}
where the last inequality follows by what we have shown above. 
\end{proof}

\begin{proof}[Proof of \Cref{thm:reg_bound_time_var}]
By definition of $\cR_K$ and \Cref{lem:optimism},
\begin{align*}
\cR_K & = \sum_{k=1}^K (V^{k,\star}_1(s_1) - V_1^{k,\pi_k}(s_1)) \le H \Kinit +  \sum_{k=\Kinit}^K (V_1^{k}(s_1) - V_1^{k,\pi_k}(s_1)) =: H \Kinit + \cRtil_K .
\end{align*}

\paragraph{Decomposing the regret.} 
By \Cref{lem:V_recursion},
\begin{align*}
\cRtil_K & \le \sum_{k=\Kinit}^K \sum_{h=1}^H \zeta_h^k + \sum_{k=\Kinit}^K \sum_{h=1}^H \min \{ 5 \betatil \| \bphi(s,a) \|_{\bLambda_{h,k-1}^{-1}} + 5 \sigmin \betatil^2/k^2, H \} \\
& \le \sum_{k=\Kinit}^K \sum_{h=1}^H \zeta_h^k + \sum_{k=\Kinit}^K \sum_{h=1}^H \min \{ 5 \betatil \| \bphi(s,a) \|_{\bLambda_{h,k-1}^{-1}}, H \} + \sum_{k=\Kinit }^K \sum_{h=1}^H  5 \sigmin \betatil^2/k^2.
\end{align*}
By \Cref{lem:mart_diff_reg_bound}, with probability $1-\delta$, $\sum_{k=\Kinit}^K \sum_{h=1}^H \zeta_h^k$ can be bounded as
\begin{align*}
\sum_{k=\Kinit}^K \sum_{h=1}^H \zeta_h^k \le  \sqrt{ 32 H \sum_{k=\Kinit}^K \sum_{h=1}^H \Exp_{h-1}[V_h^k](s_{h-1,k},a_{h-1,k})  \cdot \log 1/\delta} + 2 H \log 1/\delta .
\end{align*}
Furthermore, 
\begin{align*}
\sum_{k=\Kinit}^K \sum_{h=1}^H  5 \sigmin \betatil^2/k^2 \le 10 \betatil^2 H \sigmin 
\end{align*}

\paragraph{Controlling the optimistic bonuses.}
Let $\cK_{h} = \{ k \ge \Kinit \ : \ \| \bphi_{h,k} /\sighat_{h,k} \|_{\bLambda_{h,k-1}^{-1}} \le 1 \}$ and $\cK_h^c = \{ \Kinit,\ldots, K \} \backslash \cK_{h}$. Then,
\begin{align*}
\sum_{k=\Kinit}^K \sum_{h=1}^H \min \{ 5 \betatil \| \bphi_{h,k} \|_{\bLambda_{h,k-1}^{-1}}, H \} & = \sum_{k=\Kinit}^K \sum_{h=1}^H \sighat_{h,k} \min \{ 5 \betatil \| \bphi_{h,k}/ \sighat_{h,k} \|_{\bLambda_{h,k-1}^{-1}}, H/ \sighat_{h,k} \} \\
& \le 5 \betatil  \sum_{h=1}^H  \sum_{k \in \cK_{h}} \sighat_{h,k} \| \bphi_h^k/ \sighat_{h,k} \|_{\bLambda_{h,k-1}^{-1}} + \sum_{h=1}^H  H |  \cK_{h}^c | .
\end{align*}
By \Cref{lem:elip_pot_bad_event_informal}, and since $\| \bphi_{h}^k / \sighat_{h,k} \|_2 \le 1/\sigmin$ almost surely, we can bound $| \cK_{h}^c |  \le 2 d \log(1 + K / (\lambda \sigmin^2))$, which implies that
\begin{align*}
\sum_{h=1}^H  H | \cK_{h}^c | \le 2 d H^2 \log( 1 + K/(\lambda \sigmin^2))  .
\end{align*}
Denote
\begin{align*}
\eta_{h,\tau} :=  20 H \Exp_h[V_{h+1}^{\tau}](s_{h,\tau},a_{h,\tau}) .
\end{align*} 
By \Cref{lem:sigma_valid2} and the definition of $\sighat_{h,\tau}^2$, we can bound
\begin{align*}
5\betatil \sum_{h=1}^H & \sum_{k \in \cK_{h}} \sighat_{h,k} \| \bphi_{h,k} / \sighat_{h,k} \|_{\bLambda_{h,k-1}^{-1}} \\
& = 5 \betatil \sum_{h=1}^H  \sum_{k \in \cK_{h}} \frac{\sighat_{h,k}^2}{\sighat_{h,k}} \| \bphi_{h,k} / \sighat_{h,k} \|_{\bLambda_{h,k-1}^{-1}} \\
& \le 5\betatil \sum_{h=1}^H  \sum_{k \in \cK_{h}} \frac{\eta_{h,k} + 20H \beta \| \bphi_{h,k} \|_{\bLambda_{h,k-1}^{-1}} + 20H \sigmin \beta^2/k^2 + \sigmin^2}{\sighat_{h,k}} \| \bphi_{h,k} / \sighat_{h,k} \|_{\bLambda_{h,k-1}^{-1}} \\
& \le 5\betatil \sum_{h=1}^H \sum_{k \in \cK_{h}} \left (  (\sqrt{5\eta_{h,k}} + 20H \beta^2/k^2 + \sigmin) \| \bphi_{h,k}/ \sighat_{h,k} \|_{\bLambda_{h,k-1}^{-1}} + 20 H \beta \| \bphi_{h,k} / \sighat_{h,k} \|_{\bLambda_{h,k-1}^{-1}}^2 \right ) 
\end{align*}
where the final inequality follows since, by \Cref{lem:sigma_valid2}, we can lower bound $\sighat_{h,k}^2 \ge \eta_{h,k}/5$, and since we can always lower bound $\sighat_{h,k} \ge \sigmin$. 

Recalling the definition of $\cK_h$, we can bound
\begin{align*}
5 \betatil \sum_{h=1}^H  \sum_{k \in \cK_{h}} 20 H \beta \| \bphi_{h,k}/ \sighat_{h,k} \|_{\bLambda_{h,k-1}^{-1}}^2 & \le 100 H \betatil \beta \sum_{h=1}^H \sum_{k=1}^K  \min \{ \| \bphi_h^k/ \sighat_{h,k} \|_{\bLambda_{h,k}^{-1}}^2, 1 \} \\
& \le 200 H^2 \betatil \beta d \log(1 + K/(d \lambda \sigmin^2))
\end{align*}
where the last inequality follows by \Cref{lem:elip_pot}. By Cauchy-Schwarz and again using the definition of $\cK_h$, we can bound
\begin{align*}
5 \betatil&  \sum_{h=1}^H \sum_{k \in \cK_{h}}   (\sqrt{5\eta_{h,k}} + 20H \beta^2/k^2 + \sigmin) \| \bphi_h^k/ \sighat_{h,k} \|_{\bLambda_{h,k}^{-1}} \\
& \le 5 \betatil \sqrt{ 4  \sum_{h=1}^H \sum_{k \in \cK_{h}} ( 5 \eta_{h,k} + 400 H^2 \beta^4/k^4 + \sigmin^2 )}\sqrt{ \sum_{h=1}^H \sum_{k \in \cK_{h}} \| \bphi_h^k/ \sighat_{h,k} \|_{\bLambda_{h,k}^{-1}}^2 } \\
& \le 5 \betatil \sqrt{ 4  \sum_{h=1}^H \sum_{k \in \cK_h} ( 5 \eta_{h,k} + 400 H^2 \beta^4/k^4 + \sigmin^2  )}\sqrt{ \sum_{h=1}^H \sum_{k =1}^K \min \{ \| \bphi_h^k/ \sighat_{h,k} \|_{\bLambda_{h,k}^{-1}}^2, 1 \} } \\
& \le 5 \betatil \sqrt{ 2 d H \log(1 + K/(d \lambda \sigmin^2))} \sqrt{40 \sum_{h=1}^H \sum_{k=\Kinit}^K \eta_{h,k} + 3200 H^3 \beta^4 + 4 H K \sigmin^2} \\
& \le 5 \betatil \sqrt{ 2 d H \log(1 + K/(d \lambda \sigmin^2))} \left ( \sqrt{40 \sum_{h=1}^H \sum_{k=\Kinit}^K \eta_{h,k}} + 60 H^{3/2} \beta^2 + 2 \sqrt{H K \sigmin^2} \right )
\end{align*}
where we again apply \Cref{lem:elip_pot} and use that $\sqrt{a + b} \le \sqrt{a} + \sqrt{b}$ for $a,b \ge 0$.


\paragraph{Finishing the Proof.} 
By definition,
\begin{align*}
\sum_{h=1}^H \sum_{k=\Kinit}^K \eta_{h,k} & = \sum_{h=1}^H \sum_{k=\Kinit}^K  20 H \Exp_h[V_{h+1}^k](s_{h,k},a_{h,k})   \le  \sum_{h=1}^H \sum_{k=\Kinit}^K  20 H \Exp_{h-1}[V_{h}^k](s_{h-1,k},a_{h-1,k})  .
\end{align*}
Collecting terms, we have then shown that, 
\begin{align*}
\cRtil_K & \le c_1 \betatil \sqrt{d H \log(1 + K/(d \lambda \sigmin^2))} \sqrt{ H \sum_{h=1}^H \sum_{k = \Kinit}^K \Exp_{h-1}[V_{h}^k](s_{h-1,k},a_{h-1,k})} \\
& \qquad + c_2 \betatil \sqrt{d H \log(1 + K/(d \lambda \sigmin^2))} \sqrt{H \sigmin^2 K}  \\
& \qquad + c_3 \betatil \beta^2 H^2 \sqrt{d} \log(1 + K/(d \lambda \sigmin^2))
\end{align*}
for universal constants $c_1,c_2,c_3$. By \Cref{lem:vtil_to_vst} we can bound, with probability at least $1-\delta$, 
\begin{align*}
\sum_{k=\Kinit}^K \sum_{h=1}^H \Exp[V_{h-1}^{k}](s_{h-1,k},a_{h-1,k}) & \le H \cdot \left ( \sum_{k=1}^K V_1^{k,\star} +  \cRtil_K  + 2 \sqrt{\left ( \sum_{k=1}^K V_1^{k,\star} +  \cRtil_K \right ) \cdot \log 1/\delta} +  \log 1/\delta \right ) \\
& \le  4 H \log 1/\delta  \cdot \left ( \sum_{k=1}^K V_1^{k,\star} +  \cRtil_K   \right ) 
\end{align*}
so 
\begin{align*}
\cRtil_K & \le c_1 \betatil \sqrt{d H \log(1 + K/(d \lambda \sigmin^2))} \left ( \sqrt{ H^2 \log 1/\delta \cdot \left ( \sum_{k=1}^K V_1^{k,\star} + \cRtil_K \right ) } + \sqrt{H \sigmin^2 K} \right ) \\
& \qquad + c_3 \betatil \beta^2 H^2 \sqrt{d} \log(1 + K/(d \lambda \sigmin^2)) .
\end{align*}
Finally, choosing $\sigmin^2 = 1/K$ and solving the above for $\cRtil_K$ gives
\begin{align*}
\cRtil_K \le c_1 \betatil \sqrt{d H \log(1 + K/(d \lambda \sigmin^2))}  \sqrt{ H^2 \log 1/\delta \cdot \sum_{k=1}^K V_1^{k,\star}  } + c_2 \betatil \beta^2 H^3 \sqrt{d} \log(1 + K/(d \lambda \sigmin^2)) \cdot \log 1/\delta .
\end{align*}
Since $\cR_K \le H \Kinit +  \cRtil_K$, union bounding over $\cE$, which holds with probability at least $1-\delta$ by \Cref{lem:good_event}, and the two additional events stated above, and using that $\beta = 6 \sqrt{\Cmdp + \log(2HK/\delta)}$ and
\begin{align*}
\Cmdp := c (d^2 + \dR) \cdot \logterm \left (  d, \sigmin^{-1}, H, 1/\lambda, K, \RR  \right )  ,
\end{align*}
and the definition of $\betatil$, and setting $\lambda = 1/H^2$, gives the final result. 
\end{proof}


\end{document}